\documentclass{article}
\usepackage{sty/ijcai23}

\usepackage{times}
\usepackage[T1]{fontenc}
\usepackage{soul}
\usepackage{url}
\usepackage[hidelinks]{hyperref}
\usepackage[utf8]{inputenc}
\usepackage[small]{caption}
\usepackage{subcaption}
\usepackage{graphicx}
\usepackage[switch]{lineno}
\usepackage{booktabs}
\usepackage{multirow}
\usepackage{xcolor}
\usepackage{colortbl}
\usepackage{algorithm}
\usepackage[noEnd=true,indLines=true]{sty/algpseudocodex}
\usepackage{stmaryrd}
\usepackage{amsmath,amssymb,amsthm}
\usepackage{bm}
\usepackage{tikz}
\usetikzlibrary{calc}
\usetikzlibrary{positioning}
\usetikzlibrary{shapes.geometric}
\usetikzlibrary{patterns}
\usetikzlibrary{decorations.pathreplacing}
\usetikzlibrary{fit}
\usetikzlibrary{arrows.meta}
\usetikzlibrary{bending}
\usepackage[binary-units]{siunitx}
\usepackage{cleveref}
\usepackage{mystyle}

\title{Improving LaCAM for Scalable Eventually Optimal Multi-Agent Pathfinding}
\author{
  Keisuke Okumura$^{1,2}$
  \affiliations
  $^1$National Institute of Advanced Industrial Science and Technology (AIST)\\
  $^2$University of Cambridge
  \emails
  ko393@cl.cam.ac.uk
}
\begin{document}
\maketitle
\begin{abstract}
  This study extends the recently-developed LaCAM algorithm for multi-agent pathfinding (MAPF).
  LaCAM is a sub-optimal search-based algorithm that uses lazy successor generation to dramatically reduce the planning effort.
  We present two enhancements.
  First, we propose its anytime version, called \lacamstar, which eventually converges to optima, provided that solution costs are accumulated transition costs.
  Second, we improve the successor generation to quickly obtain initial solutions.
  Exhaustive experiments demonstrate their utility.
  For instance, \lacamstar sub-optimally solved 99\% of the instances retrieved from the MAPF benchmark, where the number of agents varied up to a thousand, within ten seconds on a standard desktop PC, while ensuring eventual convergence to optima;
  developing a new horizon of MAPF algorithms.
\end{abstract}

\section{Introduction}
The \emph{multi-agent pathfinding (MAPF)} problem aims to assign collision-free paths for multiple agents on a graph.
To date, various MAPF algorithms have been developed, motivated by various applications such as warehouse automation~\cite{wurman2008coordinating}.
Ideal MAPF algorithms will be complete, optimal, quick, and scalable.
However, there is generally a tradeoff between the former two and the latter two.
Conversely, \emph{the primary challenge of developments in MAPF algorithms is to guarantee solvability and solution quality, while suppressing planning efforts to secure speed and scalability}.

To break this tradeoff, we present two enhancements to the recently-developed algorithm called \emph{LaCAM (lazy constraints addition search for MAPF)}~\cite{okumura2023lacam}.
It is complete, sub-optimal, and search-based (akin to \astar search) that uses lazy successor generation.
The first enhancement is its anytime version called \emph{\lacamstar} that eventually converges to optima, provided that solution costs are accumulated transition costs.
Since solving MAPF optimally is computationally intractable~\cite{yu2013structure}, one practical approach to large instances is obtaining sub-optimal solutions and then refining their quality as time allows.
\lacamstar meets such demands.
The second enhancement is for the successor generation, i.e., tuning of the PIBT algorithm~\cite{okumura2022priority}, so as to quickly obtain initial solutions.

{
  \setlength{\tabcolsep}{1pt}
  \newcommand{\entry}[5]{
    &
    \begin{minipage}[t]{0.045\paperwidth}#1\end{minipage} &   
    \begin{minipage}[t]{0.125\paperwidth}#2\end{minipage} &   
    \begin{minipage}[t]{0.063\paperwidth}#3\end{minipage} &   
    \begin{minipage}[t]{0.072\paperwidth}#4\end{minipage} &   
    \begin{minipage}[t]{0.035\paperwidth}#5\end{minipage}     
  }
  \newcommand{\tri}[1]{\textcolor[HTML]{#1}{\tiny\m{\blacksquare}}}
  \medskip
  \begin{figure}[t!]
    \centering
    \includegraphics[width=1.0\linewidth]{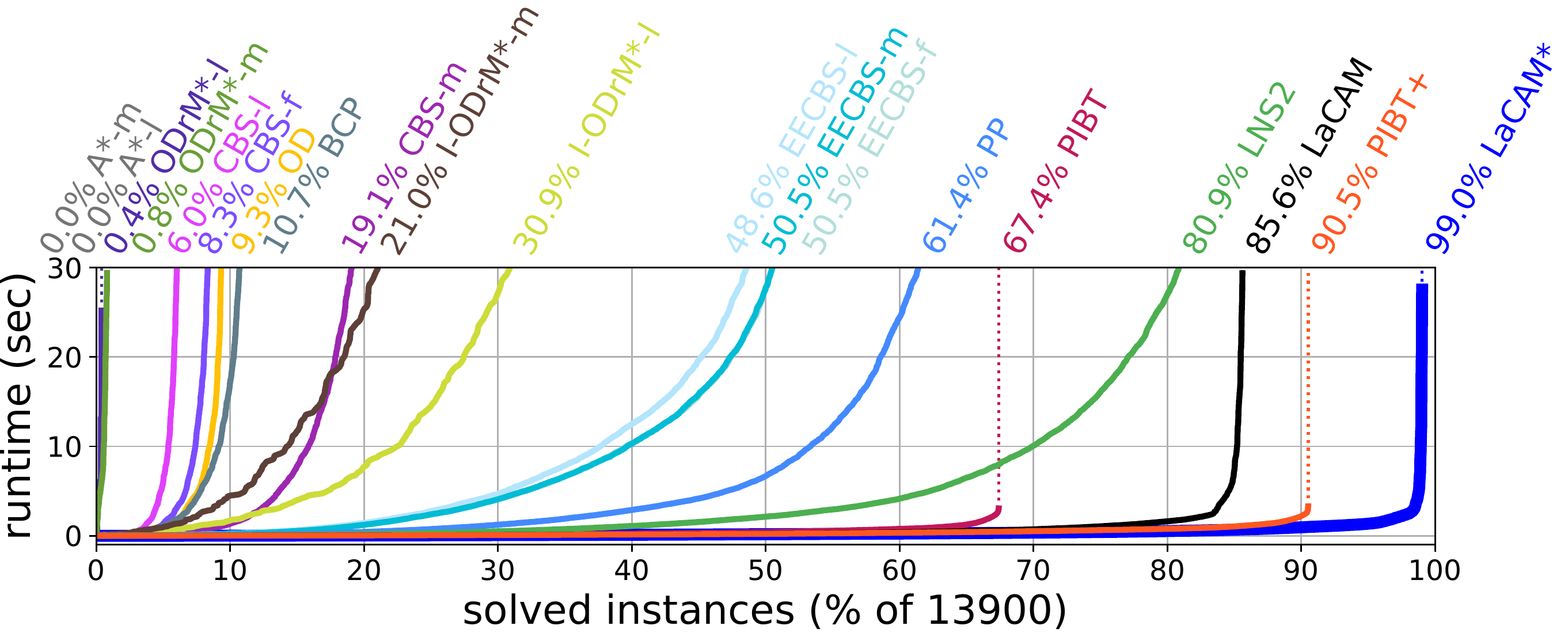}
    \\\smallskip
    \scriptsize
    \begin{tabular}{rlllll}
      \toprule
      & algorithm & reference & solvability & optimality & metrics
      \\\midrule
      \tri{757575}
      \entry{\astar}{\cite{hart1968formal}}{complete}{optimal}{m, l}
      \\\rc
      \tri{FFC107}
      \entry{OD}{\cite{standley2010finding}}{complete}{sub-optimal}{(greedy)}
      \\
      \tri{689F38}\tri{512DA8}
      \entry{ODrM$^\ast$}{\cite{wagner2015subdimensional}}{complete}{optimal}{m, l}
      \\\rc
      \tri{5D4037}\tri{CDDC39}
      \entry{I-ODrM$^\ast$}{\cite{wagner2015subdimensional}}{complete}{bnd. sub-opt.}{m, l}
      \\
      \tri{607D8B}
      \entry{BCP}{\cite{lam2022branch}}{solution cmp.}{optimal}{f}
      \\\rc
      \tri{7C4DFF}\tri{9C27B0}\tri{E040FB}
      \entry{CBS}{\cite{sharon2015conflict}}{solution cmp.}{optimal}{f, m, l}
      \\
      \tri{B2DFDB}\tri{00BCD4}\tri{B3E5FC}
      \entry{EECBS}{\cite{li2021eecbs}}{solution cmp.}{bnd. sub-opt.}{f, m, l}
      \\\rc
      \tri{448AFF}
      \entry{PP}{\cite{silver2005cooperative}}{incomplete}{sub-optimal}{}
      \\
      \tri{4CAF50}
      \entry{LNS2}{\cite{li2022mapf}}{incomplete}{sub-optimal}{}
      \\\rc
      \tri{C2185B}\tri{FF5722}
      \entry{PIBT$^{(+)}$}{\cite{okumura2022priority}}{incomplete}{sub-optimal}{}
      \\
      \tri{000000}
      \entry{LaCAM}{\cite{okumura2023lacam}}{complete}{sub-optimal}{}
      \\\rc
      \tri{0000FF}
      \entry{\w{\lacamstar}}{\w{this paper}}{\w{complete}}{\w{eventually opt.}}{\w{m, l}}
      \\
      \bottomrule
    \end{tabular}
    \caption{
      Performance on the MAPF benchmark.
      \emph{upper:}~The number of solved instances among 13,900 instances on 33 four-connected grid maps, retrieved from~\protect\cite{stern2019def}.
      The size of agents varies up to 1,000.
      `-f,' `-m,' and `-l' respectively mean that an algorithm tries to minimize flowtime, makespan, or sum-of-loss.
      The scores of \lacamstar are for initial solutions.
      \emph{lower:}~Representative or state-of-the-art MAPF algorithms.
      ``solution cmp.'' means that an algorithm ensures to find solutions for solvable instances but it never identifies unsolvable ones.
      ``bnd. sub-opt.'' means a bounded sub-optimal algorithm.
      Their sub-optimality was set to five.
    }
    \label{fig:percentage}
  \end{figure}
}

With these enhancements, we empirically demonstrate that \lacamstar can break the trade-off.
For instance, it sub-optimally solved 99\% of the instances retrieved from the MAPF benchmark~\cite{stern2019def} within ten seconds while guaranteeing the eventual optimality, on a standard desktop PC.
As illustrated in \cref{fig:percentage}, this result is beyond frontiers of existing MAPF algorithms.
In what follows, we present preliminaries, \lacamstar, improved successor generation, empirical results, and discussion in order.
The appendix and code are available at \url{https://kei18.github.io/lacam2/}.

{
  \setlength{\tabcolsep}{1pt}
  \newcommand{\entry}[1]{
    \begin{minipage}{0.19\linewidth}
      \includegraphics[width=1\linewidth]{fig/raw/lacam-#1.pdf}
    \end{minipage}
  }
  \begin{figure*}[th!]
    \centering
    \begin{tabular}{ccccc}
      \entry{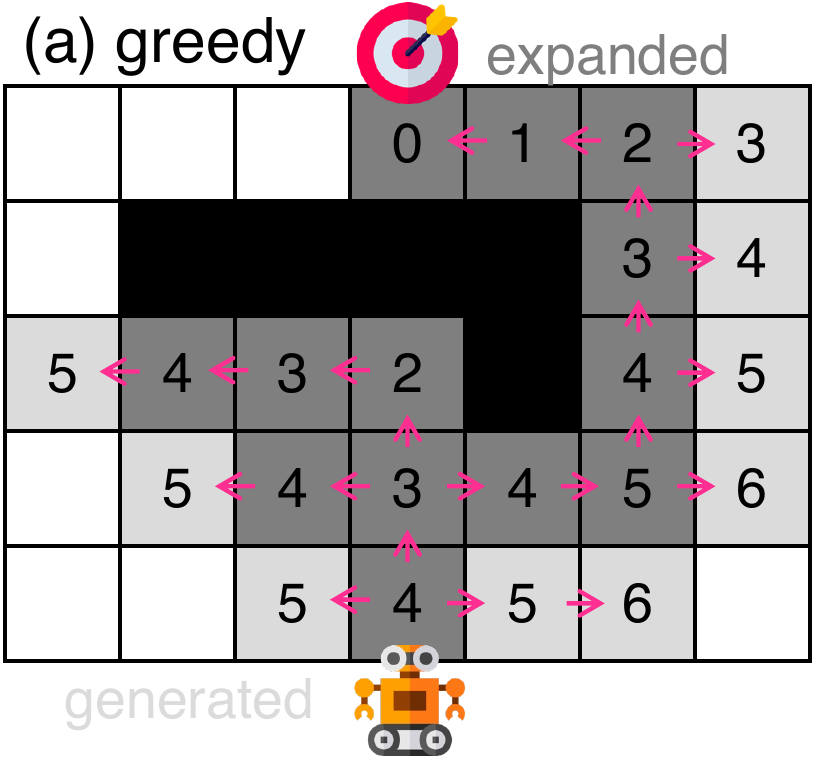} &
      \entry{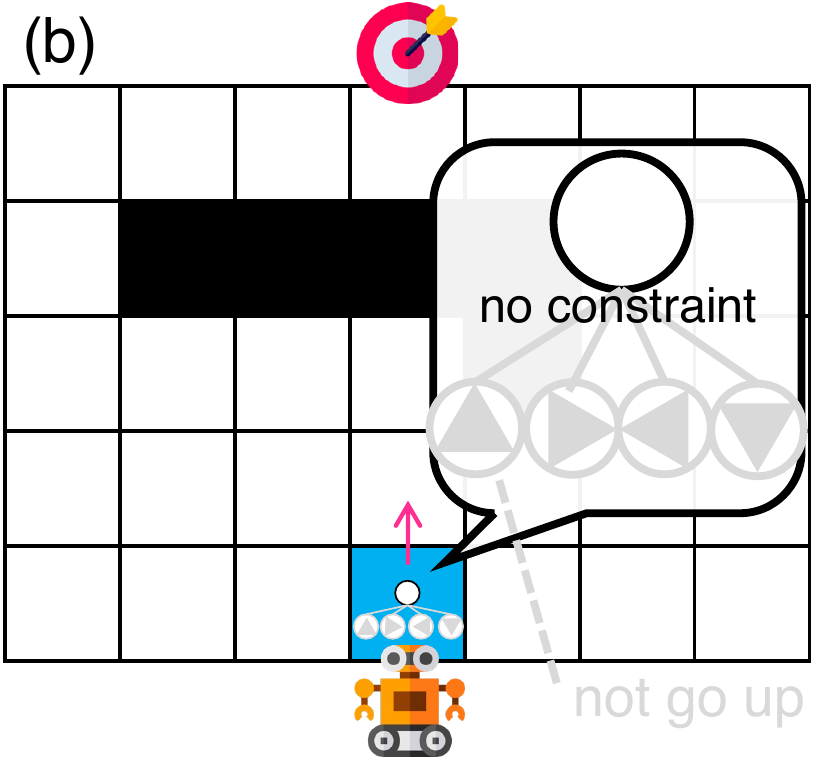} &
      \entry{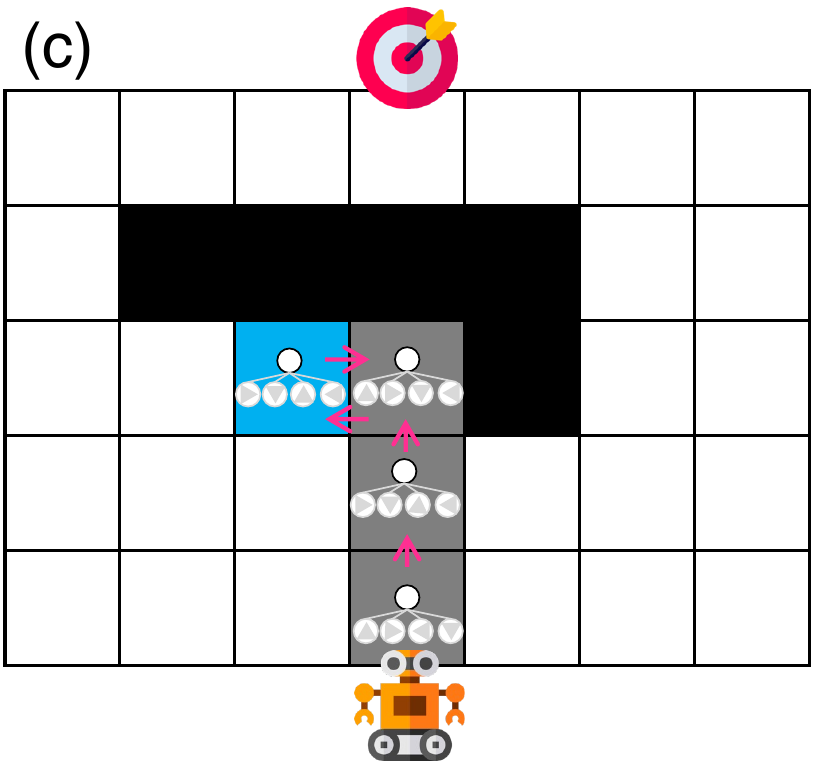} &
      \entry{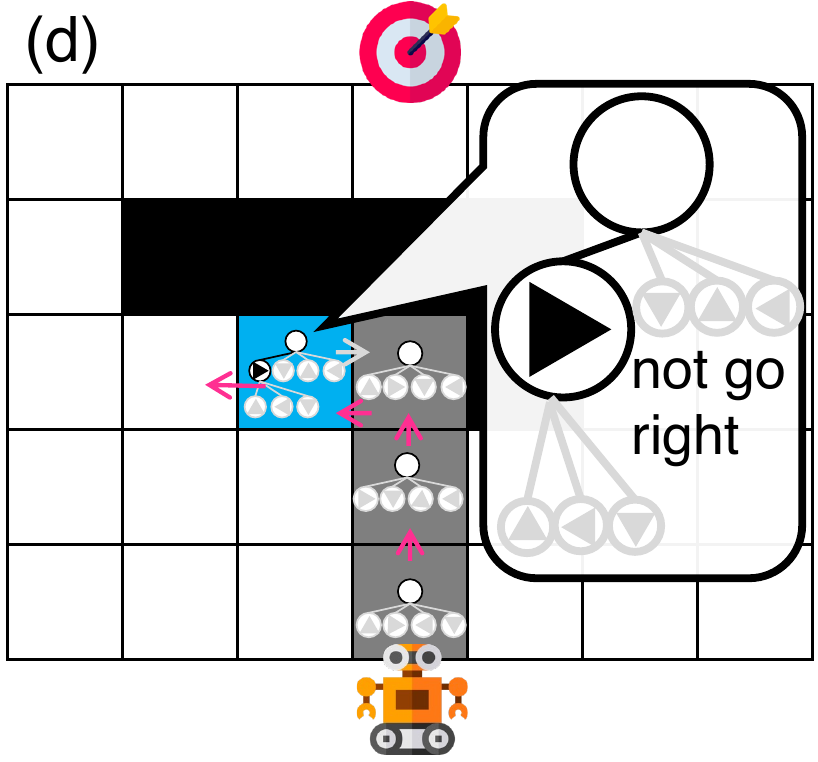} &
      \entry{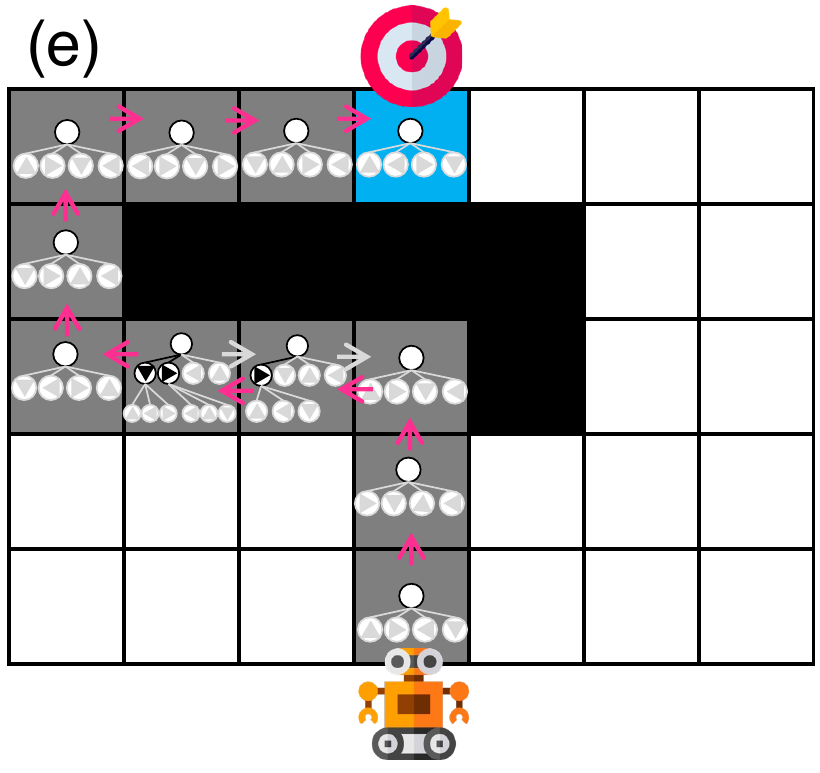}
    \end{tabular}
    \vspace{-0.2cm}
    \caption{Illustration of LaCAM using single-agent grid pathfinding.}
    \label{fig:example}
  \end{figure*}
}

\section{Preliminaries}
\label{sec:prelim}

\subsection{Notation, Problem Definition, and Assumption}
\paragraph{Notations.}
$S[k]$ denotes the $k$-th element of the sequence $S$, where the index starts at one.
For convenience, we use $\bot$ as an ``undefined'' or ``not found'' sign.

\paragraph{Instance.}
An \emph{MAPF instance} is defined by a graph $G = (V, E)$, a set of agents $A = \{1, \ldots, n\}$, a tuple of distinct starts $\S = (s_i \in V)^{i \in A}$ and goals $\G = (g_i \in V)^{i \in A}$.

\paragraph{Configuration.}
A \emph{configuration} is a tuple of locations for all agents.
For instance, $\Q = (v_1, v_2, \ldots, v_n) \in V^{|A|}$ is a configuration, where $\Q[i] = v_i$ is the location of agent $i \in A$.
The start and goal configurations are \S and \G, respectively.

\paragraph{Collision and Connectivity.}
A configuration \Q has a \emph{vertex collision} when there is a pair of agents $i, j \in A, i \neq j$, such that $\Q[i] = \Q[j]$.
Two configurations $X$ and $Y$ have an \emph{edge collision} when there is a pair of agents $i, j \in A, i \neq j$, such that $X[i] = Y[j] \land Y[i] = X[j]$.
Let $\neigh(v)$ be a set of vertices adjacent to $v \in V$.
Two configurations $X$ and $Y$ are \emph{connected} when $Y[i] \in \neigh(X[i]) \cup \{X[i]\}$ for all $i \in A$, and, there are neither vertex nor edge collisions in $X$ and $Y$.

\paragraph{Decision Problem.}
Given an MAPF instance, an \emph{MAPF problem} is a decision problem that asks existence of a sequence of configurations $\paths = (\Q_0, \Q_1, \ldots, \Q_k)$, such that $\Q_0 = \S$, $\Q_k = \G$, and any two consecutive configurations in \paths are connected.
A \emph{solution} to MAPF is a \paths that satisfies the conditions.
To align with the literature, the index of \paths exceptionally starts at zero (i.e., $\ct{0} = \Q_0$).
An algorithm is said to be \emph{complete} when it returns solutions for solvable instances and reports non-existence for unsolvable instances.
Otherwise, it is called \emph{incomplete}.

\paragraph{Optimization Problems.}
Given a transition (or edge) cost between two configurations, $\cost_e: V^{|A|} \times V^{|A|} \mapsto \mathbb{R}_{\geq 0}$, we aim at minimizing accumulated transition costs of a solution \paths, denoted as $\cost(\paths) \defeq \sum_{t=0}^{k-1}\cost_e \bigl(\ct{t}, \ct{t+1}\bigr)$.
This notation can represent various optimization metrics.
For instance, \cost is called \emph{makespan} when $\cost_e(X, Y) \defeq 1$.
It is called \emph{sum-of-fuels} (aka. \emph{total travel distance}) when $\cost(X, Y) \defeq |\{ i \in A; X[i] \neq Y[i]\}|$.
\emph{Sum-of-loss} counts actions of non-staying at goals, defined by $\cost_e(X, Y) \defeq |\{ i \in A \mid \lnot(X[i] = Y[i] = g_i)\}|$.
A solution \paths is \emph{optimal} when there is no solution $\paths'$ such that $\cost(\paths') < \cost(\paths)$.

\paragraph{Remarks for Flowtime.}
Another common metric of MAPF is \emph{flowtime} (aka. \emph{sum-of-costs}): $\sum_{i \in A} t_i$, where $t_i \leq k$ is the earliest timestep such that $\loc{i}{t_i} = \loc{i}{t_i+1} = \ldots = \loc{i}{k} = g_i$.
The difference from sum-of-loss is cost contribution of agents who once reach their goal and leave there temporarily.
The flowtime is history-dependent on paths of each agent;
hence, it is impossible to represent as it is with accumulative costs.%
\footnote{
However, it is worth noting that flowtime can be defined by introducing virtual goals, where once an agent has arrived there, it cannot move anywhere in the future.
}
Instead, this paper considers sum-of-loss as seen in~\cite{standley2010finding,wagner2015subdimensional}.

\paragraph{Admissible Heuristic.}
We assume an \emph{admissible heuristic} $\h: V^{|A|} \mapsto \mathbb{R}_{\geq 0}$, such that $\h(Q)$ is always the optimal cost from \Q to \G or less;
e.g., $\h(\Q) \defeq \sum_{i \in A}\dist(\Q[i], g_i)$ is available for the sum-of-\{loss, fuels\}, where $\dist: V\times V\mapsto \mathbb{N}_{\geq 0}$ is the shortest path length on $G$.

\paragraph{Understanding MAPF as Graph Pathfinding.}
Using configurations, consider a graph $H$ comprising vertices that represent configurations, and edges that represent the connectivity of configurations.
Then, by regarding \S and \G as start and goal vertices respectively, \emph{optimal MAPF is equivalent to the shortest pathfinding problem on $H$}.
This is a key perspective to understand LaCAM$^{(\ast)}$.

\subsection{LaCAM}
LaCAM~\cite{okumura2023lacam} was originally developed as a sub-optimal complete MAPF algorithm.
In a nutshell, \emph{it is a graph pathfinding algorithm}, like \astar, but some parts are specific to MAPF.
Below, we provide the essence of the graph pathfinding part, using an example of single-agent grid pathfinding.
The details of the MAPF-specific part are delivered in the appendix.

\paragraph{Classical Search.}
See~\cref{fig:example}a that illustrates how a typical search scheme solves grid pathfinding.
Specifically, we show the greedy best-first search with a heuristic of the Manhattan distance.
Here, a location of the agent corresponds to a \emph{configuration} (aka. \emph{state}) of the search.
From the start configuration, the search generates three successor nodes: left, up, and right.
Each node corresponds to one configuration.
The search then takes one of the generated nodes according to the heuristic, and generates successors.
This procedure continues until finding the goal configuration.

\paragraph{Branching Factor.}
Consider how many nodes are generated to estimate the search effort.
Though the solution length is 8, 22 nodes are generated.
This number is related to the number of connected configurations (i.e., \emph{branching factor}).
It is four in grid pathfinding, therefore, the number of node generations remains acceptable.
However, the branching factor of MAPF is exponential for the number of agents.
Consequently, the generation itself becomes intractable.
This is why a vanilla \astar is hopeless to solve large MAPF instances.

\paragraph{Configuration Generator.}
LaCAM tries to relieve this huge-branching-factor issue when a \emph{configuration generator} is available.
Given a configuration and \emph{constraints}, it generates a connected configuration following constraints.
Constraints should be embodied by each domain.
In this example, consider a constraint as a prohibition of direction, such as \emph{not} moving up, left, right, or down.

\paragraph{Constraint Tree.}
Each search node of LaCAM contains not only a configuration but also constraints, taking the form of a tree structure.
For each node invoke, LaCAM gradually develops the tree by \emph{low-level search}, implemented by, e.g., breadth-first search (BFS).
A node on the tree has a constraint and represents several constraints by tracing a path to the root.

\paragraph{Search Flow.}
We now explain LaCAM using \cref{fig:example}b--e, with a depth-first search (DFS) style.
The attempt to find a sequence of configurations is called \emph{high-level search}.
At first, a search node of the start is examined (\cref{fig:example}b).
The node poses no constraints for the first invoke, meaning that the configuration generator can output any connected configuration.
Suppose that the generator outputs an ``up'' configuration, following the Manhattan distance guide, illustrated by the pink arrow.
Preparing for the second invoke, the node expands a constraint tree with new constraints (e.g., ``not go up'').
The high-level search does not discard the examined node immediately, rather, it discards when all connected configurations have been generated (i.e., when the low-level search completes).
Next, \cref{fig:example}c--d show an example of the second invoke of nodes.
In \cref{fig:example}c, the generator outputs an already-known configuration.
Since this example assumes DFS, LaCAM examines the blue-colored node again in \cref{fig:example}d.
This time, the generator must follow a constraint ``not go right'' and its parent ``no constraint.''
The example then generates a ``left'' configuration.
The search continues until finding the goal (\cref{fig:example}e) and obtains a solution path by backtracking.

\paragraph{Adaptation to MAPF.}
With appropriate designs of constraints and tree construction, LaCAM can be an exhaustive search and guarantees completeness for graph pathfinding problems.
In \cite{okumura2023lacam}, such an example is shown for MAPF by letting a constraint specify which agent is where in the next configuration.
Moreover, LaCAM can greatly decrease the number of node generations if the configuration generator is promising in outputting configurations that are close to the goal.
This reduction could be a silver bullet to achieve quick planning, especially in planning problems where the branching factor is huge like MAPF.
The remaining question is a realization of good configuration generators.
In the original paper, \emph{PIBT (priority inheritance with backtracking)}~\cite{okumura2022priority} served as it, explained in \cref{sec:prelim:pibt}.

\paragraph{Pseudocode.}
\Cref{algo:lacam} shows DFS-based LaCAM.
Each search node $\N$ stores \emph{(i)}~a configuration, \emph{(ii)}~a constraint tree embodied by a queue (assuming BFS) and \emph{(iii)}~a pointer to a parent node (see \cref{algo:lacam:init-node}).
Nodes are stored in an \open list and \explored table, akin to general search schema, and are processed one by one.
We abstract how to create constraint trees for MAPF by \cref{algo:lacam:low-level-expansion}, which is elaborated in the appendix.

{
  \begin{algorithm}[t!]
    \caption{LaCAM. $\C\sub{init}$ means ``no constraint.''}
    \label{algo:lacam}
    \small
    \begin{algorithmic}[1]
    \Input{MAPF instance}
    \Output{solution or \nosolution}
      \State initialize \open, \explored
      \Comment{stack, hash table}
      \State $\N\init \leftarrow \bigl\langle~
        \config: \S,
        \tree: \llbracket~\C\init~\rrbracket,
        \parent: \bot
        ~\bigr\rangle$
      \label{algo:lacam:init-node}
      \State $\open.\push\left(\N\init\right)$;~~$\explored[\S] = \N\init$
      \While{$\open \neq \emptyset$}
      \State $\N \leftarrow \open.\funcname{top}()$
      \IfSingle{$\N.\config = \G$}{\Return $\backtrack(\N)$}
      \IfSingle{$\N.\tree = \emptyset$}{$\open.\pop()$;~\Continue}
      \Comment{discard node}
      \State $\C \leftarrow \N.\tree.\pop()$
      \Comment{get constraint}
      \State $\funcname{low\_level\_expansion}(\N, \C)$
      \Comment{proceed low-level search}
      \label{algo:lacam:low-level-expansion}
      \State $\Q\new \leftarrow \funcname{configuration\_generator}(\N, \C)$
      \label{algo:lacam:new-config}
      \IfSingle{$\Q\new = \bot$}{\Continue}
      \Comment{generator may fail}
      \IfSingle{$\explored\left[\Q\new\right] \neq \bot$}{\Continue}
      \label{algo:lacam:closed}
      \State $\N\new \leftarrow \langle~
        \config: \Q\new,
        \tree: \left\llbracket~\C\init~\right\rrbracket,
        \parent: \N~\bigr\rangle$
        \label{algo:lacam:successor}
      \State $\open.\push\left(\N\new\right)$;\;$\explored\left[\Q\new\right] = \N\new$
      \EndWhile
      \State \Return \nosolution
    \end{algorithmic}
  \end{algorithm}
}

{
  \begin{algorithm}[t!]
    \caption{PIBT}
    \label{algo:pibt}
    \small
    \begin{algorithmic}[1]
      \Input{configuration $\Q\from$, agents $A$, goals $( g_1, \ldots g_n )$}
      \Output{configuration $\Q\to$ (each element is initialized with $\bot$)}
      \State \textbf{for}~$i \in A$~\textbf{do};~
             \textbf{if}~{$\Q\to[i] = \bot$}~\textbf{then}~{$\PIBT(i)$}
      \label{algo:pibt:call-func}
      \State \Return $\Q\to$
      \label{algo:pibt:top:end}
      \Procedure{\PIBT}{$i$}
      \Comment{return \valid or \invalid}
      \label{algo:pibt:recursive:start}
      \State $C \leftarrow \neigh\left(\Q\from[i]\right) \cup \left\{\Q\from[i]\right\}$
      \label{algo:pibt:cand}
      \Comment{candidate vertices}
      \State sort $C$ in ascending order of $\dist(u, g_i)$ where $u \in C$
      \label{algo:pibt:sort}
      \For{$v \in C$}
      \label{algo:pibt:loop-start}
      \IfSingle{collisions in $\Q\to$ supposing $\Q\to[i]=v$}{\Continue}
      \label{algo:pibt:collision-check}
      \State $\Q\to[i] \leftarrow v$
      \label{algo:pibt:reserve}
      \If{$\exists j \in A~\text{s.t.}~j \neq i \land \Q\from[j]=v \land \Q\to[j]=\bot$}
      \label{algo:pibt:check-k}
      \IfSingle{$\PIBT(j)= \invalid$}{\Continue}
      \label{algo:pibt:call-pi}
      \EndIf
      \State \Return~\valid
      \Comment{assignment done}
      \label{algo:pibt:valid}
      \EndFor
      \label{algo:pibt:loop-end}
      \State $\Q\to[i] \leftarrow \Q\from[i]$;~\Return~\invalid
      \label{algo:pibt:stay}
      \EndProcedure
      \label{algo:pibt:recursive:end}
    \end{algorithmic}
  \end{algorithm}
}

{
  \tikzset{
    vertex/.style={circle,draw,black,align=center,inner sep=0.1cm, minimum size=0.6cm,anchor=center},
    pi/.style={very thick,->},
  }
  \newcommand{\edgesize}{0.5}
  \newcommand{\drawgraph}{
    \node[vertex](v1) at (0, 0) {};
    \node[vertex,right=\edgesize of v1.center](v2) {};
    \node[vertex,right=\edgesize of v2.center](v3) {};
    \node[vertex,above=\edgesize of v2.center](v4) {};
    \foreach \u / \v in {v1/v2,v2/v3,v2/v4,v4/v3}
    \draw[] (\u) -- (\v);
  }
  \newcommand{\cross}{$\mathbin{\tikz [x=1.4ex,y=1.4ex,line width=.2ex,black] \draw (0,0) -- (1,1) (0,1) -- (1,0);}$}
  \begin{figure}[t!]
    \centering
    \scalebox{0.8}{
    \begin{tikzpicture}
      \begin{scope}[shift={(0, 0)}]
        \drawgraph
        \node[vertex,fill=blue,text=white] at (v4) {\large$i$};
        \node[vertex,fill=teal,text=white] at (v1) {\large$k$};
        \node[vertex,fill=red,text=white,label=above:{$v$}] at (v3) {\large$j$};
        \node[above right=-0.3cm of v3]{\cross};
        \draw[pi](v4)--(v3);
        \draw[pi](v1)--(v2);
        \node[anchor=east,align=left] at (0.5, 0.75) {(a)~fixed order:\\$i, k, j$};
      \end{scope}
      \begin{scope}[shift={(4.6, 0)}]
        \drawgraph
        \node[vertex,fill=blue,text=white] at (v4) {\large$i$};
        \node[vertex,fill=teal,text=white] at (v1) {\large$k$};
        \node[vertex,fill=red,text=white] at (v3)  {\large$j$};
        \draw[pi](v4)--(v3);
        \draw[pi](v3)--(v2);
        \draw[pi](v1) to[out=45,in=0] ($(v1)+(0,0.7)$) to[out=180,in=110] (v1);
        \node[anchor=east,align=left] at (0.3, 0.75) {(b)~with PIBT:\\$i$ first};
      \end{scope}
    \end{tikzpicture}
    }
    \caption{
      Concept of PIBT.
      $\Q\from$ is illustrated.
      Bold arrows represent assignments of $\Q\to$.
      \emph{(a)}~Consider a fixed assignment order of $i$, $k$, and $j$.
      If $i$ and $k$ are assigned following the illustrated arrows, $j$ has no candidate vertex as $\Q\to[j]$ (annotated with $\times$).
      \emph{(b)}~
      This pitfall is overcome by doing the assignment for $j$ prior to $k$, reacting to $i$'s assignment request.
    }
    \label{fig:pibt-example}
  \end{figure}
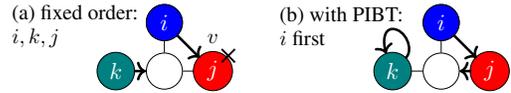
}

\subsection{PIBT}
\label{sec:prelim:pibt}
PIBT~\cite{okumura2022priority} was originally developed to solve MAPF iteratively.
In a nutshell, \emph{it is a configuration generator, which generates a new connected configuration ($\Q\to$), given another ($\Q\from$) as input}.
By continuously generating configurations, PIBT can generate a solution for MAPF.

\paragraph{Concept.}
To determine $\Q\to$, PIBT sequentially assigns a vertex to each agent while avoiding assignments that trigger collisions.
This assignment order adaptively changes.
Specifically, before fixing $\Q\to[i]$ to $v \in \neigh(\Q\from[i]) \cup \{\Q\from[i]\}$, PIBT first checks the existence of $j \in A$ such that $\Q\from[j] = v$.
If such $j$ exists, $j$ may have no candidate vertex for $\Q\to[j]$, due to collision avoidance (see \cref{fig:pibt-example}a).
Therefore, PIBT next determines $\Q\to[j]$ prior to locations of other agents (e.g., $k$ in \cref{fig:pibt-example}).
This scheme is called \emph{priority inheritance}.
If $\Q\to[j]$ is successfully assigned, $\Q\to[i]$ is fixed to $v$ (\cref{fig:pibt-example}b);
otherwise, $\Q\to[i]$ needs to take another vertex other than $v$.

\paragraph{Pseudocode.}
\Cref{algo:pibt} implements the concept above, by recursively calling a procedure \PIBT (\cref{algo:pibt:recursive:start}--), which takes an agent $i$ and eventually assigns $\Q\to[i]$.
The assignment attempts of \reflines{algo:pibt:loop-start}{algo:pibt:loop-end} are performed for candidate vertices regarding $\Q\from[i]$, in ascending order of the distance toward $g_i$.
The attempts continue until $\Q\to[i]$ is determined, resulting in \valid outcome (\cref{algo:pibt:valid}).
When all attempts failed, $\PIBT(i)$ assigns $\Q\from[i]$ to $\Q\to[i]$ and returns \invalid (\cref{algo:pibt:stay}).
Priority inheritance is triggered as necessary, taking the form of calling \PIBT for another agent $j$ (\cref{algo:pibt:call-pi}).
The success of priority inheritance triggers receipt of \valid;
otherwise, \invalid is backed, and then $i$ continues the assignment attempts.

\paragraph{Dynamic Priorities.}
In addition to priority inheritance, PIBT prioritizes assignments for agents that are not on their goal, which is done by sorting $A$ of \cref{algo:pibt:call-func} in that way.
This scheme is convenient for lifelong scenarios wherein all agents are not necessarily being goals simultaneously.

\section{\lacamstar: Eventually Optimal Algorithm}
\Cref{algo:star} presents \lacamstar.
The same lines as LaCAM (\cref{algo:lacam}) are grayed out.
The blue-colored lines are not necessary from the theoretical side but are effective in speeding up the search.
The main differences from LaCAM are two:
\emph{(i)}~it continues the search when finding the goal configuration \G, and, \emph{(ii)}~it rewrites parent relations between search nodes as necessary.
For convenience, the transition cost $\cost_e$ and admissible heuristic \h can take nodes as arguments, instead of configurations.
Below, the updated parts are explained.

\paragraph{Keeping Goal Node.}
\lacamstar retains the goal node $\N\goal$, rather than immediately returning solutions when first finding the goal \G (\cref{algo:star:goal}).
The search terminates when there is no remaining node in \open;
otherwise, there is an interruption from users such as timeout (\cref{algo:star:while-start}).
A solution is then constructed by backtracking from $\N\goal$ (\reflines{algo:star:optimal}{algo:star:sub-optimal}).
Doing so makes \lacamstar an anytime algorithm, that is, \emph{after finding the goal node, it is interruptible whenever a solution is required, while gradually refining solution quality as time allows}.

\paragraph{Search Node Ingredients.}
Each high-level node contains a set \neighbors that stores connected configurations (i.e., nodes) and g-value that represents cost-to-come from the start \S (\cref{algo:star:init-node}).
They are initialized and updated appropriately when finding a new configuration (\reflines{algo:star:new-node}{algo:star:append-neighbor}).

\paragraph{Updating Parents and Costs.}
To maintain the optimality, when finding an already known configuration, \lacamstar updates \neighbors (\cref{algo:star:add-neighbor-known}).
This is followed by updates of g-value and \parent, performed by Dijkstra's algorithm~(\reflines{algo:star:dijkstra-start}{algo:star:dijkstra-end}).
\Cref{fig:rewrite} illustrates the update.

\paragraph{Discarding Redundant Nodes.}
Once the goal node is found, \lacamstar discards nodes that do not contribute to improving solution quality (\cref{algo:star:branching}).
It also revives nodes as necessary when their $g$-values are updated (\cref{algo:star:reinsert}).

\begin{theorem}
  \lacamstar (\cref{algo:star}) is complete and optimal.
  \label{thrm:optimality}
\end{theorem}
\begin{proof}
  Consider \cref{algo:star} without blue lines (\cref{algo:star:branching,algo:star:reinsert});
  they just speed up the search without breaking the optimal search structure.
  In this proof, the term ``path'' refers to a sequence of connected configurations.

  {
  \renewcommand{\hl}[1]{\textcolor{blue}{#1}}
  \newcommand{\dopen}{\m{\mathcal{D}}}
  \newcommand{\f}{\funcname{f}}
  \begin{algorithm}[t!]
    \caption{\lacamstar. $\C\sub{init}$ means ``no constraint.''}
    \label{algo:star}
    \small
    \begin{algorithmic}[1]
    \Input{MAPF instance, edge cost $\cost_e$, admissible heuristic $\h$}
    \Output{solution, \nosolution, or \failure}
    \Notation{$\f(\N) \defeq \N.g + \h(\N)$;\; $\spadesuit \defeq \left(\N\goal \neq \bot\right)$}
      \State \same{initialize $\open, \explored$;}~$\N\goal \leftarrow \bot$
      \label{algo:star:init-var}
      \State $\same{\N\init \leftarrow \bigr\langle
        \config:\S,
        \tree:\llbracket \C\init \rrbracket,
        \text{\it parent}:\bot,}   
        ~\neighbors: \emptyset,
        g: 0
        \same{\bigr\rangle}$
        \label{algo:star:init-node}
      \State \same{$\open.\push\left(\N\init\right)$; $\explored[\S] = \N\init$}
      \While{$\same{\open \neq \emptyset}~\land~\lnot\interrupt()$}
      \label{algo:star:while-start}
      \State \same{$\N \leftarrow \open.\funcname{top}()$}
      \IfSingle{$\N.\config = \G$}{$\N\goal \leftarrow \N$}
      \label{algo:star:goal}
      \IfSingleHl{$\spadesuit \land \f\left(\N\goal\right) \leq \f\left(\N\right)$}
      {$\open.\pop()$;~\Continue}
      \label{algo:star:branching}
      \IfSingleSame{$\N.\tree = \emptyset$}{$\open.\pop()$;~\Continue}
      \State \same{$\C \leftarrow \N.\tree.\pop()$}
      \State \same{$\funcname{low\_level\_expansion}(\N, \C)$}
      \label{algo:star:low-level-expansion}
      \State \same{$\Q\new \leftarrow \funcname{configuration\_generator}(\N, \C)$}
      \IfSingleSame{$\Q\new = \bot$}{\Continue}
      \If{$\explored\left[\Q\new\right] \neq \bot$}
      \label{algo:star:already-known}
      \State $\N.\neighbors.\append\left(\explored\left[\Q\new\right]\right)$
      \label{algo:star:add-neighbor-known}
      \State $\dopen \leftarrow \llbracket~\N~\rrbracket$
      \label{algo:star:dijkstra-start}
      \Comment{Dijkstra, priority queue of $g$-value}
      \While{$\dopen \neq \emptyset$}
      \State $\N\from \leftarrow \dopen.\pop()$
      \For{$\N\to \in \N\from.\neighbors$}
      \State $g \leftarrow \N\from.g + \cost_e\left(\N\from, \N\to\right)$
      \If{$g < \N\to.g$}
      \label{algo:star:g-value-pruning}
      \State $\N\to.g \leftarrow g$;~$\N\to.\parent \leftarrow \N\from$;~$\dopen.\push(\N\to)$
      \label{algo:star:dijkstra-end}
      \IfSingleHl{$\spadesuit \land \f\left(\N\to\right) < \f\left(\N\goal\right)$}
                 {$\open.\push\left(\N\to\right)$}
      \label{algo:star:reinsert}
      \EndIf
      \EndFor
      \EndWhile
      \Else
      \State $\same{\N\new \leftarrow}
      \begingroup
      \addtolength{\jot}{-0.3em}
      \begin{aligned}[t]
        \same{\bigl\langle}
        &\same{
          \config: \Q\new,
          \tree: \llbracket \C\init \rrbracket,
          \parent: \N,
        }
        \\
        &\neighbors: \emptyset,
        g: \N.g + \cost_e\left(\N, \Q\new\right)
        \same{\bigr\rangle}
        \end{aligned}\endgroup$
        \label{algo:star:new-node}
        \State \same{$\open.\push\left(\N\new\right)$;\;$\explored\left[\Q\new\right] = \N\new$}
        \State $\N.\neighbors.\append\left(\N\new\right)$
        \label{algo:star:append-neighbor}
      \EndIf
      \EndWhile
      \label{algo:star:while-end}
      \If{$\spadesuit \land \open = \emptyset$}
      \Return $\backtrack\left(\N\goal\right)$
      \label{algo:star:optimal}
      \Comment{optimal}
      \ElsIf{$\spadesuit$}
      \Return $\backtrack\left(\N\goal\right)$
      \label{algo:star:sub-optimal}
      \Comment{sub-optimal}
      \ElsIf{$\open = \emptyset$}
      \Return \nosolution
      \Else
      ~\Return \failure
      \EndIf
    \end{algorithmic}
  \end{algorithm}
}

  {
  \tikzset{
    vertex/.style={circle,draw,black,align=center,inner sep=0.05cm, minimum size=0.2cm,anchor=center},
  }
  \begin{figure}[t!]
    \centering
    \scalebox{0.9}{
    \begin{tikzpicture}
      \small
      \begin{scope}[shift={(0, 0)}]
        \node[vertex](v0) at (0, 0) {0};
        \node[vertex, minimum size=0.5cm](v1) at (v0) {};
        \node[vertex](v2) at (1.0, 0.8) {1};
        \node[vertex](v3) at (2.0, 1.0) {2};
        \node[vertex](v4) at (2.8, 0.3) {3};
        \node[vertex](v5) at (2.1, -0.3) {4};
        \node[vertex](v6) at (1.2, -0.1) {5};
        \node[vertex](v7) at (2.2, -1.0) {6};
        \node[vertex](v8) at (0.5, -0.8) {1};
        \node[](v9) at (3.3, -0.5) {};
        \node[](v10) at (3.3, 1.1) {};
        \node[anchor=west] at (-0.5, 0.9) {before};
        \foreach \u / \v in {v0/v2,v2/v3,v3/v4,v4/v5,v5/v6,v6/v7,v0/v8,v7/v9,v3/v10}
        \draw[very thick,->] (\u) -- (\v);
        \foreach \u / \v in {v4/v2,v6/v3,v3/v2,v5/v4,v6/v5}
        \draw[->,densely dashed] (\u) -- (\v);
        \draw[->,red,densely dashed,very thick](v8)--(v6);
      \end{scope}
      \begin{scope}[shift={(4.3, 0)}]
        \node[vertex](v0) at (0, 0) {0};
        \node[vertex, minimum size=0.5cm](v1) at (v0) {};
        \node[vertex](v2) at (1.0, 0.8) {1};
        \node[vertex](v3) at (2.0, 1.0) {2};
        \node[vertex](v4) at (2.8, 0.3) {3};
        \node[vertex,text=red](v5) at (2.1, -0.3) {3};
        \node[vertex,text=red](v6) at (1.2, -0.1) {2};
        \node[vertex,text=red](v7) at (2.2, -1.0) {3};
        \node[vertex](v8) at (0.5, -0.8) {1};
        \node[](v9) at (3.3, -0.5) {};
        \node[](v10) at (3.3, 1.1) {};
        \node[anchor=west] at (-0.5, 0.9) {after};
        \foreach \u / \v in {v0/v2,v2/v3,v3/v4,v3/v10,v0/v8}
        \draw[very thick,->] (\u) -- (\v);
        \foreach \u / \v in {v4/v2,v6/v3,v3/v2,v5/v6,v4/v5,v5/v4}
        \draw[->,densely dashed] (\u) -- (\v);
        \foreach \u / \v in {v8/v6,v6/v5,v6/v7,v7/v9}
        \draw[->,red,very thick](\u)--(\v);
      \end{scope}
    \end{tikzpicture}
    }
    \caption{
      Updating parents and costs.
      Each circle is a search node (i.e., configuration), including its g-value of makespan.
      Arrows represent known neighboring relations.
      Among them, solid lines represent \parent.
      The updated parts are red-colored.
      \emph{left}:~A new neighbor relationship, a red dashed arrow, has been found.
      \emph{right}:~Rewrite the search tree.
      Observe that the rewriting occurs in a limited part of the tree due to g-value pruning~(\cref{algo:star:g-value-pruning}).
    }
    \label{fig:rewrite}
  \end{figure}
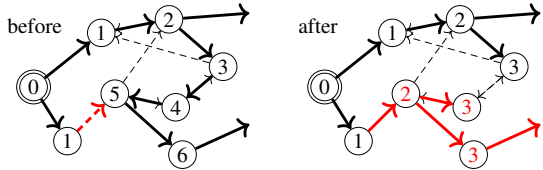
}

  First, we introduce a directed graph $H$, where its vertex corresponds to a configuration.
  Initially, $H$ has only a start vertex \S.
  Then, the search iterations gradually develop $H$.
  When a new node is created (\cref{algo:star:new-node}), its configuration is added to $H$.
  An arc $(X, Y)$ of $H$ occurs when the search finds a connection from $X$ to $Y$, i.e., $\N^Y \in \N^X.\neighbors$.

  We now prove that: \emph{($\clubsuit$)~for any configuration \Q in $H$, a path from \S to \Q constructed by backtracking (i.e., by following \N.\parent) is the shortest path in $H$, regarding accumulative transition costs, at the beginning of each search iteration}.
  This is proven by induction.
  Initially, $H$ comprises only \S, satisfying $\clubsuit$.
  Assume now that $\clubsuit$ is satisfied in the previous iteration of \reflines{algo:star:while-start}{algo:star:while-end}.
  In the next iteration, $H$ is updated by:
  \emph{(i)}~generating a new configuration (\reflines{algo:star:new-node}{algo:star:append-neighbor}), or
  \emph{(ii)}~finding a known configuration (\reflines{algo:star:add-neighbor-known}{algo:star:dijkstra-end}).
  Case \emph{(i)} holds $\clubsuit$ because only a new vertex and an arc toward the vertex are added to $H$.
  Case \emph{(ii)} holds $\clubsuit$ because only an arc is added for the node \N (\cref{algo:star:add-neighbor-known}), and then \reflines{algo:star:dijkstra-start}{algo:star:dijkstra-end} perform Dijkstra's algorithm starting from \N that maintains the tree structure of the shortest paths.

  The search space of \lacamstar is finite (see~\cite{okumura2023lacam});
  the search terminates in finite time.
  Each node eventually examines all connected configurations.
  Consequently, when terminated, $H$ includes all possible paths from \S to \G for solvable instances.
  Together with $\clubsuit$, \lacamstar returns an optimal solution, otherwise, reports the non-existence.
\end{proof}

\paragraph{Implementation Tips.}
Following~\cite{okumura2023lacam}, when finding an already known configuration at \cref{algo:star:already-known}, our implementation reinserts the corresponding node to \open.
Moreover, with a small probability (e.g., 0.1\%), the implementation reinserts a node of \S instead of the found one.
Doing so enables the search to ``escape'' from configurations being bottlenecks.
Such techniques relying on non-determinism have been seen in other search problems~\cite{kautz2002dynamic} and MAPF studies~\cite{cohen2018rapid,andreychuk2018two}.
Indeed, we informally observed that this random replacement slightly improved the success rate.
Note that the optimality still holds with these modifications.

\section{Improving Configuration Generator}
The performance of LaCAM heavily relies on a configuration generator, therefore, the development of good generators is critical.
The implementation in \cite{okumura2023lacam} uses a vanilla PIBT of \cref{algo:pibt}, resulting in poor performances in several scenarios, especially in instances with narrow corridors.
This is because PIBT itself often fails such scenarios, hence being an ineffective guide for LaCAM.
This section elaborates on this phenomenon and presents an improved version.

\subsection{Failure Analysis of PIBT}
As seen in \cref{sec:prelim:pibt}, PIBT sequentially assigns the next locations for agents.
Since this order prioritizes agents being not on their goals, livelock situations might be triggered.
See \cref{fig:pibt-failure};
two agents reach their goal vertex periodically but PIBT never reaches the goal configuration.

LaCAM can break such livelocks by posing constraints.
However, it may require significant effort because appropriate combinations of constraints should be explored.
Even worse, with more agents, the search effort dramatically increases, as demonstrated in \cref{table:dislike-example}.

\subsection{Enhancing PIBT by Swap}
Livelocks in PIBT can be resolved by \emph{swap} operation, originally developed in rule-based MAPF algorithms~\cite{luna2011push,de2014push}.
In short, this operation swaps locations of two agents using a vertex with a degree of three or more.
\Cref{fig:swap} shows an example.
Here, we extract its essence and incorporate it into PIBT.
Specifically, this is done by adjusting vertex scoring at \cref{algo:pibt:sort} in \cref{algo:pibt}.

{
  \begin{algorithm}[t!]
    \caption{procedure \PIBT with swap}
    \label{algo:pibt-swap}
    \small
    \begin{algorithmic}[1]
      \State \same{$C \leftarrow \neigh\left(\Q\from[i]\right) \cup \left\{ \Q\from[i] \right\}$}
      \State \same{sort $C$ in increasing order of $\dist(u, g_i)$ where $u \in C$}
      \label{algo:pibt-swap:sort}
      \State $j \leftarrow \funcname{swap\_required\_and\_possible}\left(i, C[1], \Q\from\right)$
      \label{algo:pibt-swap:identify}
      \IfSingle{$j \neq \bot$}{reverse $C$}
      \label{algo:pibt-swap:reverse}
      \ForSame{$v \in C$}
      \State \same{\reflines{algo:pibt:collision-check}{algo:pibt:call-pi} of \cref{algo:pibt}}
      \IfSingle{$v = C[1] \land j \neq \bot \land \Q\to[j]=\bot$}{$\Q\to[j] \leftarrow \Q\from[i]$}
      \label{algo:pibt-swap:pull}
      \State \same{\Return \valid}
      \EndFor
      \State \same{$\Q\to[i] \leftarrow \Q\from[i]$;~\Return \invalid}
    \end{algorithmic}
  \end{algorithm}
}

{
  \tikzset{
    vertex/.style={circle,draw,black,align=center,inner sep=0.1cm, minimum size=0.6cm,anchor=center},
  }
  \newcommand{\edgesize}{0.4}
  \newcommand{\drawgraph}{
    \node[vertex](v1) at (0, 0) {};
    \node[vertex,right=\edgesize of v1.center](v2) {};
    \node[vertex,right=\edgesize of v2.center](v3) {};
    \node[vertex,below=\edgesize of v2.center](v4) {};
    \node[vertex,below=\edgesize of v4.center](v5) {};
    \node[vertex,below=\edgesize of v5.center](v6) {};
    \foreach \u / \v in {v1/v2,v2/v3,v2/v4,v4/v5,v5/v6}
    \draw[] (\u) -- (\v);
  }
  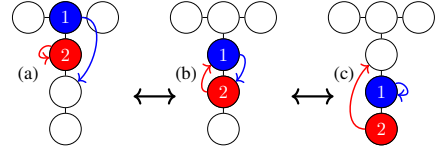
\begin{figure}[t!]
    \centering
    \scalebox{0.7}{
    \begin{tikzpicture}
      \begin{scope}[shift={(-3, 0)}]
        \drawgraph
        \node[vertex,fill=blue,text=white] at (v2) {$1$};
        \node[vertex,fill=red,text=white] at (v4) {$2$};
        \draw[thick,->,blue](v2) to[out=0,in=40] (v5);
        \draw[thick,->,red](v4) to[out=150,in=90] ($(v4)+(-0.5,0)$) to[out=-90,in=190] (v4);
        \node[below=0.5 of v1] {(a)};
      \end{scope}
      \begin{scope}[shift={(0, 0)}]
        \drawgraph
        \node[vertex,fill=blue,text=white] at (v4) {$1$};
        \node[vertex,fill=red,text=white] at (v5) {$2$};
        \draw[thick,->,blue](v4) to[out=0,in=40] (v5);
        \draw[thick,->,red](v5) to[out=180,in=220] (v4);
        \node[below=0.5 of v1] {(b)};
      \end{scope}
      \begin{scope}[shift={(3, 0)}]
        \drawgraph
        \node[vertex,fill=blue,text=white] at (v5) {$1$};
        \node[vertex,fill=red,text=white] at (v6) {$2$};
        \draw[thick,->,blue](v5) to[out=30,in=90] ($(v5)+(0.5,0)$) to[out=-90,in=-10] (v5);
        \draw[thick,->,red](v6) to[out=180,in=220] (v4);
        \node[below=0.5 of v1] {(c)};
      \end{scope}
      \draw[very thick,<->](-1,-1.5) -- (-0.2,-1.5);
      \draw[very thick,<->](2,-1.5) -- (2.8,-1.5);
    \end{tikzpicture}
    }
    \caption{
      Failure example of PIBT.
      Goals are represented by arrows.
    }
    \label{fig:pibt-failure}
  \end{figure}
}

{
  \tikzset{
    vertex/.style={circle,draw,black,align=center,inner sep=0.1cm, minimum size=0.6cm,anchor=center},
  }
  \newcommand{\edgesize}{0.4}
  \newcommand{\drawgraph}{
    \node[vertex](v1) at (0, 0) {};
    \node[vertex,right=\edgesize of v1.center](v2) {};
    \node[vertex,right=\edgesize of v2.center](v3) {};
    \node[vertex,below=\edgesize of v2.center](v4) {};
    \node[vertex,below=\edgesize of v4.center](v5) {};
    \node[vertex,below=\edgesize of v5.center](v6) {};
    \foreach \u / \v in {v1/v2,v2/v3,v2/v4,v4/v5,v5/v6}
    \draw[] (\u) -- (\v);
  }
  \begin{figure}[t!]
    \centering
    \scalebox{0.7}{
    \begin{tikzpicture}
      \begin{scope}[shift={(0, 0)}]
        \drawgraph
        \node[vertex,fill=blue,text=white] at (v4) {$1$};
        \node[vertex,fill=red,text=white] at (v5) {$2$};
        \draw[thick,->,blue](v4) to[out=0,in=40] (v5);
        \draw[thick,->,red](v5) to[out=180,in=220] (v4);
        \node[below=0.5 of v1] {(a)};
      \end{scope}
      \begin{scope}[shift={(3, 0)}]
        \drawgraph
        \node[vertex,fill=blue,text=white] at (v2) {$1$};
        \node[vertex,fill=red,text=white] at (v4) {$2$};
        \draw[thick,->,blue](v2) to[out=0,in=40] (v5);
        \draw[thick,->,red](v4) to[out=150,in=90] ($(v4)+(-0.5,0)$) to[out=-90,in=190] (v4);
        \node[below=0.5 of v1] {(b)};
      \end{scope}
      \begin{scope}[shift={(6, 0)}]
        \drawgraph
        \node[vertex,fill=blue,text=white] at (v3) {$1$};
        \node[vertex,fill=red,text=white] at (v2) {$2$};
        \draw[thick,->,blue](v3) to[out=270,in=40] (v5);
        \draw[thick,->,red](v2) to[out=220,in=120] (v4);
        \node[below=0.5 of v1] {(c)};
      \end{scope}
      \begin{scope}[shift={(9, 0)}]
        \drawgraph
        \node[vertex,fill=blue,text=white] at (v2) {$1$};
        \node[vertex,fill=red,text=white] at (v1) {$2$};
        \draw[thick,->,blue](v2) to[out=-30,in=40] (v5);
        \draw[thick,->,red](v1) to[out=270,in=180] (v4);
        \node[below=0.5 of v1] {(d)};
      \end{scope}
      \draw[very thick,->](2,-1.5) -- (2.8,-1.5);
      \draw[very thick,->](5,-1.5) -- (5.8,-1.5);
      \draw[very thick,->](8,-1.5) -- (8.8,-1.5);
    \end{tikzpicture}
    }
    \caption{
    Swap operation.
    The last two steps are omitted because of just moving two agents toward their goal.
    }
    \label{fig:swap}
  \end{figure}
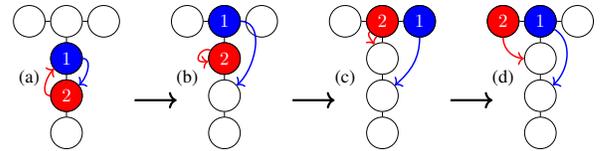
}

{
  \setlength{\tabcolsep}{5pt}
  \begin{table}[t!]
    \centering
    \small
    \begin{minipage}{0.3\linewidth}
      \includegraphics[width=1\linewidth]{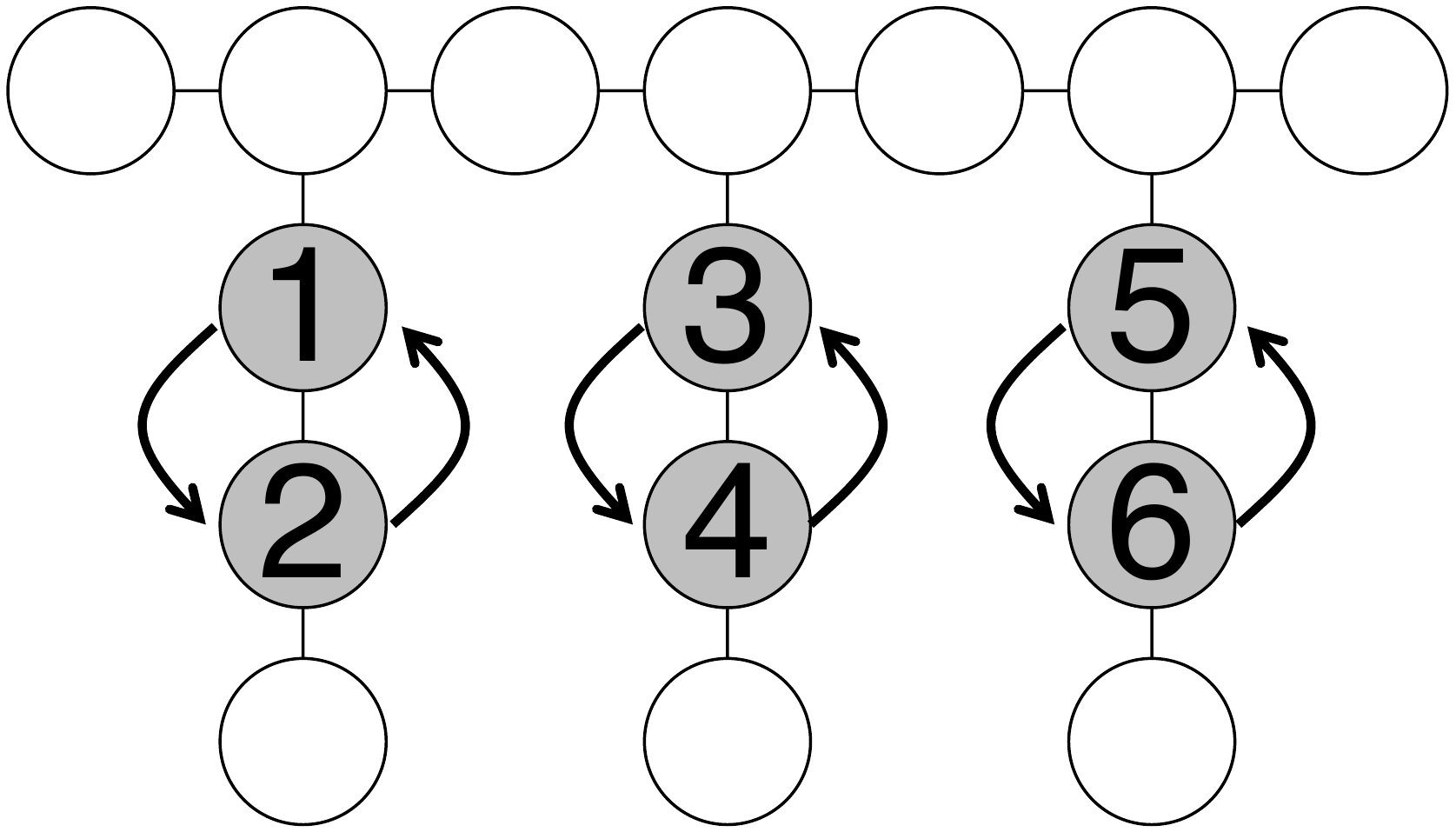}
    \end{minipage}\hspace{0.2cm}
    \begin{tabular}{lrrr}
      \toprule
      & $|A|$$=$2 & $|A|$$=$4 & $|A|$$=$6
      \\\midrule
      w/\cref{algo:pibt} & 128 & 23,907 & 287,440
      \\
      w/\cref{algo:pibt-swap} & 6 & 8 & 8
      \\\bottomrule
    \end{tabular}
    \caption{
      The number of search iterations of LaCAM to solve the instances.
      When $|A|=2$, only agents-$\{1, 2\}$ appear, and so forth.
    }
    \label{table:dislike-example}
  \end{table}
}

\Cref{algo:pibt-swap} extends the \PIBT procedure of \cref{algo:pibt}.
The modification is simple;
if an agent $i$ and neighboring agent $j$ are judged to require swapping locations (\cref{algo:pibt-swap:identify}), $i$ reverses the order of candidate vertices $C$ (\cref{algo:pibt-swap:reverse});
that is, $i$ tries to be apart from $g_i$.
Then, if $i$ successfully moves to the first vertex in the candidates $C$, $i$ \emph{pulls} $j$ to the current occupying vertex (\cref{algo:pibt-swap:pull}).
With an appropriate implementation of the function $\funcname{swap\_required\_and\_possible}$, PIBT does not fall into the livelock of \cref{fig:pibt-failure}, rather, it can generate a sequence of configurations shown in \cref{fig:swap}.

The $\funcname{swap\_required\_and\_possible}$ function is a pattern detector and implementation-depending.
\emph{We do not aim at designing complete detectors because pitfalls of the detectors can be complemented by LaCAM}.
However, a well-tuned implementation can relax the search effort of LaCAM.
Below, we illustrate our example implementation, while omitting tiny fine-tunings.

\subsubsection{Pattern Detector Implementation}
Assume that the detector is called for $i$.
Assume further that another agent $j$ is on $C[1]$ for $i$ at \cref{algo:pibt-swap:sort} of \cref{algo:pibt-swap}, such that the degree is 2 or less.
Our detector uses two emulations.

The first emulation asks about the necessity of the swap.
This is done by continuously moving $i$ to $j$'s location while moving $j$ to another vertex not equal to $i$'s location, ignoring the other agents.
The emulation stops in two cases:
\emph{(i)}~The swap is \emph{not required} when $j$'s location has a degree of more than two.
\emph{(ii)}~The swap is \emph{required} when $j$'s location has a degree of one, or, when $i$ reaches $g_i$ while $j$'s nearest neighboring vertex toward its goal is $g_i$.

If the swap is required, the second emulation asks about the possibility of the swap.
This is done by reversing the emulation direction;
that is, continuously moving $j$ to $i$'s location while moving $i$ to another vertex.
It stops in two cases:
\emph{(i)}~The swap is \emph{possible} when $i$'s location has a degree of more than two.
\emph{(ii)}~The swap is \emph{impossible} when $i$ is on a vertex with degree of one.

The function returns $j$ when the swap is required and possible.
For instance, in configurations of \cref{fig:swap}(a,b), the swap is required for both agents.
However, the swap is possible only for agent-$1$, then, the order of candidate vertices is reversed for agent-$1$ (\cref{algo:pibt-swap:reverse}).
Consequently, \cref{algo:pibt-swap} generates configurations of \cref{fig:swap}(b,c).
An exception is the case of \cref{fig:swap}c, where agent-$2$ needs to reverse its candidates to generate a configuration of \cref{fig:swap}d.
We note that this case is also possible to be detected by applying the two emulations.

\section{Evaluation}
This section empirically assesses the two improvements, comprising:
\emph{(i)}~how the improved configuration generator reduces planning effort,
\emph{(ii)}~how \lacamstar refines solution,
\emph{(iii)}~how discarding redundant nodes speeds up the convergence,
\emph{(iv)}~evaluation with small complicated instances,
\emph{(v)}~evaluation with the MAPF benchmark,
\emph{(vi)}~comparison with another anytime MAPF algorithm, and
\emph{(vii)}~evaluation with extremely dense scenarios.

\paragraph{Setup.}
The experiments were run on a desktop PC with Intel Core i7-7820X \SI{3.6}{\giga\hertz} CPU and \SI{32}{\giga\byte} RAM.
A maximum of 16 different instances were run in parallel using multi-threading.
\lacamstar was coded in C++.
All experiments used four-connected grid maps retrieved from the MAPF benchmark~\cite{stern2019def}.
Unless mentioned, this section uses a timeout of \SI{30}{\second} for solving MAPF.
Baseline MAPF algorithms are summarized in \cref{fig:percentage}.
Their implementation details are available in the appendix.

\paragraph{Effect of Improved Configuration Generator.}
\Cref{table:dislike-example} presents the number of search iterations of LaCAM on an instance that requires ``swap,'' using a vanilla PIBT (\cref{algo:pibt}) and the improved one (\cref{algo:pibt-swap}) as a configuration generator.
\Cref{table:warehouse} further compares the generators with larger instances.
The results show that \cref{algo:pibt-swap} dramatically reduced the search iterations of LaCAM, contributing to smaller computation time in large instances.
Note however that the pattern detector has runtime overhead, as seen in $|A|=100$ of \cref{table:warehouse}.

{
  \setlength{\tabcolsep}{2pt}
  \newcommand{\ci}[1]{{\tiny(#1)}}
  \begin{table}[t!]
    \centering
    \small
    \begin{tabular}{rrlrlrlrl}
      \toprule
      & \multicolumn{4}{c}{search iterations}
      & \multicolumn{4}{c}{runtime (\SI{}{\milli\second})}
      \\\cmidrule(lr){2-5}\cmidrule(lr){6-9}
      $|A|$
      & \multicolumn{2}{c}{w/\cref{algo:pibt}}
      & \multicolumn{2}{c}{w/\cref{algo:pibt-swap}}
      & \multicolumn{2}{c}{w/\cref{algo:pibt}}
      & \multicolumn{2}{c}{w/\cref{algo:pibt-swap}}
      \\
      \midrule
      100
      & 374 & \ci{344,54468}
      & 366	& \ci{338,401}
      & 65 & \ci{31,1218}
      & 112 & \ci{34,216}
      \\
      300
      & 54802 & \ci{388,369131}
      & 392 & \ci{357,482}
      & 3049 & \ci{291,18858}
      & 301 &	\ci{187,409}
      \\
      500
      & 181459 & \ci{44534,268724}
      & 410 & \ci{391,432}
      & 18063 & \ci{4598,29820}
      & 500 & \ci{347,574}
      \\
      \bottomrule
    \end{tabular}
    \caption{
        Effect of configuration generators.
        For each $|A|$, median, min, and max scores are presented for instances solved by both algorithms among 25 instances retrieved from \protect\cite{stern2019def}, on \mapname{warehouse-20-40-10-2-1}, illustrated in \cref{fig:mapf-bench-short}.
    }
    \label{table:warehouse}
  \end{table}
}

{
  \setlength{\tabcolsep}{0pt}
  \newcommand{\entry}[1]{
    \begin{minipage}{0.42\linewidth}
      \centering
      \includegraphics[width=0.9\linewidth]{fig/raw/#1.pdf}
    \end{minipage}
  }
  \begin{figure}[t!]
    \centering
    \small
    \begin{tabular}{lclc}
      & \mapname{tunnel}, $|A|$$=$$4$ && \mapname{loop-chain}, $|A|$$=$$5$
      \\
      \rotatebox{90}{loss}
      &
      \entry{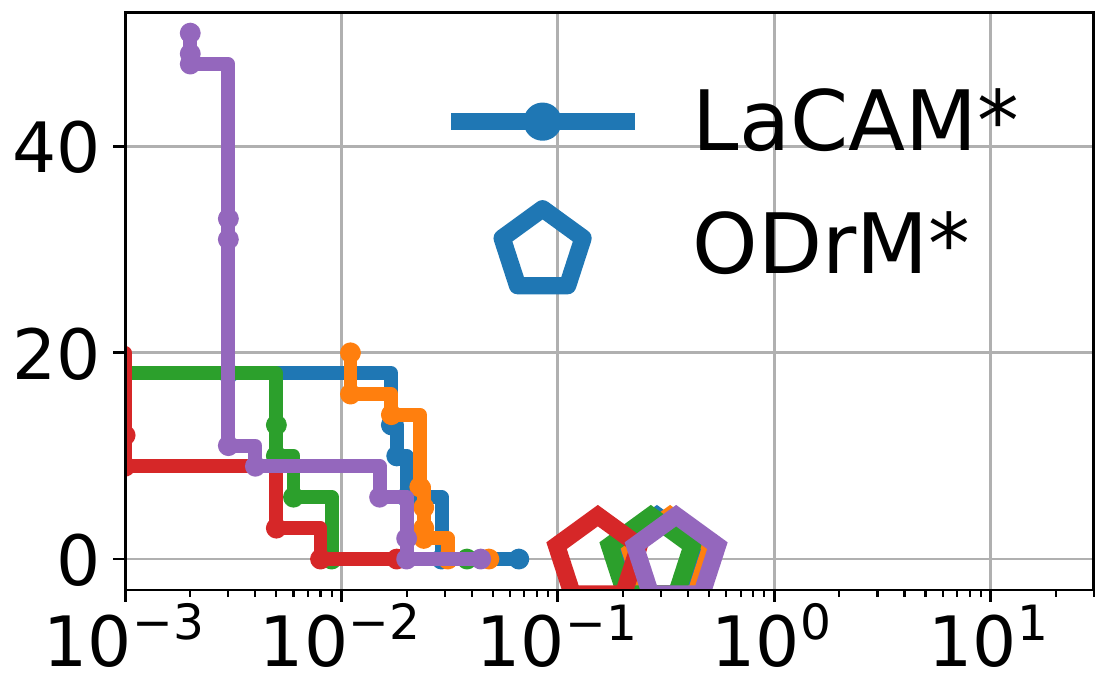}
      &
      \rotatebox{90}{loss}
      &
      \entry{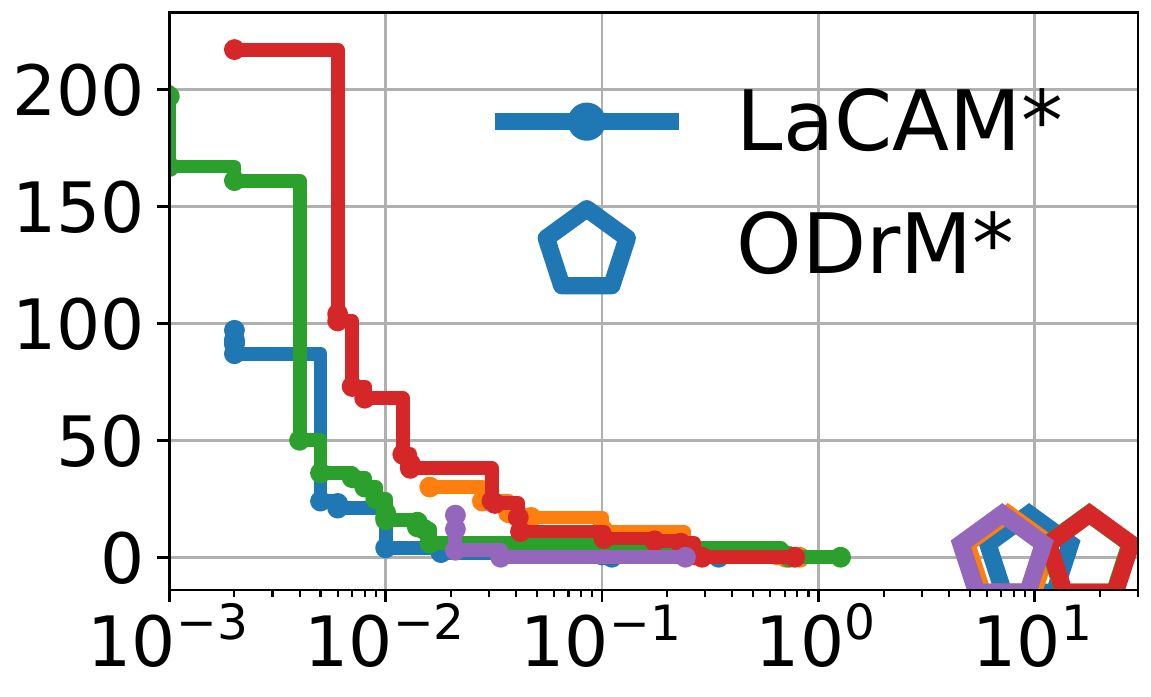}
      \\
      &time(\SI{}{\milli\second})
      &
      &time(\SI{}{\milli\second})
      \smallskip\\
      &\mapname{random-32-32-20}, $|A|$$=$$30$
      &
      &\mapname{random-32-32-20}, $|A|$$=$$100$
      \\
      \rotatebox{90}{loss}
      &
      \entry{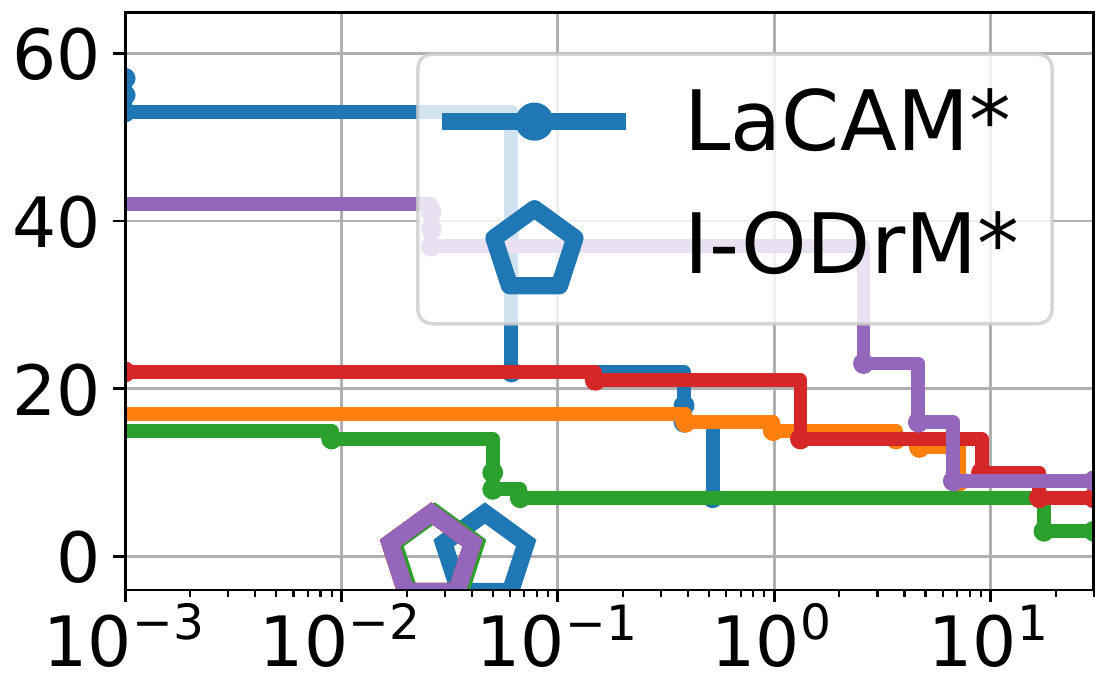}
      &
      \rotatebox{90}{loss}
      &
      \entry{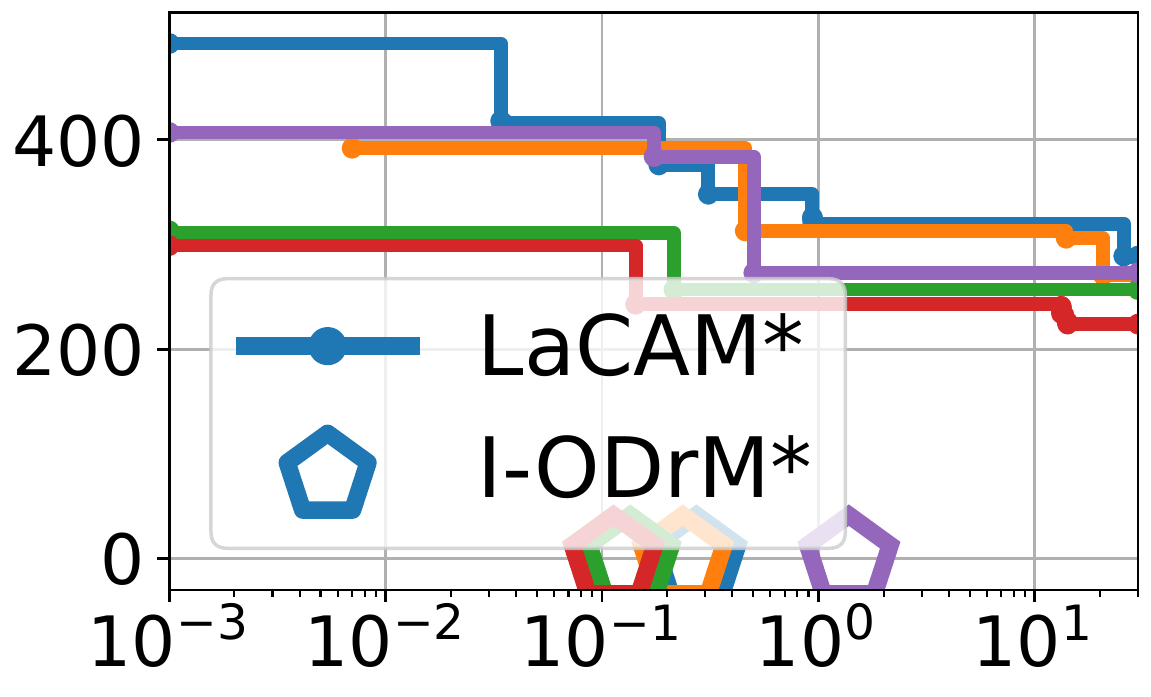}
      \\
      &time(\SI{}{\milli\second})
      &
      &time(\SI{}{\milli\second})
    \end{tabular}
    \caption{
     Refinement of \lacamstar.
     Three maps were used, shown in \cref{table:small-complicated,fig:mapf-bench-short}.
     For each chart, five identical instances were used where starts and goals were set randomly.
     The optimization was for sum-of-loss.
     ``loss'' shows the gaps from scores of \ao{(I-)ODrM}.
     In \mapname{random-32-32-20}, the bounded sub-optimal version with sub-optimality of 1.5 was used because \ao{ODrM} failed to solve the instances.
     \lacamstar used \cref{algo:pibt} as a configuration generator.
    }
    \label{fig:refine}
  \end{figure}
}

{
  \setlength{\tabcolsep}{2pt}
  \renewcommand{\c}[1]{\multicolumn{1}{c}{#1}}
  \begin{table}[t!]
    \centering
    \small
    \begin{tabular}{lrrrrrr}
      \toprule
      & \c{\mapname{tree}}
      & \c{\mapname{corners}}
      & \c{\mapname{tunnel}}
      & \c{\mapname{string}}
      & \c{\mapname{loop-chain}}
      & \c{\mapname{connector}}
      \\
      \midrule
      no discard & 10K & 2M & 410M & 19M & \NA & \NA
      \\
      w/\cref{algo:pibt} & 1K & 28K & 287K & 103K & \NA & \NA
      \\
      w/\cref{algo:pibt-swap} & \w{1K} & \w{1K} & \w{199K} & \w{103K} & \NA & \NA
      \\
      \bottomrule
    \end{tabular}
    \caption{
    The number of search iterations for termination.
    ``no discard'' is blue-lines omitted version of \lacamstar.
    The metric was for makespan.
    The instances are displayed in \cref{table:small-complicated}.
    In the last two instances, \lacamstar did not terminate before the timeout.
    }
    \label{table:discarding-nodes}
  \end{table}
}

{
  \setlength{\tabcolsep}{1.2pt}
  \renewcommand{\arraystretch}{0.9}
  \newcommand{\drawmap}[1]{
    \multicolumn{2}{c}{
      \begin{minipage}{0.12\linewidth}
        \centering
        \includegraphics[width=1.0\linewidth]{fig/raw/maps/#1}
      \end{minipage}
    }
  }
  \newcommand{\lines}{
    \cmidrule(lr){2-3}
    \cmidrule(lr){4-5}
    \cmidrule(lr){6-7}
    \cmidrule(lr){8-9}
    \cmidrule(lr){10-11}
    \cmidrule(lr){12-13}
  }
  \begin{table}[t!]
    \centering
    \scriptsize
    \begin{tabular}{lrrrrrrrrrrrrc}
      \toprule
      & \multicolumn{2}{c}{\mapname{tree}}
      & \multicolumn{2}{c}{\mapname{corners}}
      & \multicolumn{2}{c}{\mapname{tunnel}}
      & \multicolumn{2}{c}{\mapname{string}}
      & \multicolumn{2}{c}{\mapname{loop-chain}}
      & \multicolumn{2}{c}{\mapname{connector}}
      \\
      \begin{minipage}{0.1\linewidth}
      unit~of\\time:~\SI{}{\milli\second}
      \end{minipage}
      & \drawmap{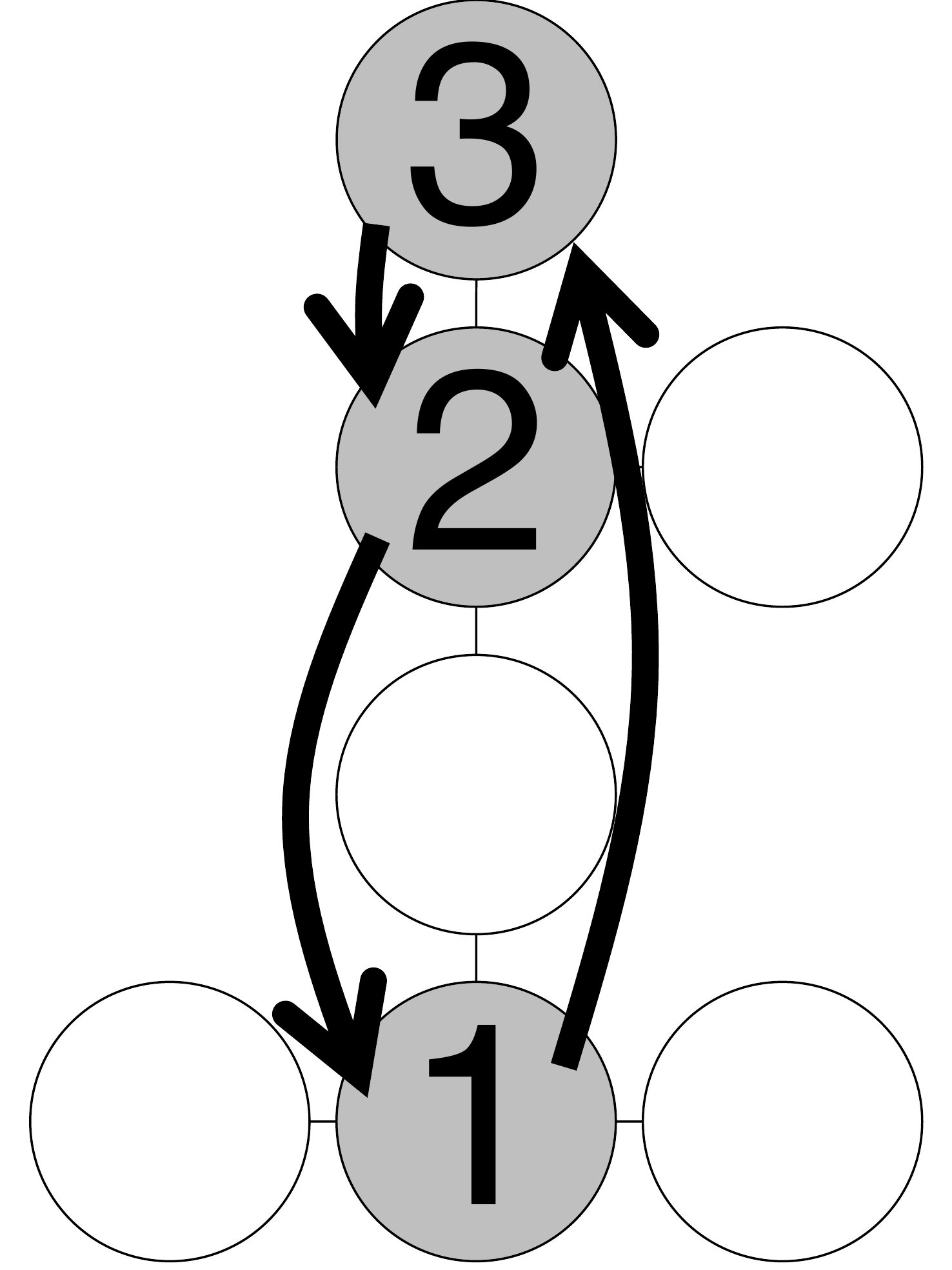}
      & \drawmap{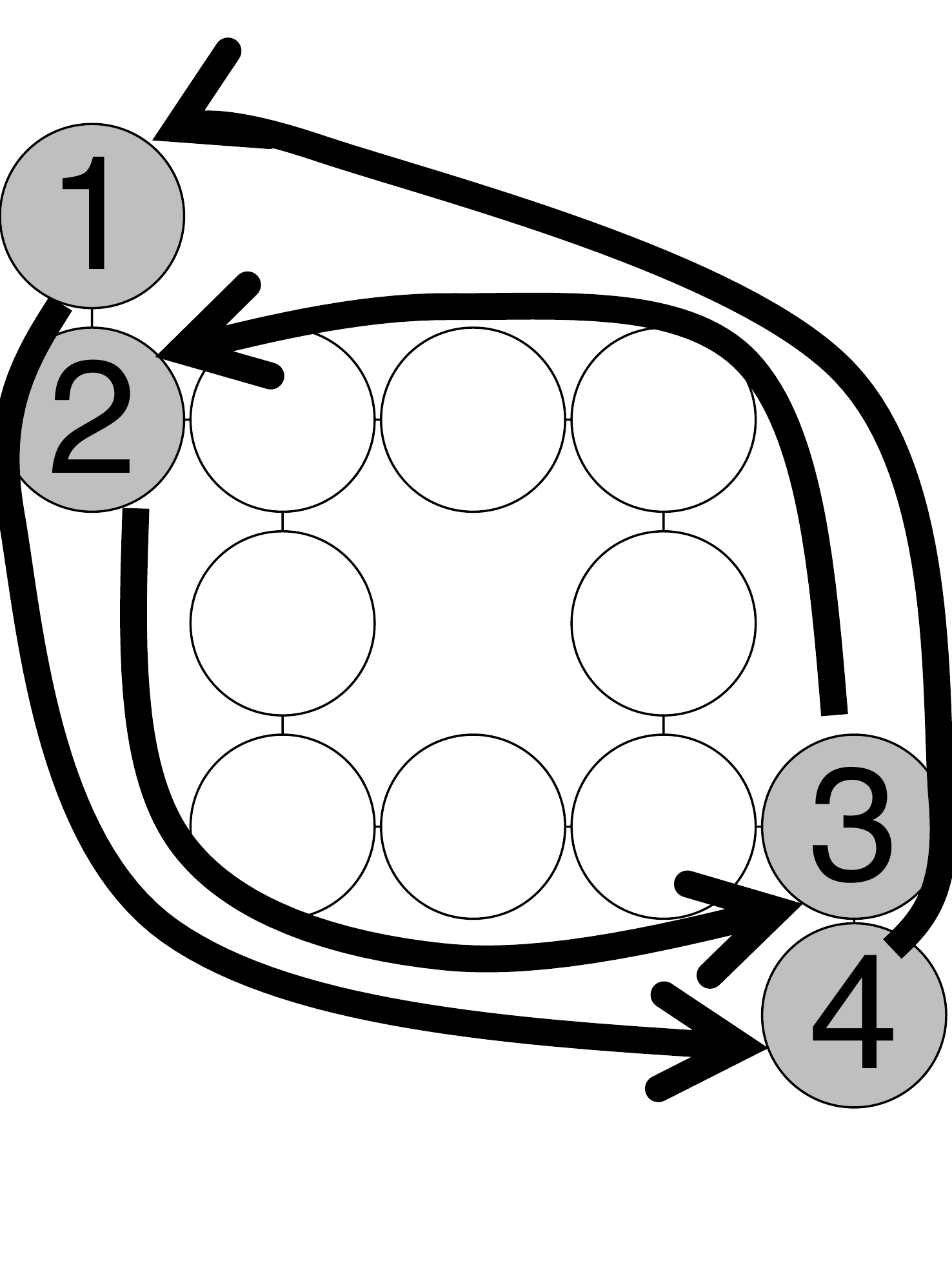}
      & \drawmap{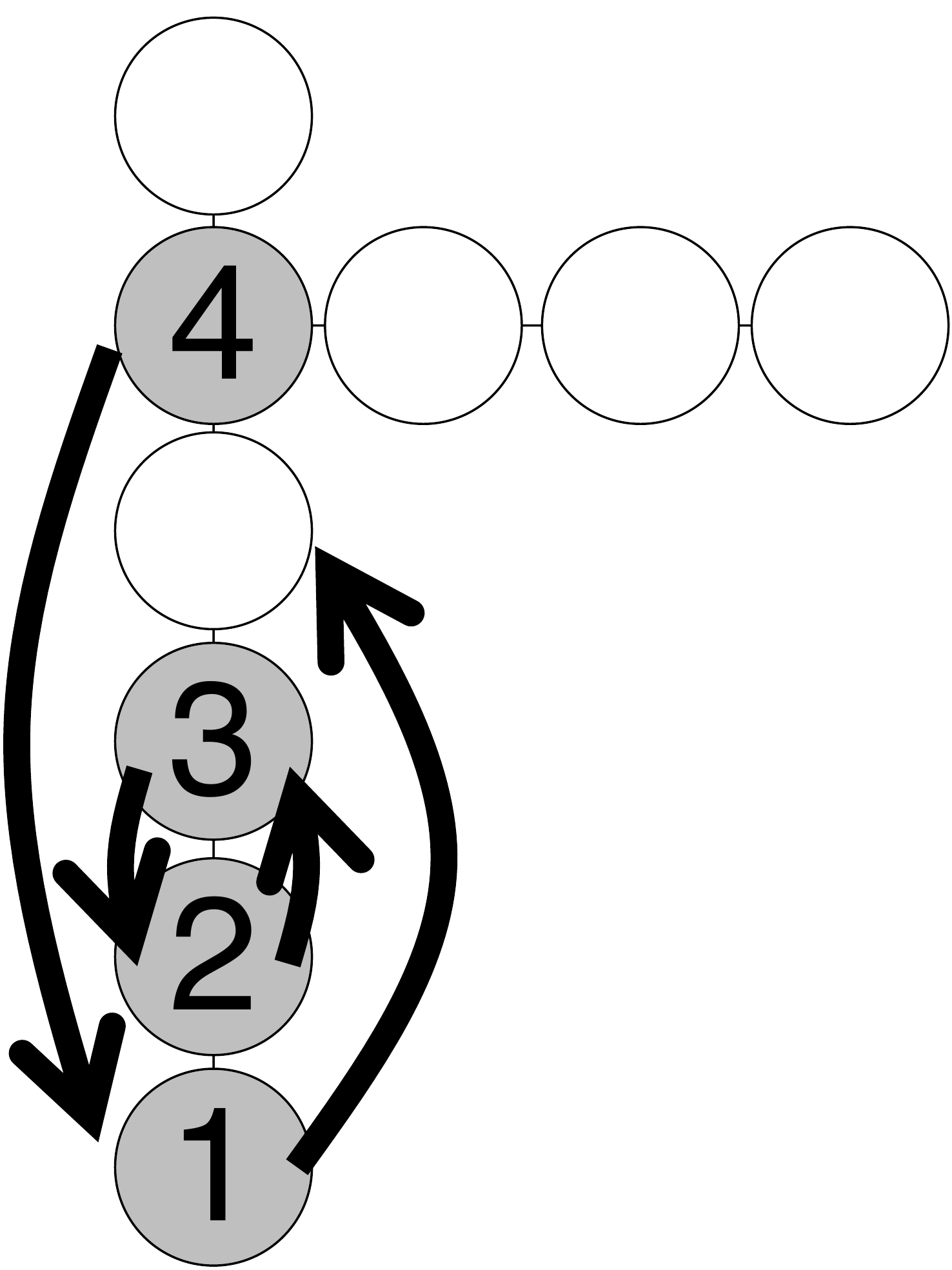}
      & \drawmap{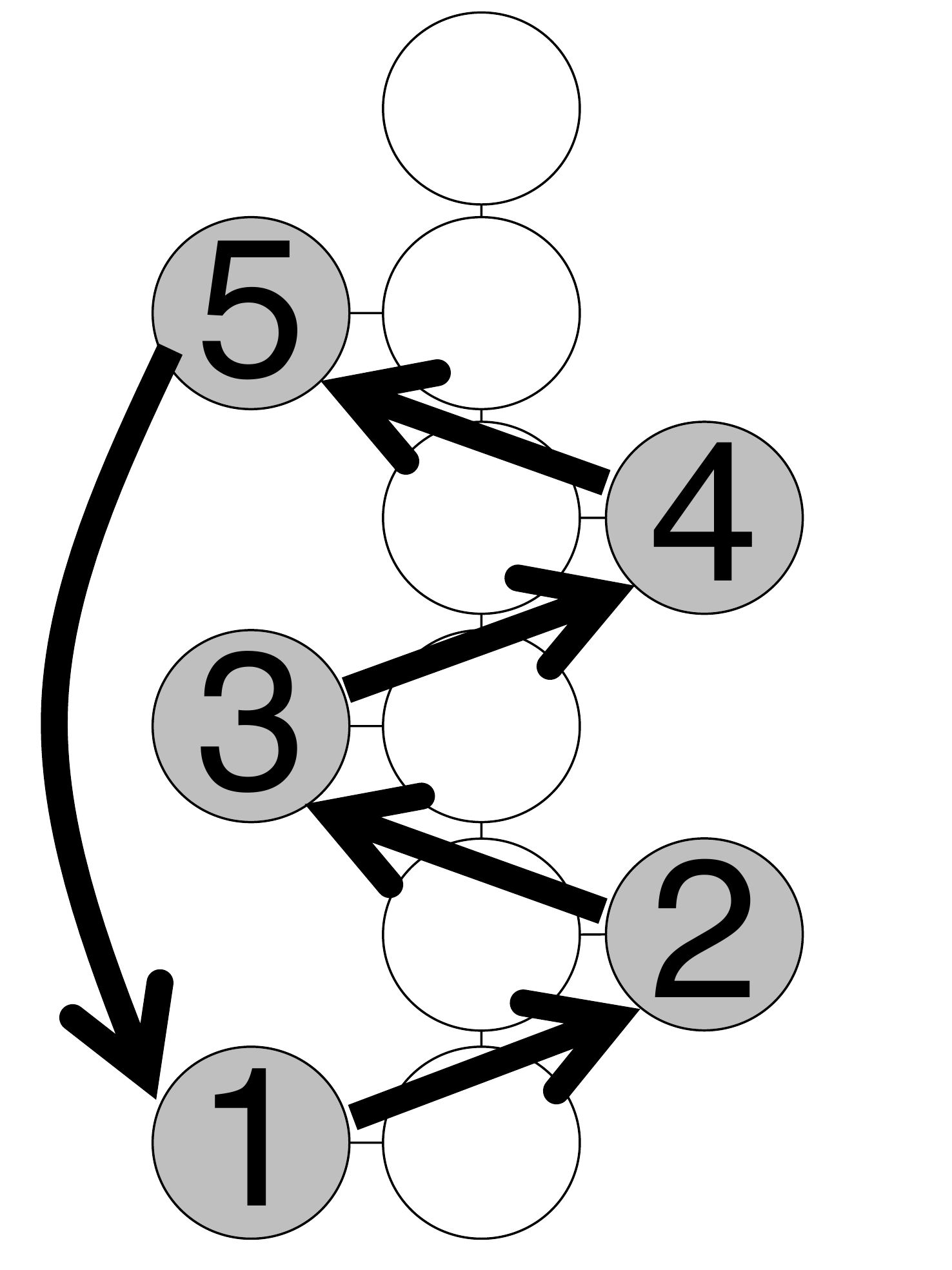}
      & \drawmap{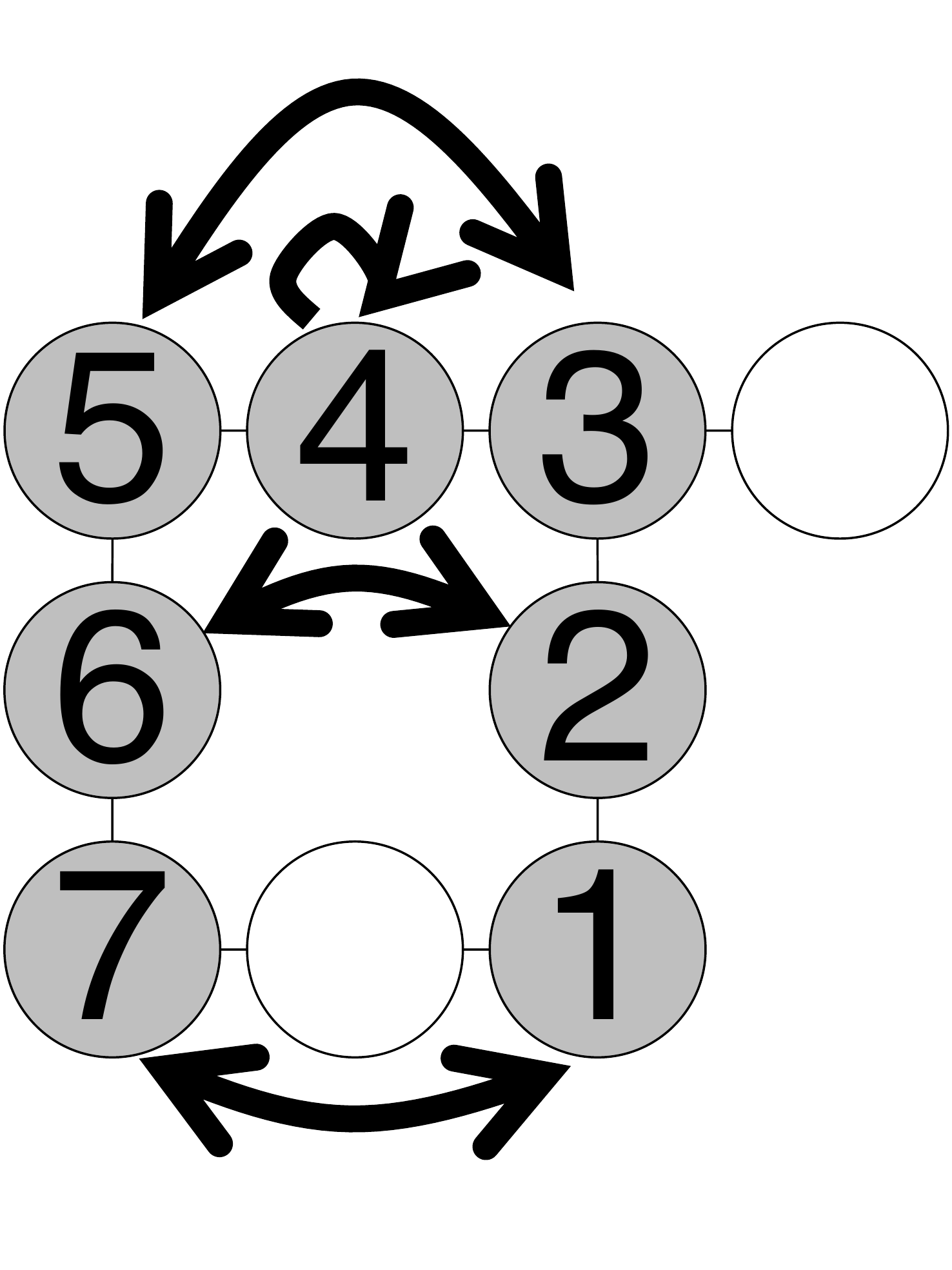}
      & \drawmap{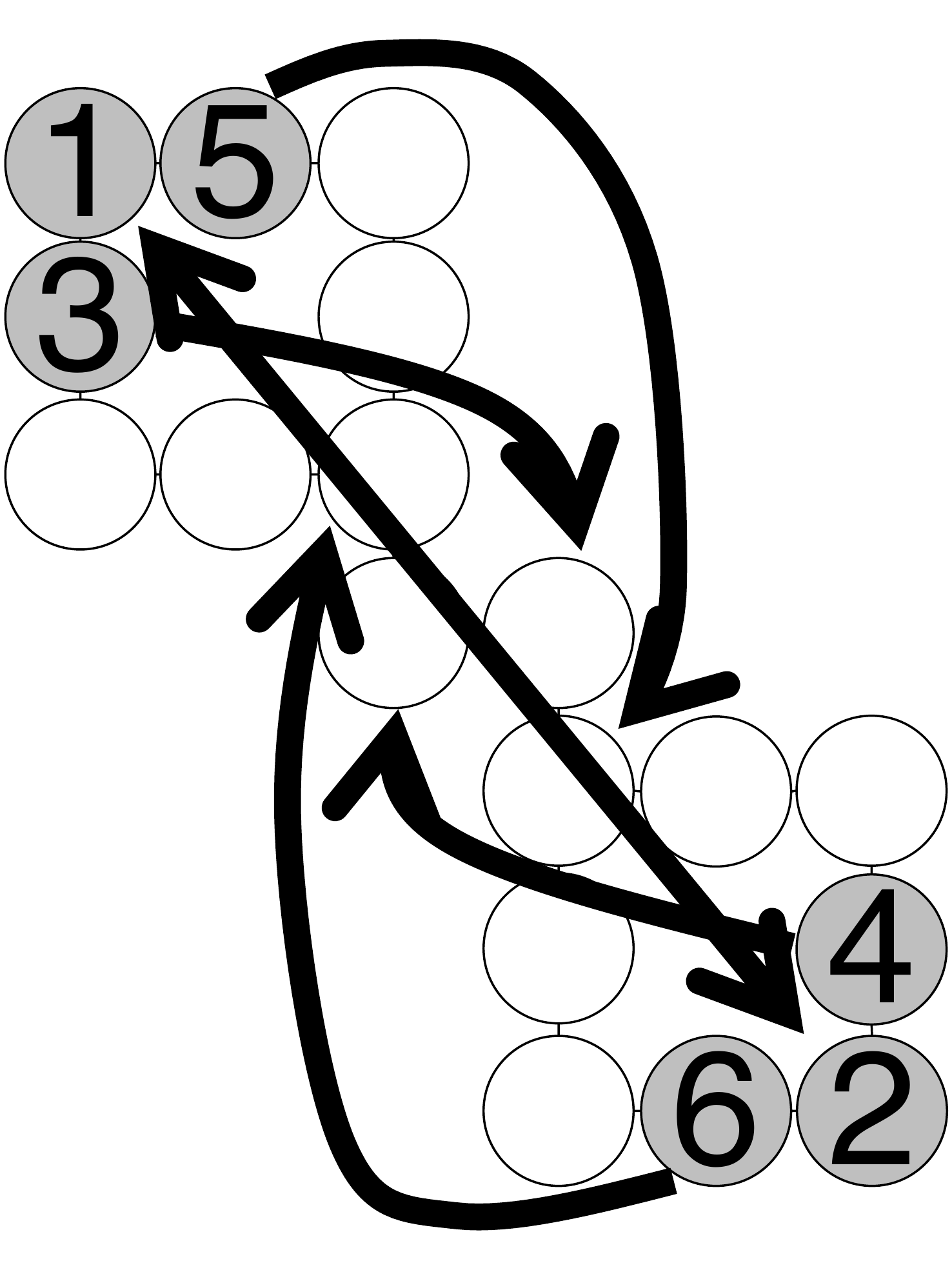}
      \\
      &
      time & s-opt &
      time & s-opt &
      time & s-opt &
      time & s-opt &
      time & s-opt &
      time & s-opt &
      solved
      \\
      \lines
      \cmidrule(lr){14-14}
      \lacamstar
      & 0 & 1.20 & 0 & 1.23 & 0 & 1.41 & 0 & 1.81 & 2 & 6.58 & 0 & 1.62 & \multirow{2}{*}{\w{6/6}}
      \\
      after \SI{1}{\second}
      & 0 & 1.00 & 2 & 1.00 & 6 & 1.00 & 7 & 1.00 & 578 & 1.35 & 226 & 1.00 &
      \\
      \cmidrule(){1-14}
      \astar
      & 0 & 1.00 & 0 & 1.00 & 30 & 1.00 & 27 & 1.00 & 11125 & 1.00 & \NA & \NA & 5/6
      \\
      ODrM$^\ast$
      & 5 & 1.00 & 2 & 1.00 & 396 & 1.00 & 402 & 1.00 & \NA & \NA & \NA & \NA & 4/6
      \\
      I-ODrM$^\ast$
      & 1 & 1.00 & 0 & 1.50 & 70 & 1.07 & 2 & 1.25 & \NA & \NA & \NA & \NA & 4/6
      \\
      CBS
      & 71 & 1.00 & 0 & 1.00 & \NA & \NA & 149 & 1.00 & \NA & \NA & \NA & \NA & 3/6
      \\
      EECBS
      & 2 & 1.00 & 1 & 1.00 & \NA & \NA & 0 & 1.00 & \NA & \NA & \NA & \NA & 3/6
      \\
      \midrule
      OD
      & 0 & 1.00 & 0 & 1.88 & 14 & 2.73 & 0 & 1.25 & 2133 & 31.22 & 5 & 1.50 & \w{6/6}
      \\
      LaCAM
      & 0 & 1.17 & 1 & 2.12 & 92 & 2.00 & 0 & 2.25 & 55 & 17.83 & 0 & 1.56 & \w{6/6}
      \\
      PP
      & \NA & \NA & 0 & 1.00 & \NA & \NA & 0 & 1.00 & \NA & \NA & \NA & \NA & 2/6
      \\
      LNS2
      & \NA & \NA & 0 & 1.00 & \NA & \NA & 0 & 1.00 & \NA & \NA & 29 & 1.00 & 3/6
      \\
      PIBT
      & \NA & \NA & \NA & \NA & \NA & \NA & \NA & \NA & \NA & \NA & \NA & \NA & 0/6
      \\
      \pibtp
      & 0 & 3.50 & 0 & 1.25 & 0 & 4.07 & 0 & 2.12 & \NA & \NA & 0 & 1.81 & 5/6
      \\
      \midrule
      BCP
      & 194 & - & 150 & - & \NA & - & 117 & - & \NA & - & \NA & - & 3/6
      \\
      \bottomrule
    \end{tabular}
    \caption{
      Results of the small complicated instances.
      ``s-opt'' is makespan normalized by optimal ones.
      The minimum is one.
      The sum-of-loss version appears in the appendix.
      Two rows show results of \lacamstar: \emph{(i)}~scores for initial solutions and \emph{(ii)}~solution quality at \SI{1}{\second} and the runtime when that solution was obtained;
      they are an average of $10$ trials with different random seeds.
      Algorithms are categorized into \lacamstar, those optimizing makespan, sub-optimal ones, and BCP optimizing another metric (i.e., flowtime).
    }
    \label{table:small-complicated}
  \end{table}
}

{
  \setlength{\tabcolsep}{0.5pt}
  \newcommand{\entry}[3]{
    \begin{minipage}{0.185\linewidth}
      \centering
      \begin{tabular}{ll}
        \begin{minipage}{0.65\linewidth}
          \baselineskip=7pt
          {\tiny\mapname{#1}}\\
          {\tiny #2 (#3)}
        \end{minipage} &
        \begin{minipage}{0.28\linewidth}
          \includegraphics[width=1\linewidth]{fig/raw/maps/#1}
        \end{minipage}
      \end{tabular}
      \\
      \includegraphics[width=1\linewidth,height=0.70\linewidth]{fig/raw/mapf-bench/success_rate_#1}\\
      \includegraphics[width=1\linewidth,height=0.70\linewidth]{fig/raw/mapf-bench/runtime_#1}\\
      \includegraphics[width=1\linewidth,height=0.70\linewidth]{fig/raw/mapf-bench/sum_of_loss_#1}\\
    \end{minipage}
  }
  \newcommand{\labels}{
    \begin{minipage}{0.03\linewidth}
      \small
      \begin{tikzpicture}
        \node[rotate=90](l3) at (0, 6.0) {success rate};
        \node[rotate=90](l2) at (0, 3.8) {runtime (sec)};
        \node[rotate=90](l1) at (0, 1.3) {sum-of-loss / LB};
        \node[] at (0, 8) {};
      \end{tikzpicture}
    \end{minipage}
  }
  \begin{figure*}[t!]
    \centering
    \begin{tabular}{cccccc}
      \labels
      & \entry{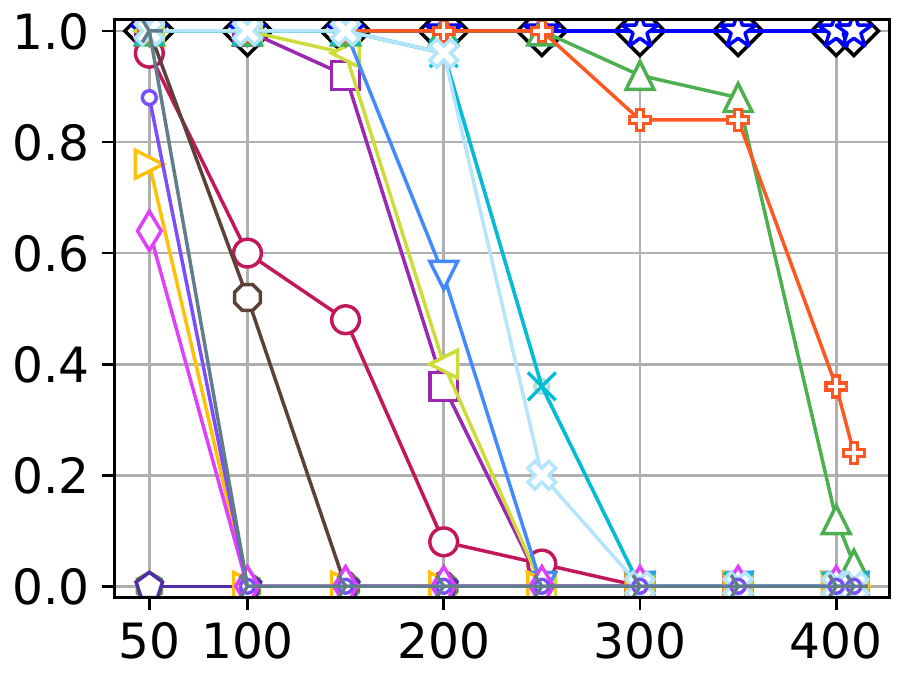}{32x32}{$|V|$$=$819}
      & \entry{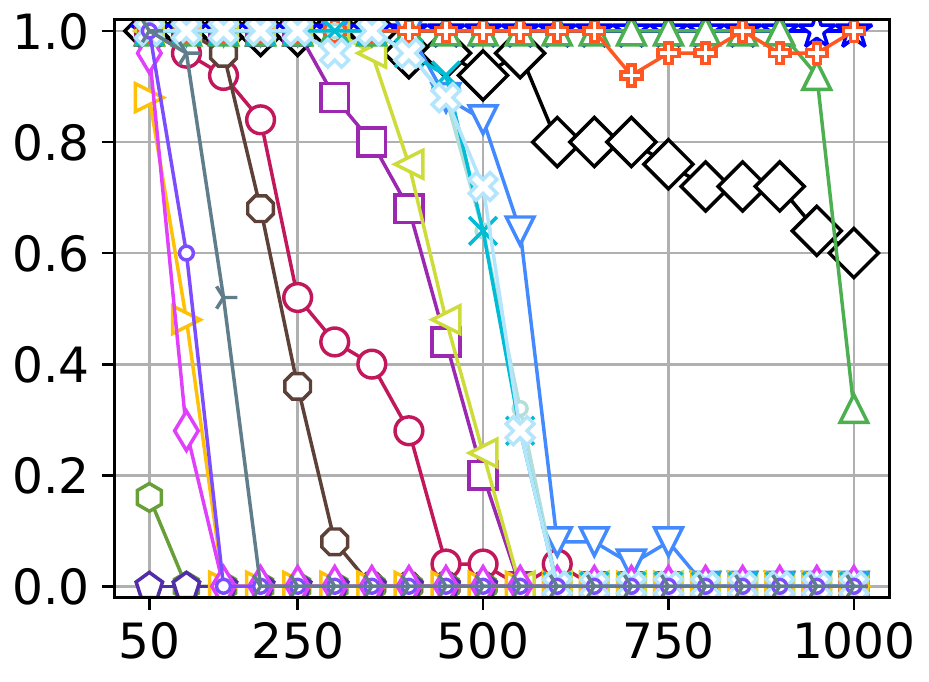}{64x64}{$|V|$$=$3,270}
      & \entry{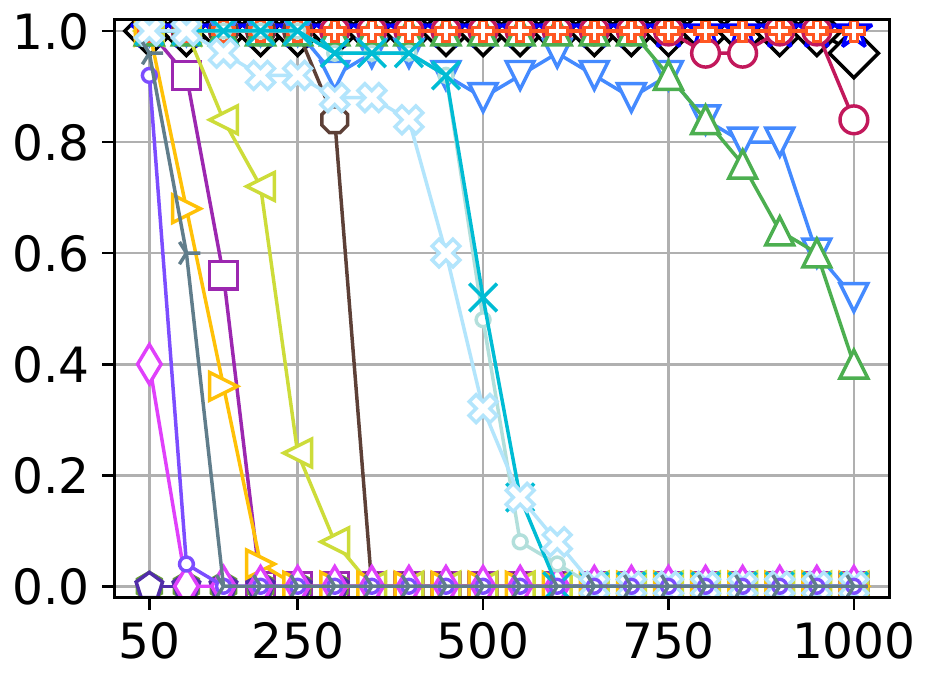}{530x481}{$|V|$$=$43,151}
      & \entry{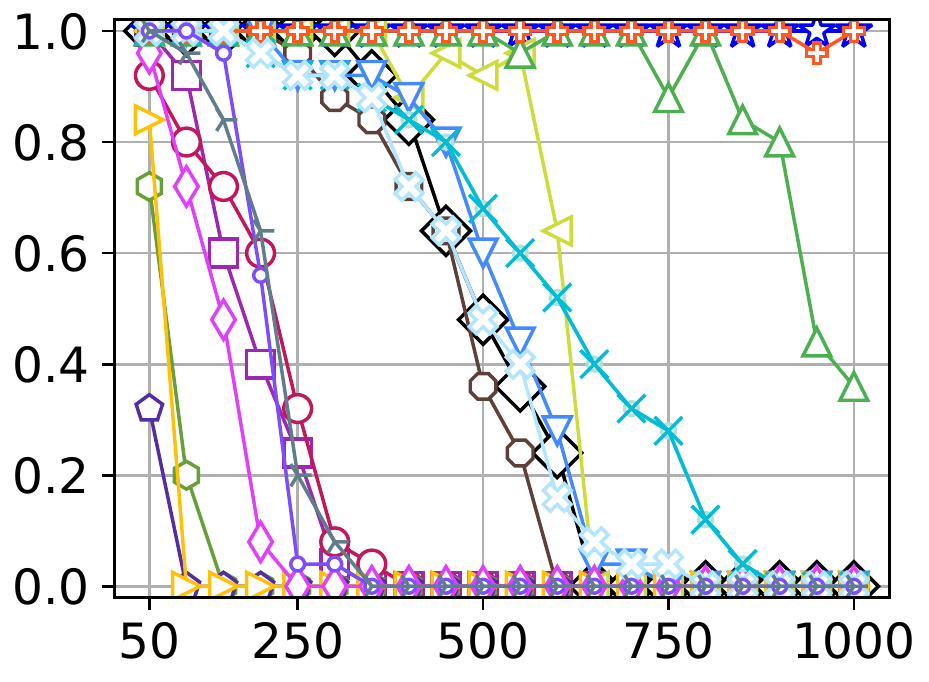}{321x123}{$|V|$$=$22,599}
      & \entry{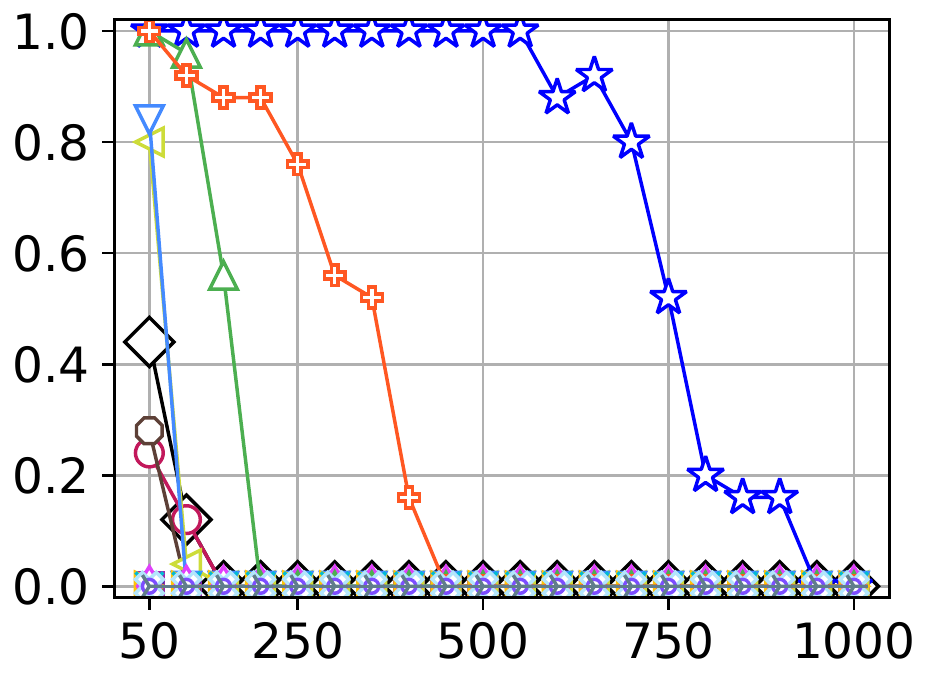}{128x128}{$|V|$$=$8,191}
      \\
      \multicolumn{6}{c}{agents:~$|A|$}
      \smallskip\\
      \multicolumn{6}{c}{\includegraphics[width=0.8\linewidth]{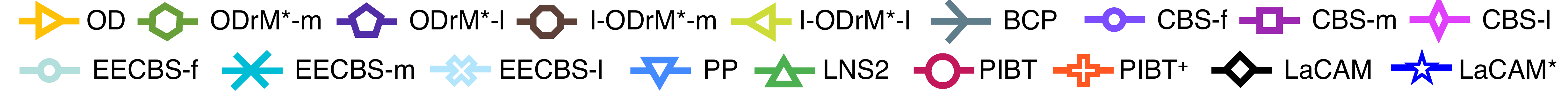}}
    \end{tabular}
    \caption{
    Results of the MAPF benchmark.
    Scores of sum-of-loss are normalized by $\sum_{i\in A}\dist(s_i, g_i)$.
    For runtime and sum-of-loss, median, min, and max scores of solved instances within each solver are displayed.
    Scores of \lacamstar are from initial solutions.
    }
    \label{fig:mapf-bench-short}
  \end{figure*}
}

\paragraph{Refinement of \lacamstar.}
\Cref{fig:refine} shows how \lacamstar refines solutions.
As baselines, we used scores of a complete and optimal algorithm called (I-)\ao{ODrM}~\cite{wagner2015subdimensional}.
In the small instances, \lacamstar quickly found initial solutions and converged to optimal ones.
Meanwhile, the convergence speed was slow in large instances with many agents.
This is due to finding new connections between known configurations becoming rare, hence reducing the chance of rewriting the search tree.

\paragraph{Effect of Discarding Redundant Nodes.}
\Cref{table:discarding-nodes} shows how discarding redundant search nodes (blue lines of \cref{algo:star}) affects the search to identify optimal solutions.
Regardless of the generators, the discarding dramatically reduced the search effort.
The reduction was larger with \cref{algo:pibt-swap} because initial solutions can be found with smaller search iterations than \cref{algo:pibt}.
Note that without the discarding, the numbers of search iterations are equivalent between \cref{algo:pibt} and \cref{algo:pibt-swap} because the search spaces are identical.
In the remaining, \lacamstar uses \cref{algo:pibt-swap} while LaCAM denotes the original implementation that uses \cref{algo:pibt}.

\paragraph{Small Complicated Instances.}
\Cref{table:small-complicated} shows the results of \lacamstar with instances retrieved from \cite{luna2011push}.
Overall, it immediately found not only initial solutions but also (near-)optimal ones.
In contrast, the baselines failed some instances or returned low-quality solutions.

\paragraph{MAPF Benchmark.}
We tested \lacamstar on the MAPF benchmark that includes 33 maps, each having 25 ``random scenarios'' which specify start-goal pairs.
From each scenario, we extracted instances by increasing the number of agents by 50 up to the maximum (1,000 in most cases) and obtained 13,900 instances in total.
The percentage of solved instances is summarized in \cref{fig:percentage}.
\Cref{fig:mapf-bench-short} presents partial results for each map.
\lacamstar only failed in the instances of \mapname{maze-128-128-1} and sub-optimally solved all the other instances within \SI{10}{\second}, outperforming the other algorithms.
The failures might be reduced by improving the pattern detector;
however, we consider such implementations are too optimized for the benchmark.
As shown in \cref{fig:mapf-bench-refine}, the refinement was steady but not dramatic due to the same reason of \cref{fig:refine}.
Further discussions are available in the appendix.

\paragraph{Comparison with Anytime MAPF Solver.}
We compared \lacamstar with AFS~\cite{cohen2018anytime}, a CBS-based anytime MAPF solver that guarantees to converge optima.
\Cref{table:anytime} summarizes the results.
Contrary to \lacamstar, AFS can obtain plausible solutions from the beginning, however, it compromises scalability.
We consider this quality gap can be overcome by developing better generators other than PIBT.

\paragraph{Extremely Dense Scenarios.}
\cref{fig:dense} reports \lacamstar in very congested scenarios that existing solvers mostly fail.
Even with such challenging cases, \lacamstar solved many instances, demonstrating its excellent scalability.

\section{Conclusion and Discussion}
The primary challenge of MAPF is to maintain solvability and solution quality while suppressing planning efforts.
To break this tradeoff, this paper presented two enhancements to LaCAM, namely, \lacamstar which eventually converges to optima and an effective configuration generator.
The enhancements were thoroughly assessed, achieving remarkable results.
From the empirical evidence, we believe that \lacamstar has developed a new frontier in MAPF.

\paragraph{Related Work.}
LaCAM$^{(\ast)}$ relates to partial successor expansion during the search, as seen in~\cite{goldenberg2014enhanced,wagner2015subdimensional}, because it also generates a subset of successors, but differs in the use of constraints and a configuration generator.
Anytime MAPF algorithms that converge to optima have been studied~\cite{standley2011complete,cohen2018anytime,vedder2021x}.
However, their scalability is limited;
they often fail to derive initial solutions, as we empirically saw.
Techniques to refine arbitrary MAPF solutions have also been studied~\cite{surynek2013redundancy,de2014push,okumura2021iterative,li2021anytime} but they do not ensure optimality.
Rewriting the search tree structure is popular in optimal motion planning~\cite{karaman2011sampling,shome2020drrt}, by which \lacamstar is partially inspired.
Incorporating swap into PIBT is inspired by sub-optimal rule-based MAPF algorithms~\cite{luna2011push,de2014push}.
Meanwhile, the solution quality of rule-based approaches themselves is often severely compromised.

\paragraph{Future Directions.}
We are interested in more effective configuration generators than PIBT variants which can output near-optimal initial solutions.
Improving the convergence speed of \lacamstar is also important.
Moreover, \lacamstar in MAPF variants, e.g., multi-robot motion planning~\cite{okumura2023sssp}, is worth to be studied.
Other than MAPF, since \lacamstar is just a graph pathfinding algorithm, applying its concept to other planning domains might be exciting.

{
  \begin{figure}
    \includegraphics[width=1\linewidth]{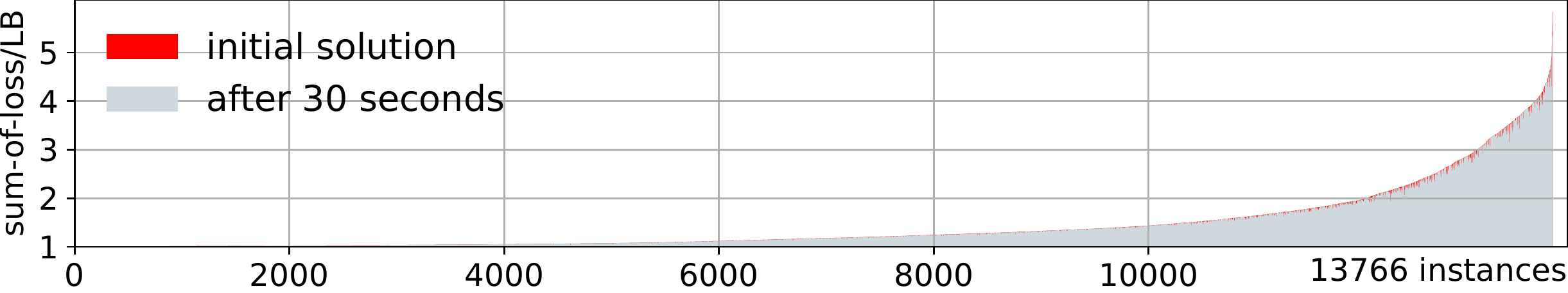}
    \caption{
      Refinement by \lacamstar for the MAPF benchmark.
      On the x-axis, the figure sorts 13,766 solved instances out of 13,900 by initial solution quality and displays the scores in red bars.
      For each instance, we also plot the solution quality at 30 seconds using gray bars.
      Hence, the effect of refinement is visualized by tiny red areas.
    }
    \label{fig:mapf-bench-refine}
  \end{figure}
}

{
  \setlength{\tabcolsep}{1.5pt}
  \begin{table}[t!]
    \centering
    \small
    \begin{tabular}{rrrrrrrrr}
      \toprule
      & \multicolumn{2}{c}{solved(\%)}
      & \multicolumn{2}{c}{time-init(\SI{}{\milli\second})}
      & \multicolumn{2}{c}{loss-init}
      & \multicolumn{2}{c}{loss-\SI{30}{\second}}
      \\
      \cmidrule(lr){2-3}
      \cmidrule(lr){4-5}
      \cmidrule(lr){6-7}
      \cmidrule(lr){8-9}
      $|A|$
      & AFS & \lacamstar
      & AFS & \lacamstar
      & AFS & \lacamstar
      & AFS & \lacamstar
      \\
      \midrule
      50 & 100 & 100 & 88 & \w{1} & \w{25} & 159 & \w{24} & 118
      \\
      100 & 56 & \w{100} & 7223 & \w{2} & \w{140} & 609 & \w{139} & 545
      \\
      150 & 0 & \w{100} & \NA & \w{4} & \NA & 1463 & \NA & 1368
      \\
      \bottomrule
    \end{tabular}
    \caption{
    Comparison of anytime MAPF algorithms.
    We used sum-of-loss and 25 ``random'' scenarios of \mapname{random-32-32-20}.
    ``init'' shows scores related to initial solutions.
    ``loss'' is the gap scores from $\sum_{i\in A}\dist(s_i, g_i)$.
    The scores are averaged for instances solved by both solvers, except for $|A|=150$ because AFS failed all.
    }
    \label{table:anytime}
  \end{table}
}

{
  \setlength{\tabcolsep}{2pt}
  \begin{table}[t!]
    \centering
    \small
    \begin{tabular}{lrrrl}
      \toprule
      map
      & $|A|$
      & \%
      & time(\SI{}{\second})
      & other algorithms
      \\
      \midrule
      {\mapname{empty-8-8}}
      & 58
      & 100
      & 0.00
      & PIBT\&LaCAM{\scriptsize~(100\%; \SI{0.00}{\second})}
      \\
      {\mapname{random-32-32-20}}
      & 737
      & 100
      & 0.63
      &LaCAM{\scriptsize~(4\%; \SI{14.81}{\second})}
      \\
      {\mapname{random-64-64-20}}
      & 2943
      & 68
      & 11.64
      & \NA
      \\
      \mapname{maze-128-128-10}
      & 9772
      & 100
      & 55.54
      & \NA
      \\
      \bottomrule
    \end{tabular}
    \caption{
      Results on extremely dense scenarios.
      $|A|$ was adjusted so that $|A|/|V|=0.9$.
      For each scenario, 25 instances were prepared while randomly placing starts and goals.
      ``\%'' is the success percentage by \lacamstar with timeout of \SI{60}{\second}.
      ``time'' is the median runtime to obtain initial solutions.
      We also tested the other solvers in \cref{fig:percentage} and report solvers that solved in at least one instance.
    }
    \label{fig:dense}
  \end{table}
}

\section*{Acknowledgments}
I am grateful to Yasumasa Tarmura for his comments on the initial manuscript.
This work was partly supported by JSPS KAKENHI Grant Numbers~20J23011 and JST ACT-X Grant Number JPMJAX22A1.
I thank the support of the Yoshida Scholarship Foundation when I was a Ph.D. student.

\bibliographystyle{named}
\bibliography{ref-macro,ref}
\appendix
\section*{Appendix}

\section{MAPF-Specific Part of LaCAM}
\Cref{algo:full} presents the full version of LaCAM (\cref{algo:lacam}).
To highlight differences, the pseudocode uses gray-out for the same lines with \cref{algo:lacam}.
In what follows, we complement the MAPF-specific part of \cref{algo:full}.

\paragraph{Constraint.}
Recall that a constraint in LaCAM is defined by each domain.
In MAPF, \emph{a constraint specifies which agent is where in the next configuration}.
Then, a node \C in the constraint tree specifies locations for multiple agents in the next configuration.
The configuration generator must follow this specification.

\paragraph{Constraint Generation.}
To determine which agent will be constrained, a high-level node \N includes \order, an enumeration of all agents sorted by specific criteria.
This is specified by two functions: \funcname{get\_init\_order} (\cref{algo:full:init-node}) and \funcname{get\_order} (\cref{algo:full:successor}).
An agent $i$ is then selected following \order and depth of the low-level node in the constraint tree, which is obtained by the function \funcname{depth} (\cref{algo:full:selection}).
Doing so ensures that each path of the constraint tree has no duplicate agents.
Constraints are generated for all possible locations from $\N.\config[i]$ (\reflines{algo:start-for-constraint}{algo:end-for-constraint}).
The constraint tree is \emph{not} developed when the depth of the low-level node equals $|A|$ (\cref{algo:full:low-start}) because  all agents have constraints in such a node.

\section{Evaluation}

\subsection{Baselines}
We tested a variety of representative or state-of-the-art MAPF algorithms that have diverse properties for solvability and optimality as follows.
See also \cref{fig:percentage}.

\begin{itemize}
\item \textbf{\astar}~\cite{hart1968formal} as a vanilla search algorithm.
  It is complete and optimal.
  The used objectives were makespan (\astar-m) and sum-of-loss (\astar-l).
\item \textbf{\astar with operator decomposition (OD)}~\cite{standley2010finding} as an adaptation of the general search to MAPF.
  It was implemented as a greedy best-first search to obtain solutions as much as possible.
  The heuristic was the sum of distance toward goals.
\item \textbf{ODrM$^\ast$}~\cite{wagner2015subdimensional} as a state-of-the-art optimal and complete algorithm, similar to \astar.
  The used objectives were sum-of-loss (ODrM$^\ast$-l) and makespan (ODrM$^\ast$-m).
  The implementation was from~\url{https://github.com/gswagner/mstar_public}.
  The original implementation uses sum-of-loss as an objective function.
  The makespan version was adapted from it.
\item \textbf{Inflated ODrM$^\ast$ (I-ODrM$^\ast$)}~\cite{wagner2015subdimensional} as a state-of-the-art bounded sub-optimal and complete algorithm, which is a variant of ODrM$^\ast$.
  The used objective is makespan (I-ODrM$^\ast$-m) and sum-of-loss (I-ODrM$^\ast$-l).
  The sub-optimality was set to five to find solutions as much as possible.
\item \textbf{BCP}~\cite{lam2022branch} as a state-of-the-art optimal solver that uses reduction to a mathematical optimization problem.
  BCP is \emph{solution complete}, namely, it cannot distinguish unsolvable instances.
  The implementation used CPLEX for mathematical optimization and was retrieved from \url{https://github.com/ed-lam/bcp-mapf}.
  The used objective was flowtime (aka. sum-of-costs), following the authors' implementation.
\item \textbf{CBS}~\cite{sharon2015conflict} with many improvement techniques as appeared in~\cite{li2021pairwise}, as a state-of-the-art optimal solver.
  This is representative of two-level combinatorial search.
  CBS is solution complete.
  The implementation was retrieved from \url{https://github.com/Jiaoyang-Li/CBSH2-RTC}.
  The used objectives were flowtime (CBS-f), makespan (CBS-m), and sum-of-loss (CBS-l).
  The original implementation uses flowtime as an objective function.
  The makespan and sum-of-loss versions were adapted from it.
\item \textbf{EECBS}~\cite{li2021eecbs} as a state-of-the-art bounded sub-optimal but solution complete algorithm, which is a variant of CBS.
  The used objectives were flowtime (EECBS-f), makespan (EECBS-m), and sum-of-loss (EECBS-l).
  The sub-optimality was set to five.
  The implementation was retrieved from~\url{https://github.com/Jiaoyang-Li/EECBS}.
  The original implementation uses flowtime as an objective function.
  The makespan and sum-of-loss versions were adapted from it.
  {
  \begin{algorithm}[t!]
    \caption{LaCAM}
    \label{algo:full}
    \small
    \begin{algorithmic}[1]
    \Input{MAPF instance}
    \Output{solution or \nosolution}
    \Preface{$\C\init \defeq \langle~\parent: \bot, \who: \bot, \where: \bot~\rangle$}
      \State \same{initialize \open, \explored}
      \Comment{stack, hash table}
      \State $\same{\N\init \leftarrow} \begin{aligned}[t]
        \same{\bigl\langle}~
        &\same{\config: \S,
        \tree: \llbracket~\C\init~\rrbracket,
        \parent: \bot,}\\
        &\order: \funcname{get\_init\_order}()
        ~\same{\bigr\rangle}\end{aligned}$
      \label{algo:full:init-node}
      \State \same{$\open.\push\left(\N\init\right)$;~~$\explored[\S] = \N\init$}
      \WhileSame{$\open \neq \emptyset$}
      \State \same{$\N \leftarrow \open.\funcname{top}()$}
      \IfSingleSame{$\N.\config = \G$}{\Return $\backtrack(\N)$}
      \IfSingleSame{$\N.\tree = \emptyset$}{$\open.\pop()$;~\Continue}
      \State \same{$\C \leftarrow \N.\tree.\pop()$}
      \If{$\depth(\C) \leq |A|$}
      \Comment{\funcname{low\_level\_expansion}}
      \label{algo:full:low-start}
      \State $i \leftarrow \N.\order\left[\depth(\C)\right]$
      \label{algo:full:selection}
      \For{$u \in \neigh(\N.\config[i]) \cup \left\{ \N.\config[i] \right\}$}
      \label{algo:start-for-constraint}
      \State $\C\new \leftarrow \langle~\parent: \C, \who: i, \where: u~\rangle$
      \State $\N.\tree.\push\left(\C\new\right)$
      \EndFor
      \label{algo:end-for-constraint}
      \EndIf
      \label{algo:full:low-end}
      \State \same{$\Q\new \leftarrow \funcname{configuration\_generator}(\N, \C)$}
      \IfSingleSame{$\Q\new = \bot$}{\Continue}
      \IfSingleSame{$\explored\left[\Q\new\right] \neq \bot$}{\Continue}
      \label{algo:full:closed}
      \State $\same{\N\new \leftarrow} \begin{aligned}[t]
        \same{\langle}~
        &\same{\config: \Q\new,
        \tree: \left\llbracket~\C\init~\right\rrbracket,
        \parent: \N,}\\
        &\order: \funcname{get\_order}(\Q\new, \N)
        ~\same{\bigr\rangle}\end{aligned}$
        \label{algo:full:successor}
      \State \same{$\open.\push\left(\N\new\right)$;\;$\explored\left[\Q\new\right] = \N\new$}
      \EndWhile
      \State \same{\Return \nosolution}
    \end{algorithmic}
  \end{algorithm}
}

\item \textbf{Prioritized planning (PP)}~\cite{erdmann1987multiple,silver2005cooperative} as a basic approach for MAPF.
  The implementation first uses a distance heuristic~\cite{van2005prioritized} for the planning order.
  Furthermore, it involves the repetition of PP with random priorities until the problem is solved.
  The implementation was adapted from~\url{https://github.com/Kei18/pibt2}.
\item \textbf{MAPF-LNS2 (LNS2)}~\cite{li2022mapf} as a state-of-the-art sub-optimal and incomplete solver, based on a large neighborhood search.
  The implementation was retrieved from~\url{https://github.com/Jiaoyang-Li/MAPF-LNS2}.
\item \textbf{PIBT}~\cite{okumura2022priority}, which is an incomplete and sub-optimal algorithm.
  A vanilla PIBT was tested because LaCAM used PIBT.
  To detect planning failure, PIBT was regarded as a failure when it reached pre-defined sufficiently large timesteps.
  The implementation was retrieved from~\url{https://github.com/Kei18/pibt2}.
\item \textbf{\pibtp}~\cite{okumura2022priority}, as a state-of-the-art scalable MAPF solver, which is incomplete and sub-optimal.
  It used push and swap~\cite{luna2011push} to complement solutions.
\item \textbf{LaCAM}~\cite{okumura2023lacam}, on which \lacamstar is based.
  This is a complete sub-optimal algorithm.
  It used a vanilla PIBT as a configuration generator.
  The implementation was from~\url{https://github.com/Kei18/lacam}.
\item \textbf{AFS}~\cite{cohen2018anytime}, an anytime version of CBS that eventually converges to optima.
  Its implementation was obtained from the authors of the paper.
  The original implementation uses flowtime as an objective function.
  The sum-of-loss versions were adapted from it.
\end{itemize}

\subsection{Small Complicated Instances for Sum-of-loss}
\Cref{table:small-complicated-sum-of-loss} presents the results of the small complicated instances for the sum-of-loss, corresponding to \cref{table:small-complicated}.
Overall, the results are similar to those for makespan;
many solvers failed in some instances or compromised the solution quality, while \lacamstar immediately found initial solutions and converged to (near-)optimal solutions.

\subsection{MAPF Benchmark}
\paragraph{How to Test Each Algorithm.}
When increasing the number of agents by $50$, we stopped testing some algorithms if they failed to solve all instances in the previous round.

\paragraph{Results.}
\Cref{fig:mapf-bench-1,fig:mapf-bench-2,fig:mapf-bench-3} present full results of the MAPF benchmark, complementing \cref{fig:mapf-bench-short}.
In addition to sum-of-loss, we also present makespan normalized by $\max_{i \in A}\dist(s_i, g_i)$.
Scores of \lacamstar are for initial solutions.
Overall, \lacamstar solved a variety of problem instances that have diverse sizes of graphs or agents, sparseness, and complexity, within several seconds.
Solution qualities of \lacamstar are comparable to other sub-optimal algorithms.
Specifically, the qualities are similar to those of PIBT and LaCAM because it is based on these algorithms.

\paragraph{Discussion of Other Algorithms.}
Although recent remarkable progress in MAPF studies, the scalability of optimal algorithms (e.g., \ao{ODrM}, BCP, CBS) is limited;
they failed to handle a few hundred agents in most cases.
Bounded sub-optimal algorithms such as \ao{I-ODrM} and EECBS can solve a variety of instances but struggle to solve challenging instances (e.g., with 1000 agents).
Sub-optimal algorithms such as PP, LNS2, or PIBT$^{(+)}$ sometimes can handle such challenging instances but still failed many instances.
From another perspective, it is ideal for an algorithm to be complete, however, the completeness can be the bottleneck for achieving speed.
This is observed in \astar, OD, and \ao{ODrM}.

In general, makespan-optimal solutions are easier to obtain than sum-of-loss-optimal ones.
This is because, given an instance, the number of makespan-optimal solutions is larger than that of sum-of-loss.
Empirically, this is validated with \ao{ODrM} and CBS in \cref{fig:percentage}.
For instance, CBS-m solved much more instances than CBS-l.
Meanwhile, we can see a reverse trend in \ao{I-ODrM}, which is a bounded sub-optimal complete algorithm;
\ao{I-ODrM}-l solved more instances than \ao{I-ODrM}-m.
We regard this as an effect of admissible heuristics.
The sum-of-loss optimization uses $\sum_{i\in A}\dist(s_i, g_i)$ while the makespan optimization uses $\max_{i\in A}\dist(s_i, g_i)$;
the former provides more informatic guidance than the latter.

{
  \setlength{\tabcolsep}{1.2pt}
  \renewcommand{\arraystretch}{0.9}
  \newcommand{\drawmap}[1]{
    \multicolumn{2}{c}{
      \begin{minipage}{0.12\linewidth}
        \centering
        \includegraphics[width=1.0\linewidth]{fig/raw/maps/#1}
      \end{minipage}
    }
  }
  \newcommand{\lines}{
    \cmidrule(lr){2-3}
    \cmidrule(lr){4-5}
    \cmidrule(lr){6-7}
    \cmidrule(lr){8-9}
    \cmidrule(lr){10-11}
    \cmidrule(lr){12-13}
  }
  \begin{table}[t!]
    \centering
    \scriptsize
    \begin{tabular}{lrrrrrrrrrrrrc}
      \toprule
      & \multicolumn{2}{c}{\mapname{tree}}
      & \multicolumn{2}{c}{\mapname{corners}}
      & \multicolumn{2}{c}{\mapname{tunnel}}
      & \multicolumn{2}{c}{\mapname{string}}
      & \multicolumn{2}{c}{\mapname{loop-chain}}
      & \multicolumn{2}{c}{\mapname{connector}}
      \\
      & \drawmap{tree}
      & \drawmap{corners}
      & \drawmap{tunnel}
      & \drawmap{string}
      & \drawmap{loop-chain}
      & \drawmap{connector}
      \\
      &
      time & s-opt &
      time & s-opt &
      time & s-opt &
      time & s-opt &
      time & s-opt &
      time & s-opt &
      solved
      \\
      \lines
      \cmidrule(lr){14-14}
      \lacamstar
      & 0 & 1.25 & 0 & 1.12 & 0 & 1.46 & 0 & 2.36 & 2 & 7.12 & 0 & (1.48) & \multirow{2}{*}{\w{6/6}}
      \\
      after \SI{1}{\second}
      & 0 & 1.00 & 13 & 1.00 & 77 & 1.00 & 8 & 1.00 & 700 & 1.44 & 408 & (1.08)
      \\
      \cmidrule(){1-14}
      \astar
      & 0 & 1.00 & 149 & 1.00 & 35 & 1.00 & 25 & 1.00 & 15124 & 1.00 & \NA & \NA & 5/6
      \\
      ODrM$^\ast$
      & 3 & 1.00 & 40 & 1.00 & 675 & 1.00 & 0 & 1.00 & \NA & \NA & \NA & \NA & 4/6
      \\
      I-ODrM$^\ast$
      & 0 & 1.00 & 0 & 1.31 & 257 & 1.08 & 0 & 1.00 & \NA & \NA & 124 & (1.39) & 5/6
      \\
      CBS
      & 70 & 1.00 & 0 & 1.00 & \NA & \NA & 180 & 1.00 & \NA & \NA & \NA & \NA & 3/6
      \\
      EECBS
      & 2 & 1.00 & 0 & 1.00 & \NA & \NA & 0 & 1.00 & \NA & \NA & 86 & (1.61) & 4/6
      \\
      \midrule
      OD
      & 0 & 1.00 & 0 & 1.50 & 14 & 2.57 & 0 & 1.20 & 2133 & 30.62 & 5 & (1.38) & \w{6/6}
      \\
      LaCAM
      & 0 & 1.23 & 1 & 1.69 & 92 & 1.91 & 0 & 3.30 & 55 & 19.15 & 0 & (1.45) & \w{6/6}
      \\
      PP
      & \NA & \NA & 0 & 1.00 & \NA & \NA & 0 & 1.00 & \NA & \NA & \NA & \NA & 2/6
      \\
      LNS2
      & \NA & \NA & 0 & 1.00 & \NA & \NA & 0 & 1.00 & \NA & \NA & 29 & (1.00) & 3/6
      \\
      PIBT
      & \NA & \NA & \NA & \NA & \NA & \NA & \NA & \NA & \NA & \NA & \NA & \NA & 0/6
      \\
      \pibtp
      & 0 & 4.38 & 0 & 1.12 & 0 & 3.91 & 0 & 2.20 & \NA & \NA & 0 & (1.68) & 5/6
      \\
      \midrule
      BCP
      & 194 & - & 150 & - & \NA & - & 117 & - & \NA & - & \NA & - & 3/6
      \\
      \bottomrule
    \end{tabular}
    \caption{
      Results of the small complicated instances for sum-of-loss.
      The time of the unit is \SI{}{\milli\second}.
      See also the caption of \cref{table:small-complicated}.
      Regarding sum-of-loss in \mapname{connector} (decorated by parentheses), the scores normalized by known best values (80) are presented due to failing to obtain the optimal scores.
    }
    \label{table:small-complicated-sum-of-loss}
  \end{table}
}

{
  \setlength{\tabcolsep}{1pt}
  \newcommand{\entry}[3]{
    \begin{minipage}{0.155\linewidth}
      \centering
      \begin{tabular}{ll}
        \begin{minipage}{0.65\linewidth}
          \baselineskip=7pt
          {\tiny\mapname{#1}}\\
          {\tiny #2 (#3)}
        \end{minipage} &
        \begin{minipage}{0.28\linewidth}
          \includegraphics[width=1\linewidth]{fig/raw/maps/#1}
        \end{minipage}
      \end{tabular}
      \\
      \includegraphics[width=1\linewidth,height=0.75\linewidth]{fig/raw/mapf-bench/success_rate_#1}\\
      \includegraphics[width=1\linewidth,height=0.75\linewidth]{fig/raw/mapf-bench/runtime_#1}\\
      \includegraphics[width=1\linewidth,height=0.75\linewidth]{fig/raw/mapf-bench/sum_of_loss_#1}\\
      \includegraphics[width=1\linewidth,height=0.75\linewidth]{fig/raw/mapf-bench/makespan_#1}\\
    \end{minipage}
  }
  \newcommand{\labels}{
    \begin{minipage}{0.02\linewidth}
      \scriptsize
      \begin{tikzpicture}
        \node[rotate=90](l3) at (0, 4.5) {success rate};
        \node[rotate=90](l2) at (0, 2.5) {runtime (sec)};
        \node[rotate=90](l1) at (0, 0.3) {sum-of-loss / LB};
        \node[rotate=90](l1) at (0, -1.8) {makespan / LB};
        \node[] at (0, 6) {};
      \end{tikzpicture}
    \end{minipage}
  }
  \begin{figure*}[th!]
    \centering
    \begin{tabular}{ccccccc}
      \labels
      & \entry{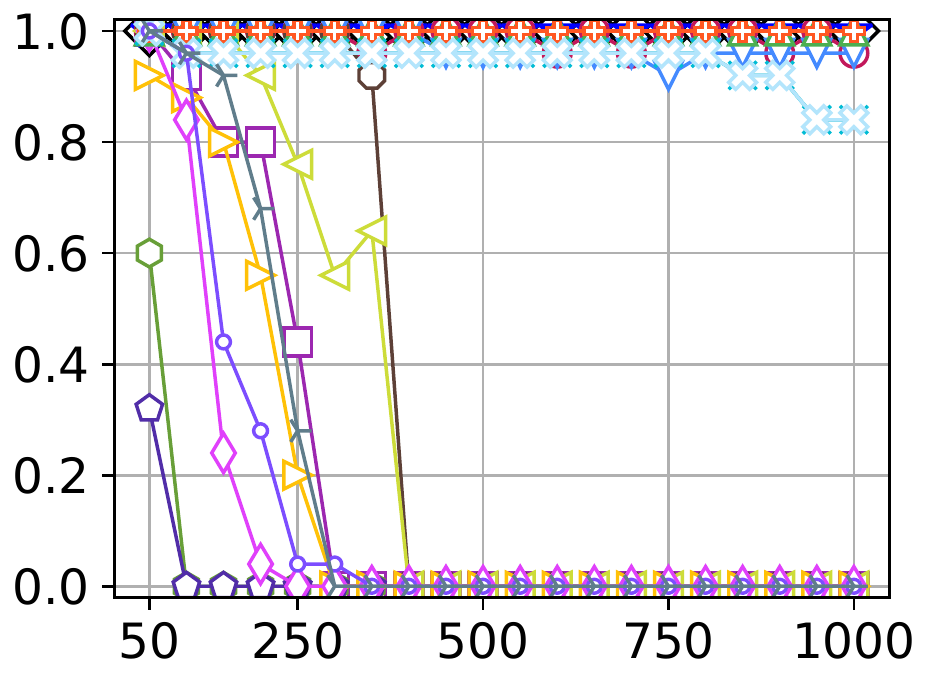}{256x256}{47,540}
      & \entry{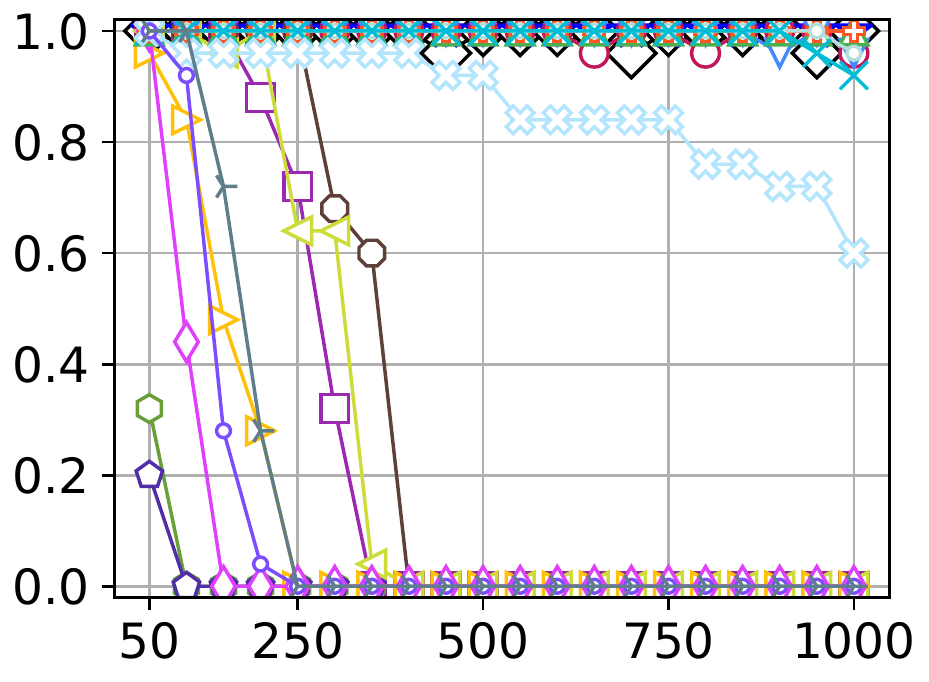}{256x256}{47,768}
      & \entry{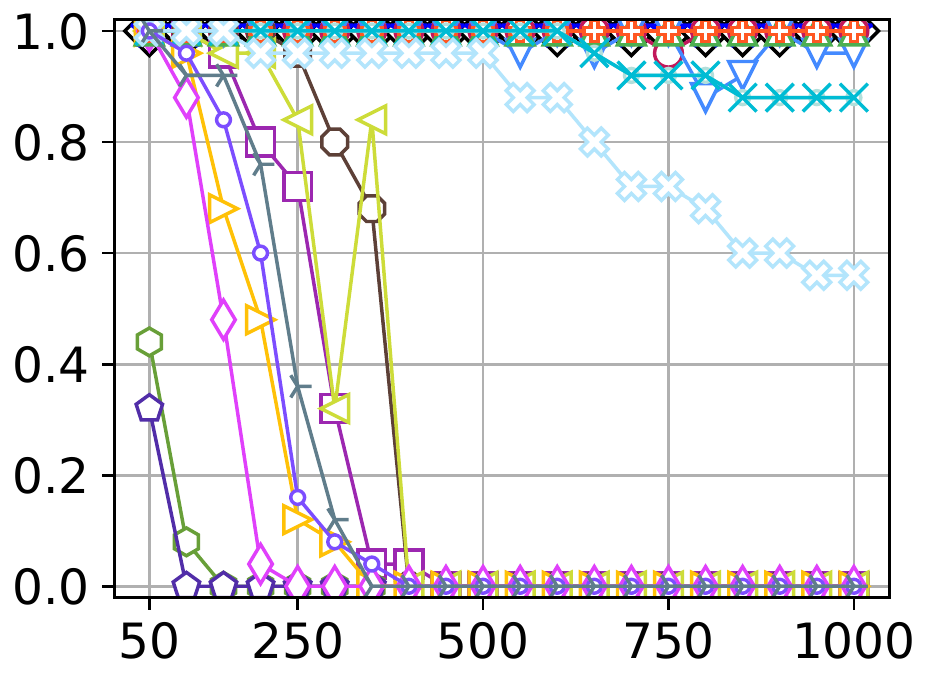}{256x256}{47,240}
      & \entry{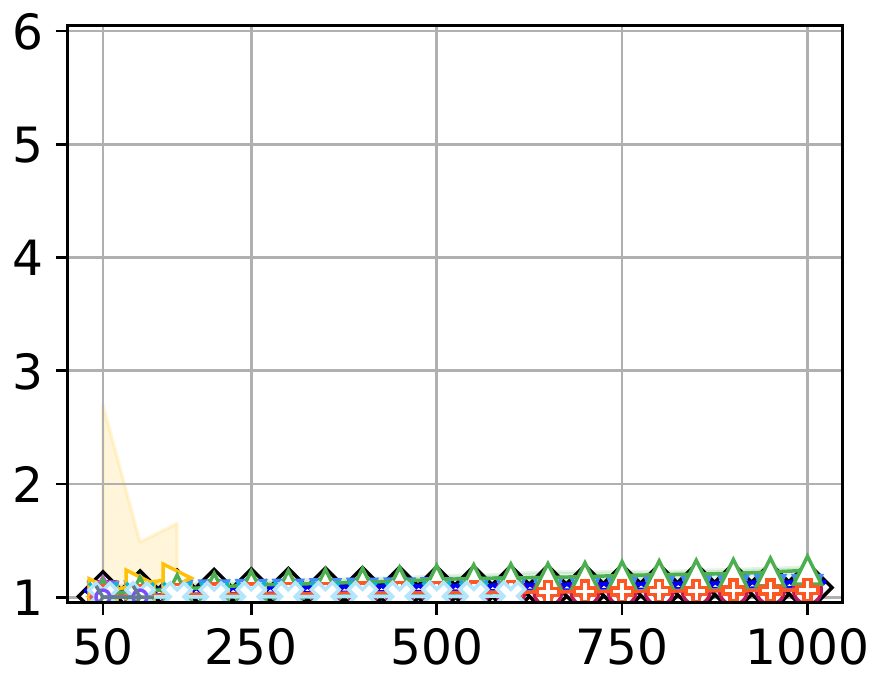}{530x481}{43,151}
      & \entry{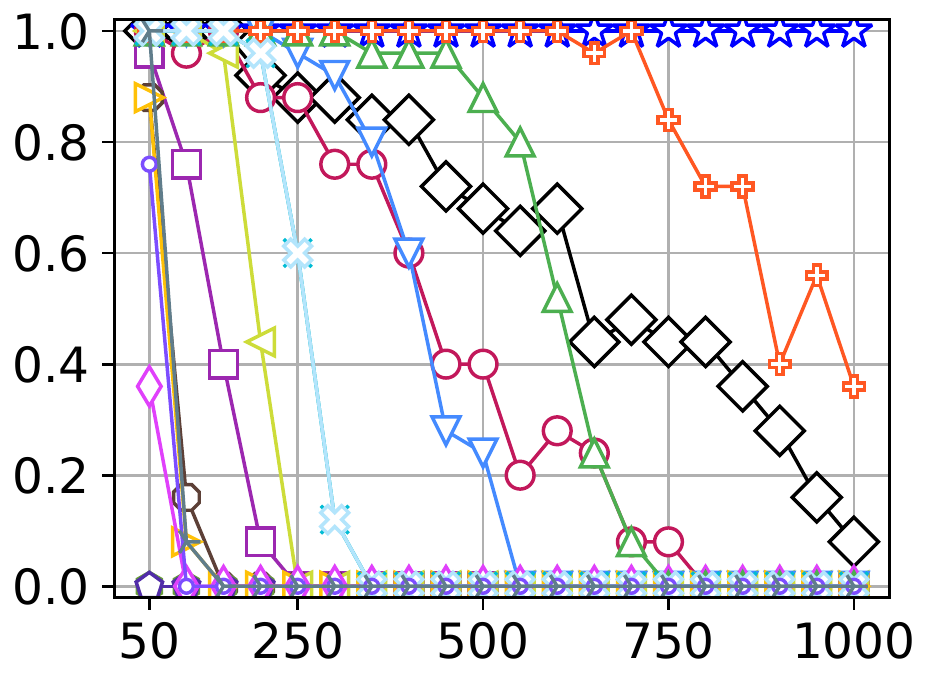}{65x81}{2,445}
      & \entry{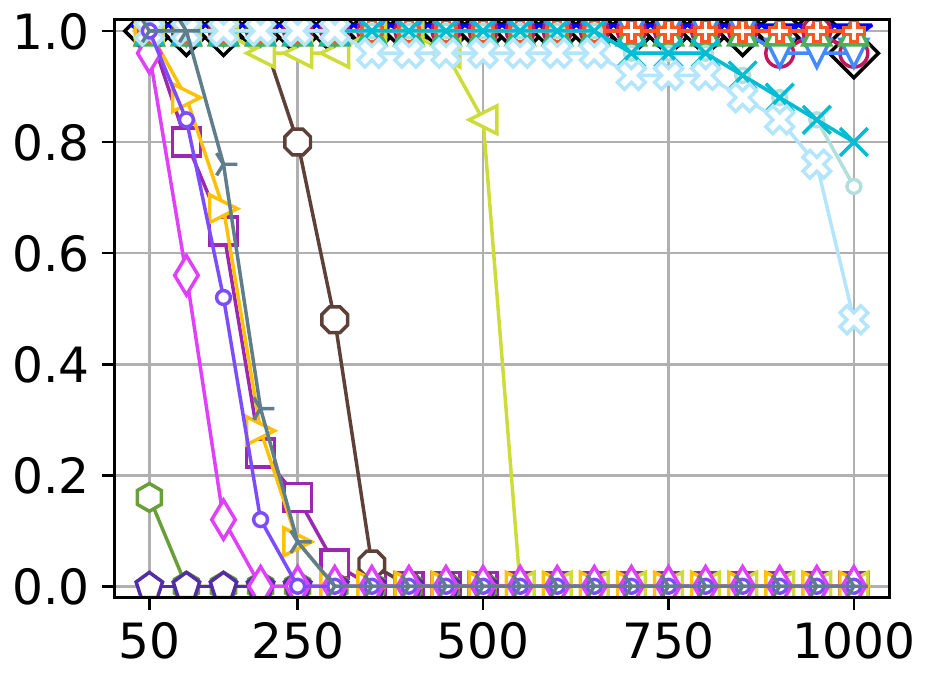}{256x257}{28,178}
      \\
      \multicolumn{7}{c}{\scriptsize agents:~$|A|$}
      \\
      \multicolumn{7}{c}{\includegraphics[width=0.8\paperwidth]{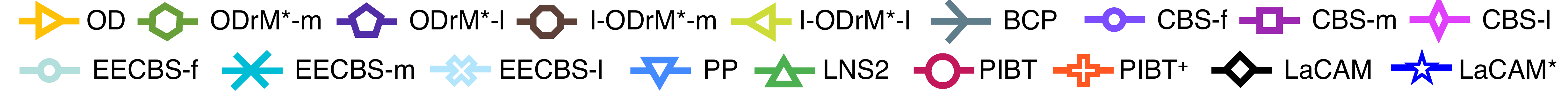}}
      \medskip\\
      \labels
      & \entry{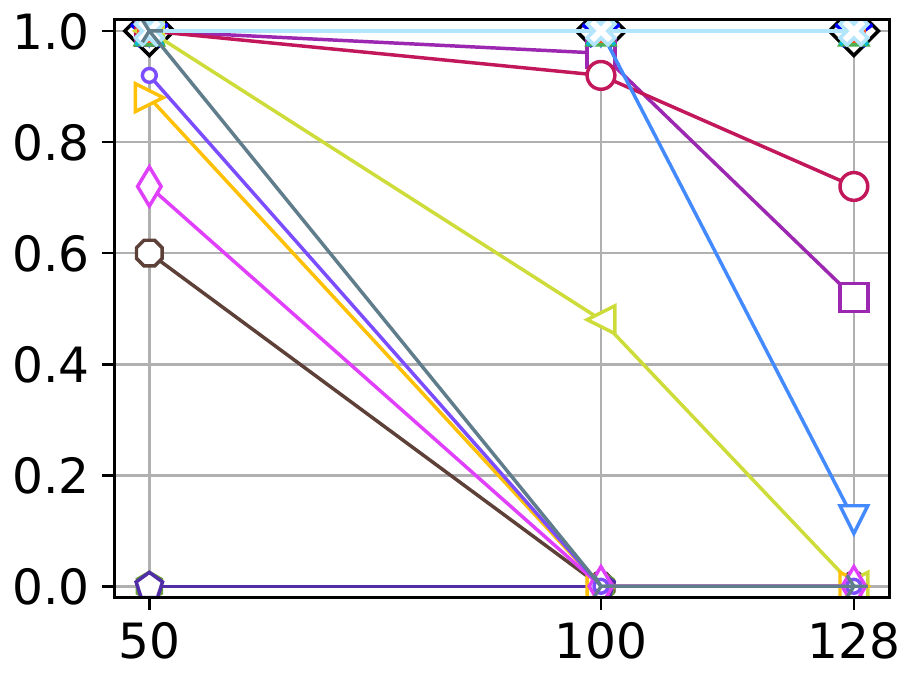}{16x16}{256}
      & \entry{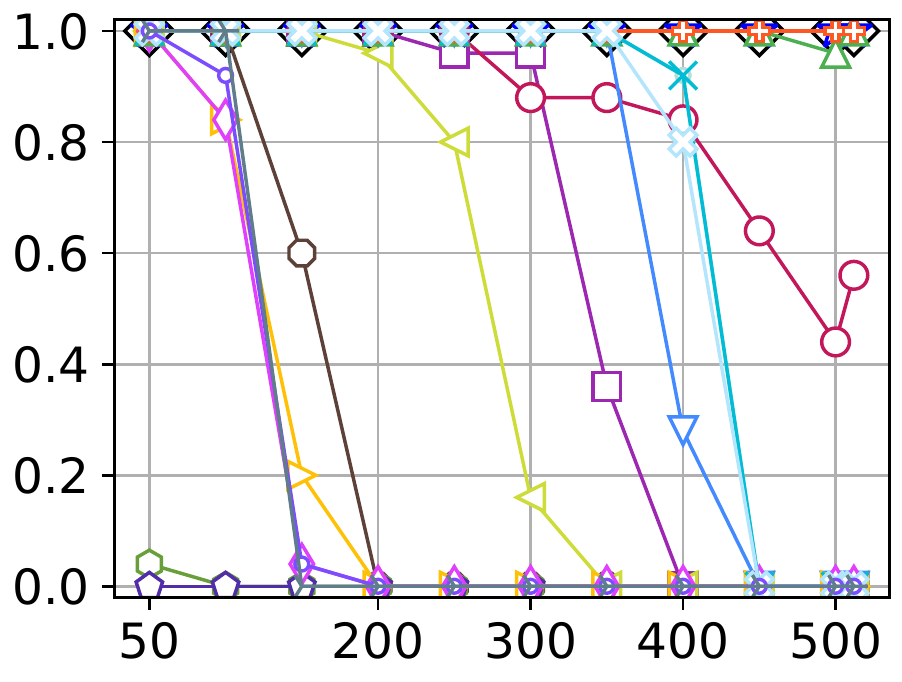}{32x32}{1,024}
      & \entry{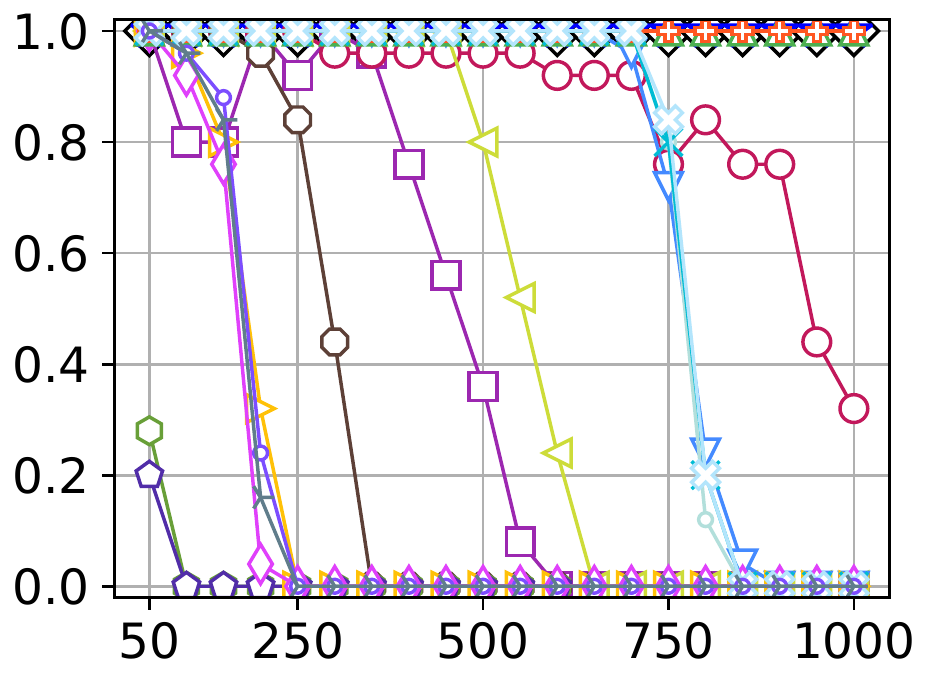}{48x48}{2,304}
      & \entry{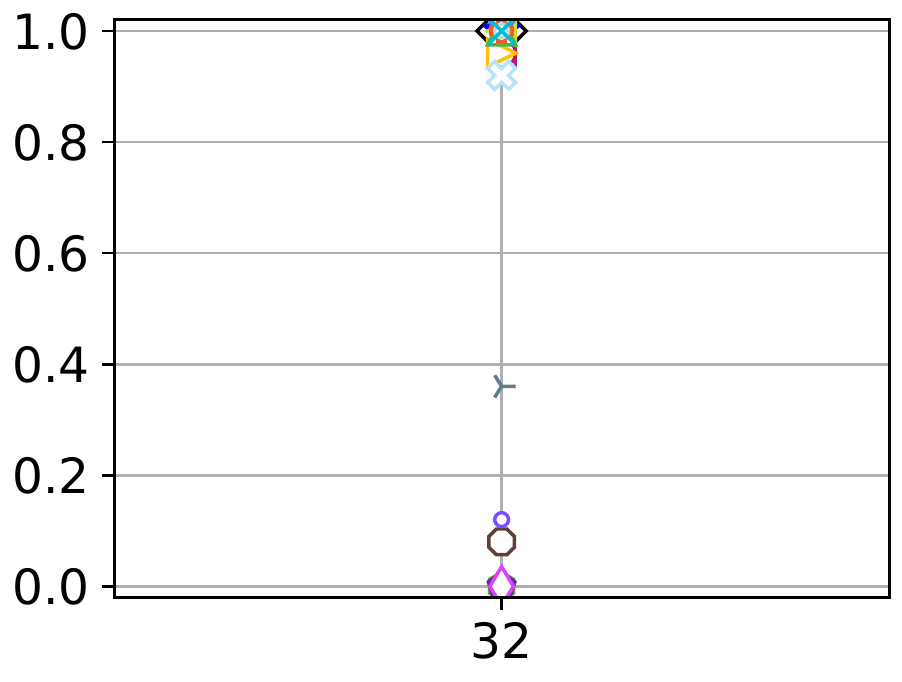}{8x8}{64}
      & \entry{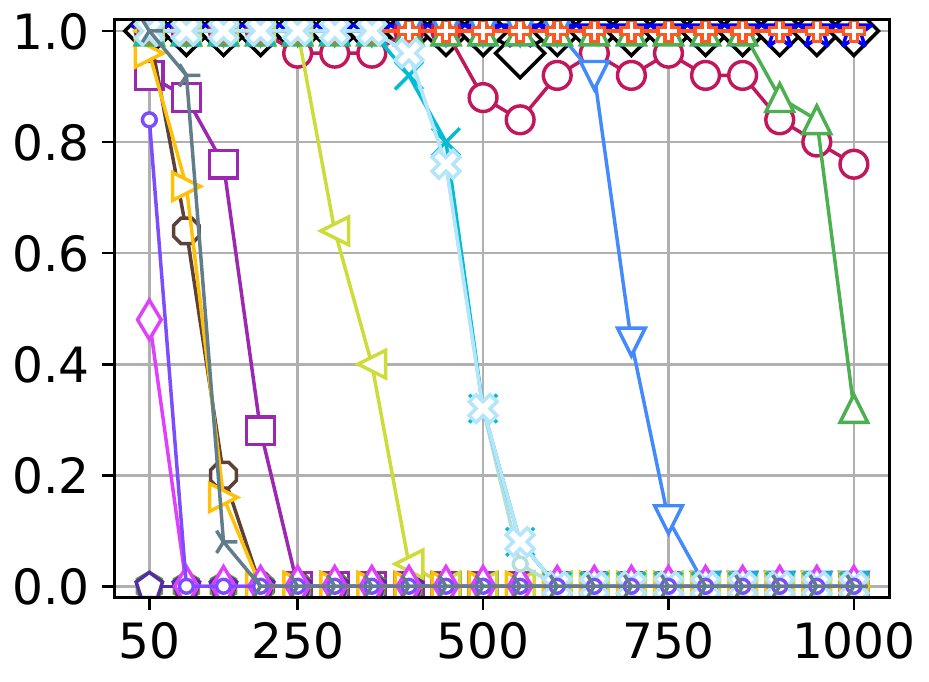}{162x141}{7,461}
      & \entry{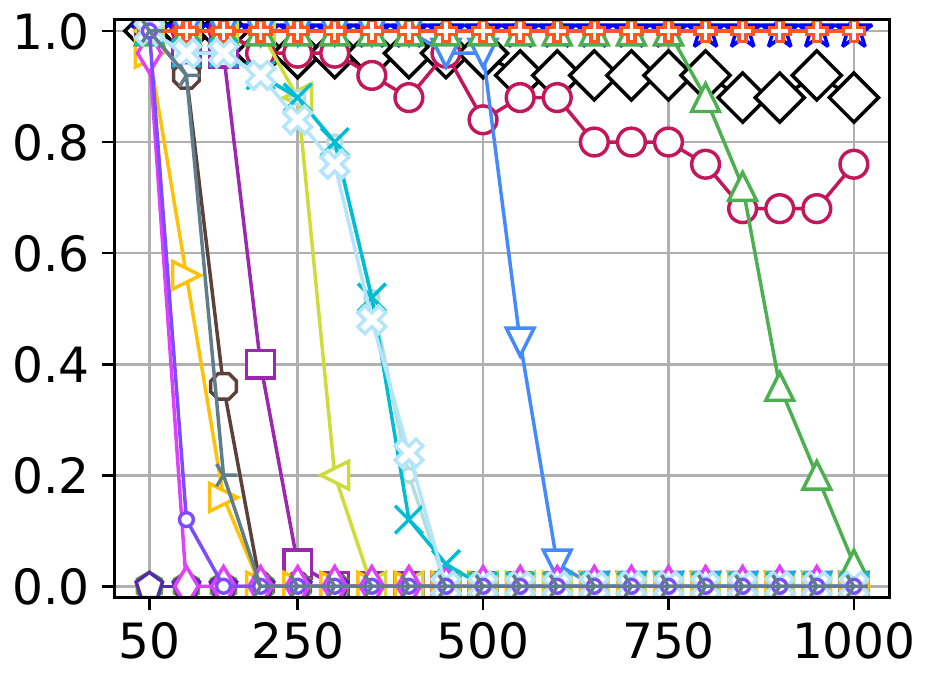}{133x270}{8,959}
      \\
      \multicolumn{7}{c}{\scriptsize agents:~$|A|$}
      \\
      \multicolumn{7}{c}{\includegraphics[width=0.8\paperwidth]{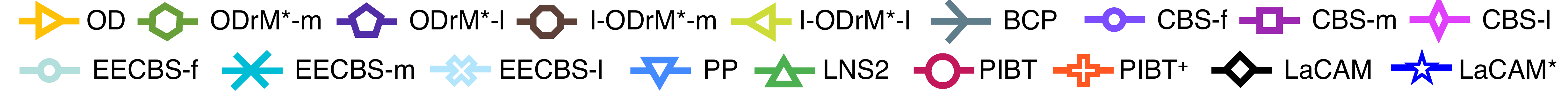}}
    \end{tabular}
    \caption{
    Result of the MAPF benchmark (1/3).
    See also the caption of \cref{fig:mapf-bench-short}.
    $|V|$ is shown in parentheses.
    }
    \label{fig:mapf-bench-1}
  \end{figure*}
  \begin{figure*}[th!]
    \centering
    \begin{tabular}{ccccccc}
      \labels
      & \entry{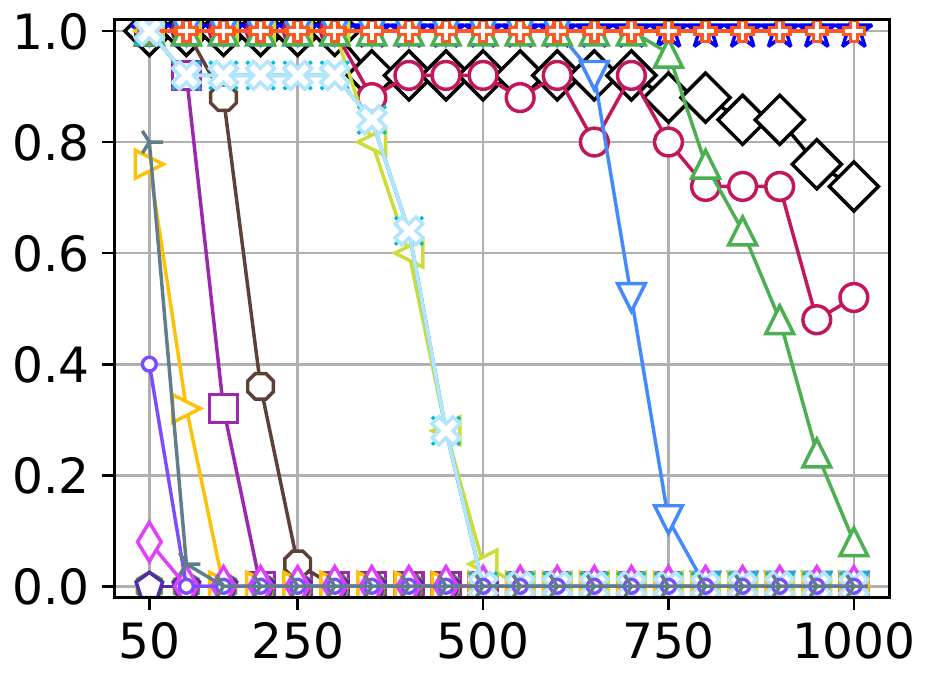}{194x194}{14,784}
      & \entry{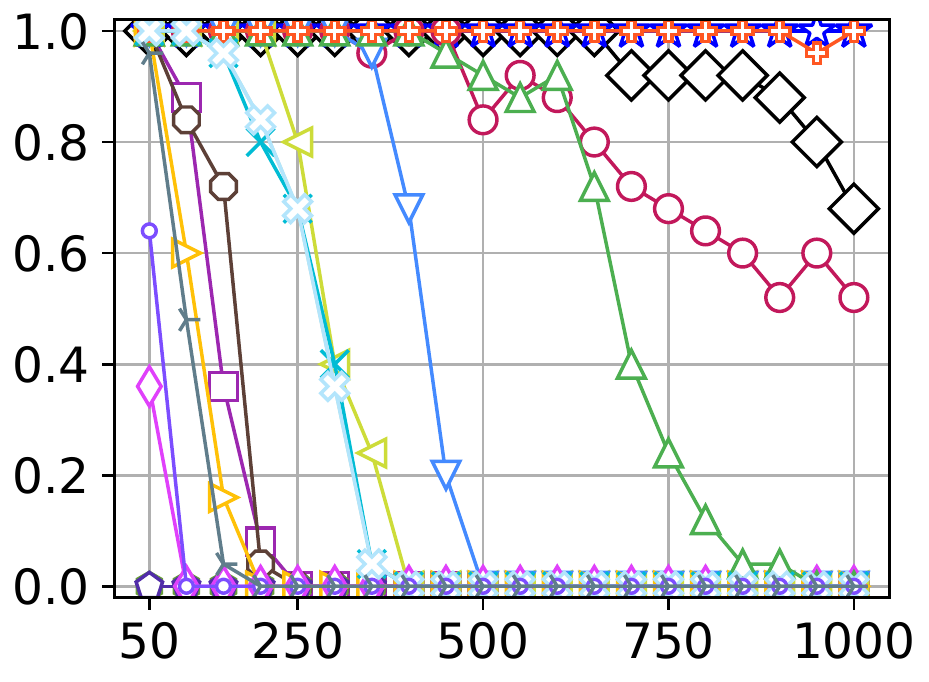}{251x180}{10,021}
      & \entry{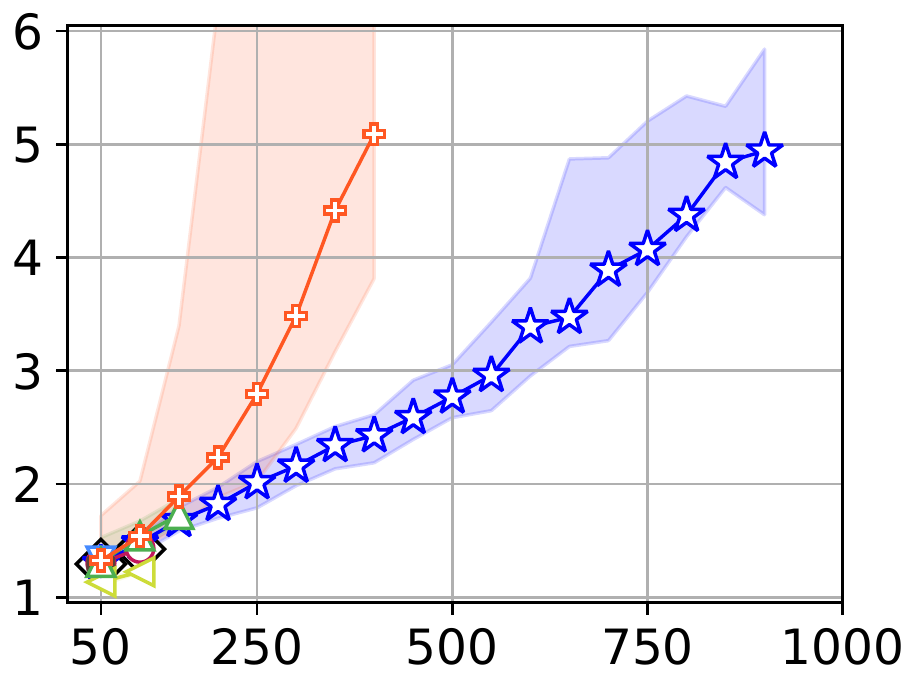}{128x128}{8,191}
      & \entry{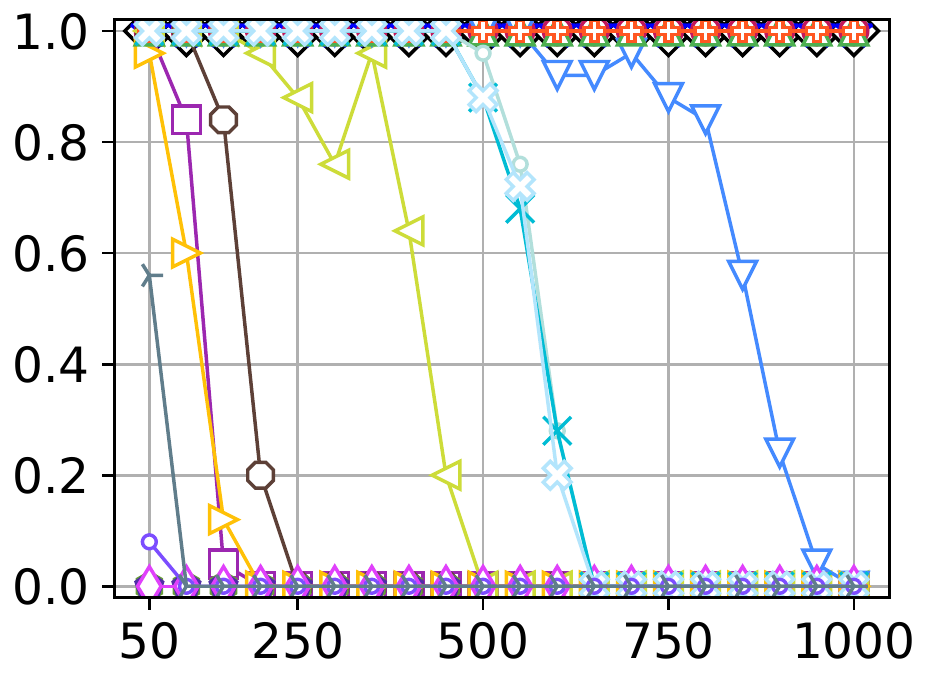}{128x128}{14,818}
      & \entry{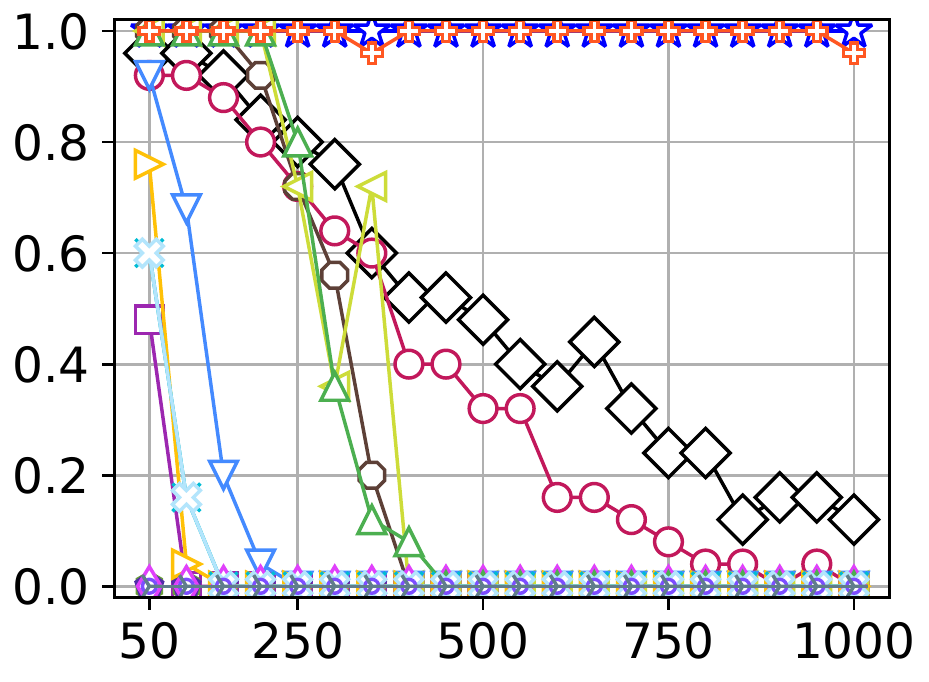}{128x128}{10,858}
      & \entry{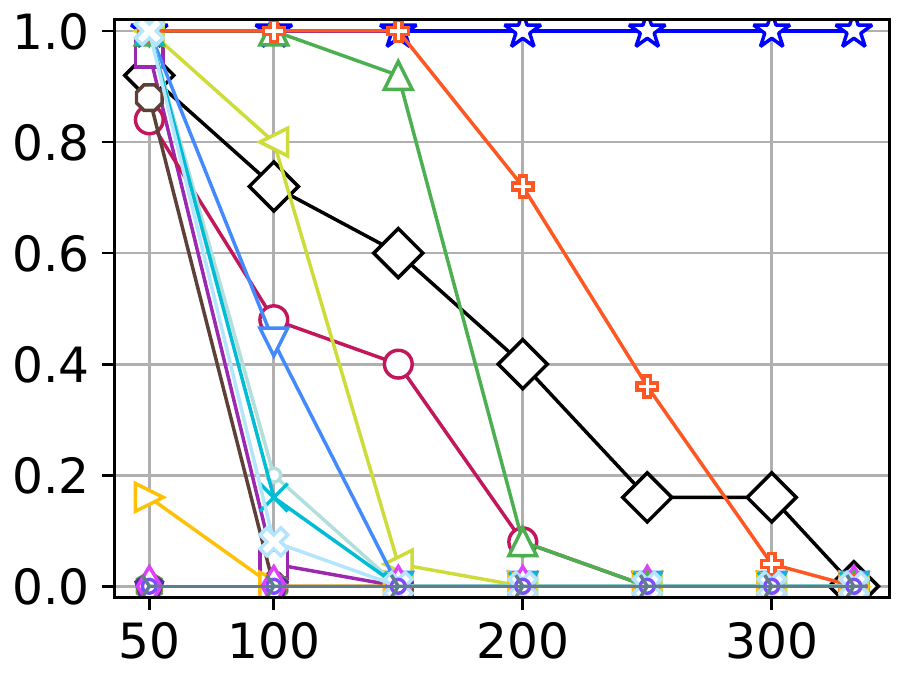}{32x32}{666}
      \\
      \multicolumn{7}{c}{\scriptsize agents:~$|A|$}
      \\
      \multicolumn{7}{c}{\includegraphics[width=0.8\paperwidth]{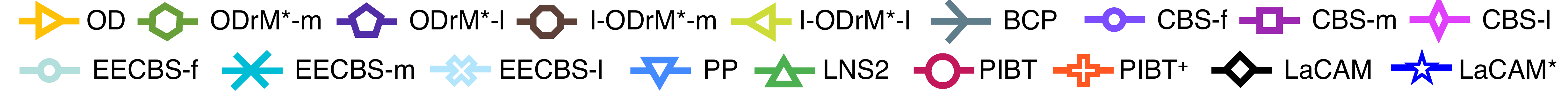}}
      \medskip\\
      \labels
      & \entry{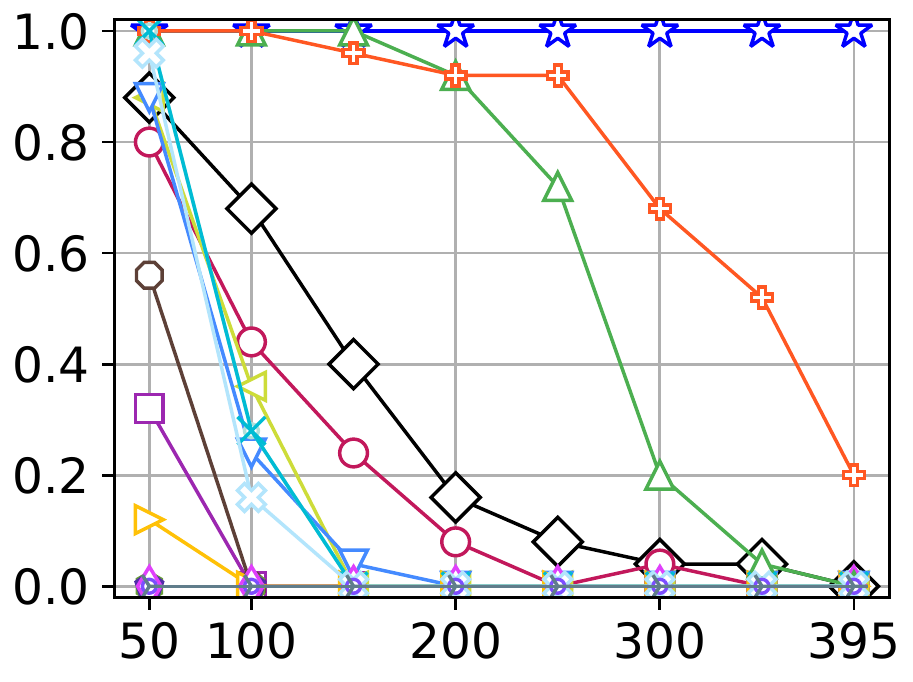}{32x32}{790}
      & \entry{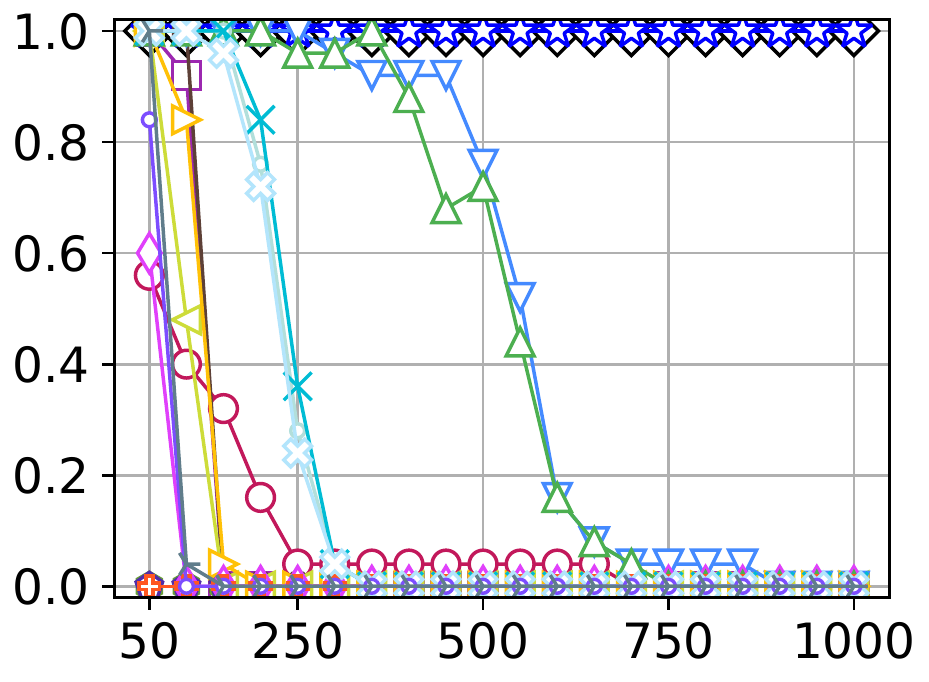}{1491x656}{96,603}
      & \entry{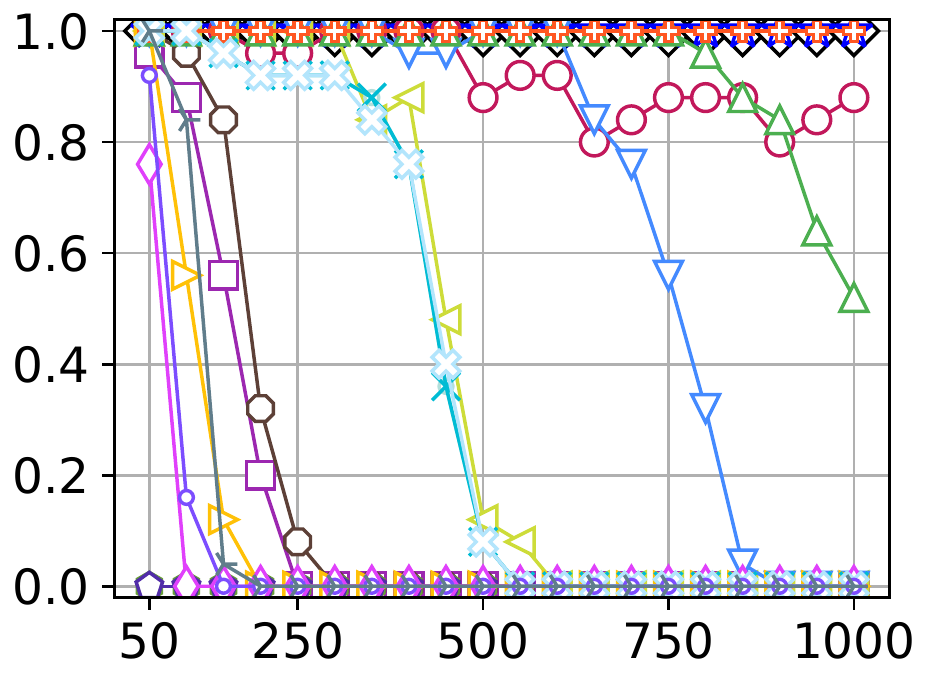}{194x194}{13,214}
      & \entry{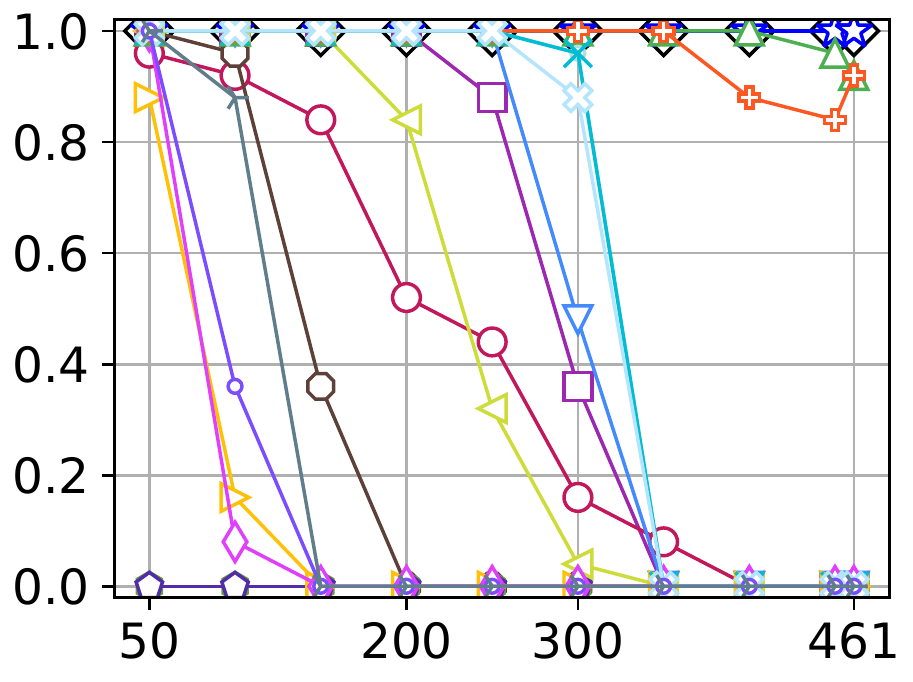}{32x32}{922}
      & \entry{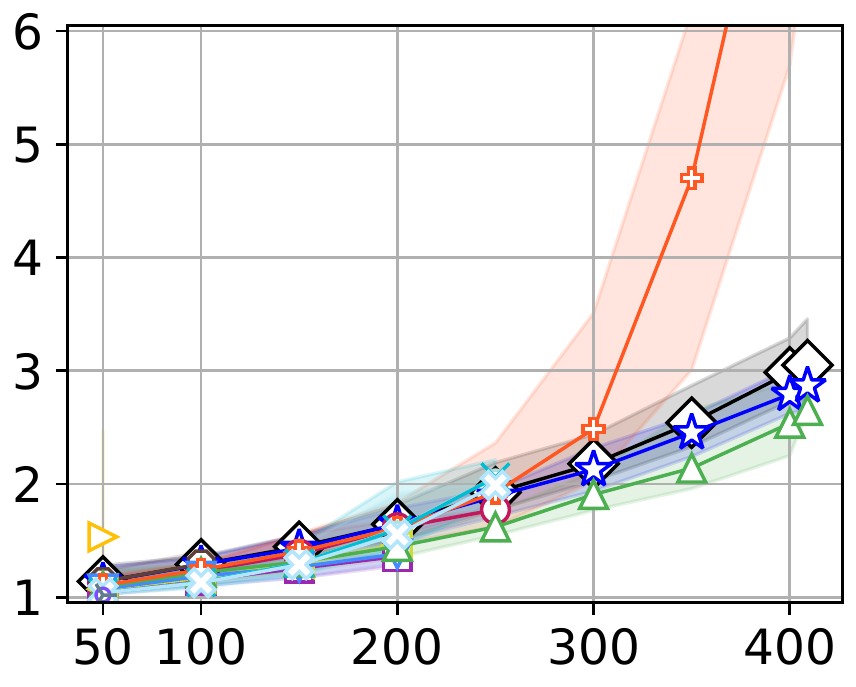}{32x32}{819}
      & \entry{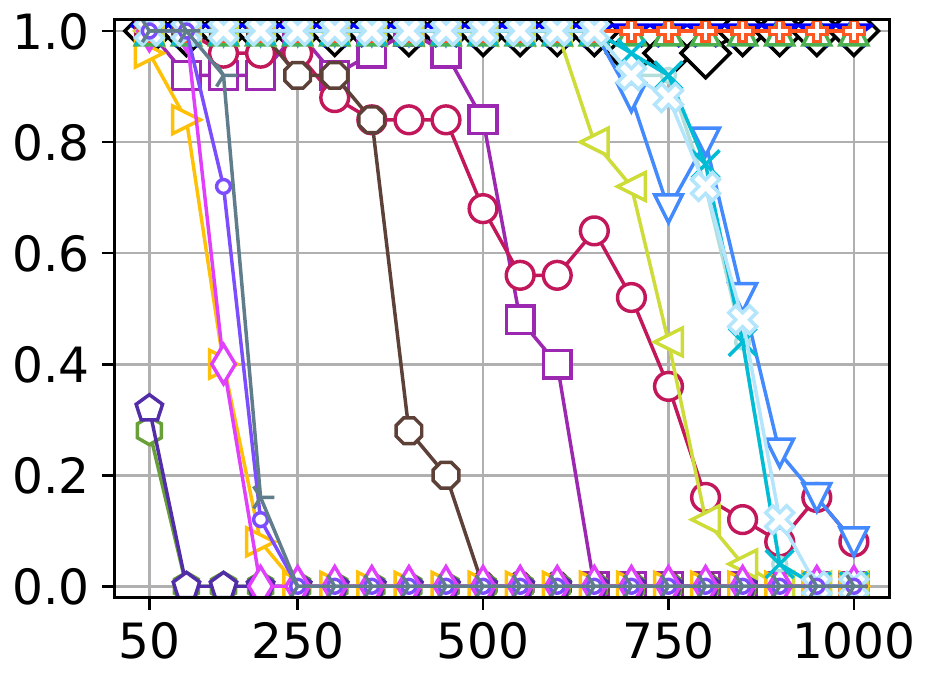}{64x64}{64x64}
      \\
      \multicolumn{7}{c}{\scriptsize agents:~$|A|$}
      \\
      \multicolumn{7}{c}{\includegraphics[width=0.8\paperwidth]{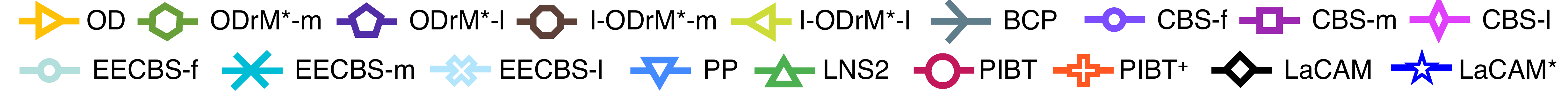}}
    \end{tabular}
    \caption{
    Result of the MAPF benchmark (2/3).
    See also the caption of \cref{fig:mapf-bench-short}.
    }
    \label{fig:mapf-bench-2}
  \end{figure*}
  \begin{figure*}[th!]
    \centering
    \begin{tabular}{ccccccc}
      \labels
      & \entry{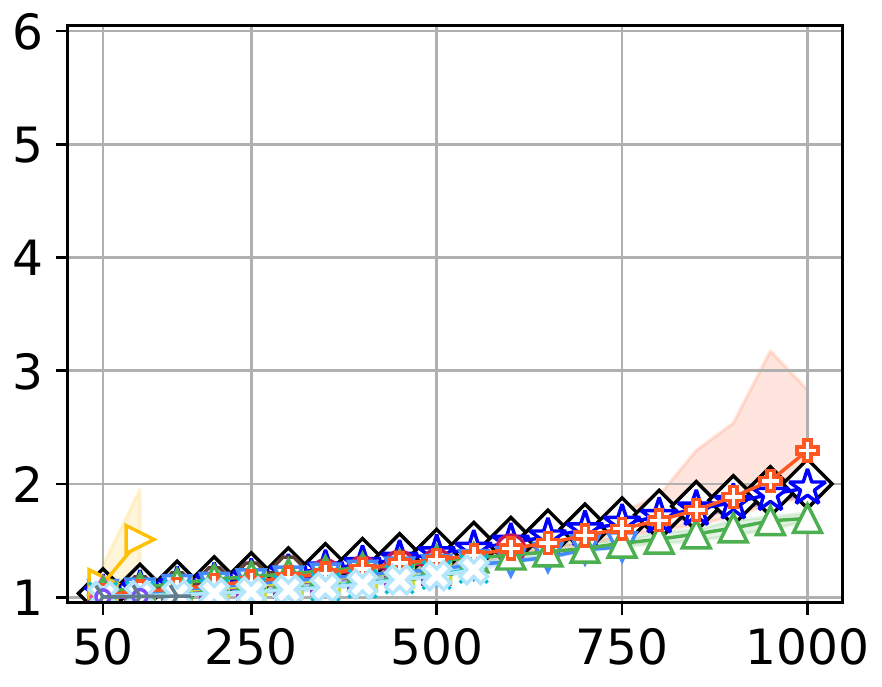}{64x64}{3,270}
      & \entry{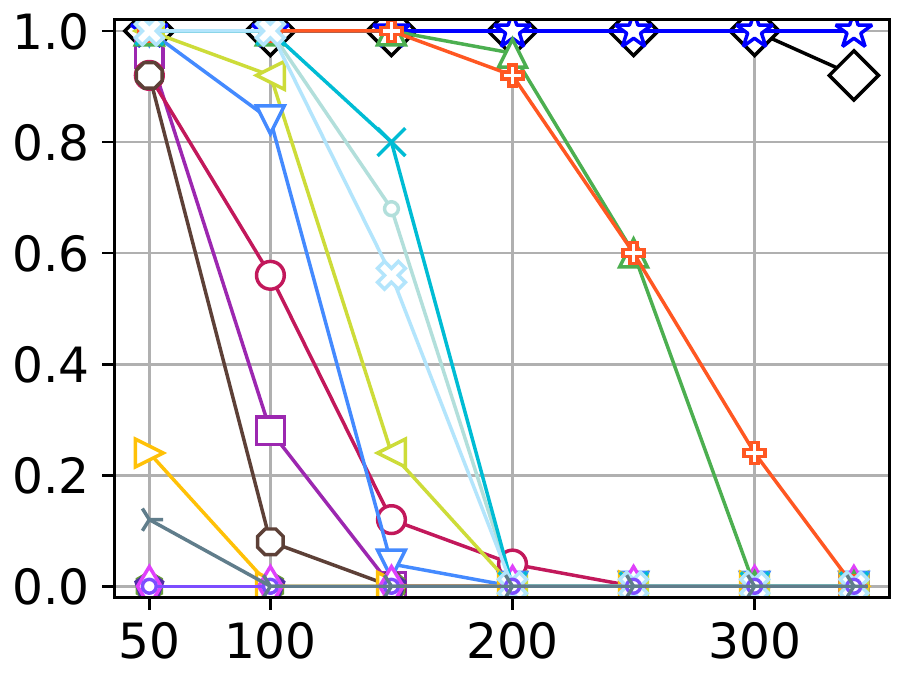}{32x32}{682}
      & \entry{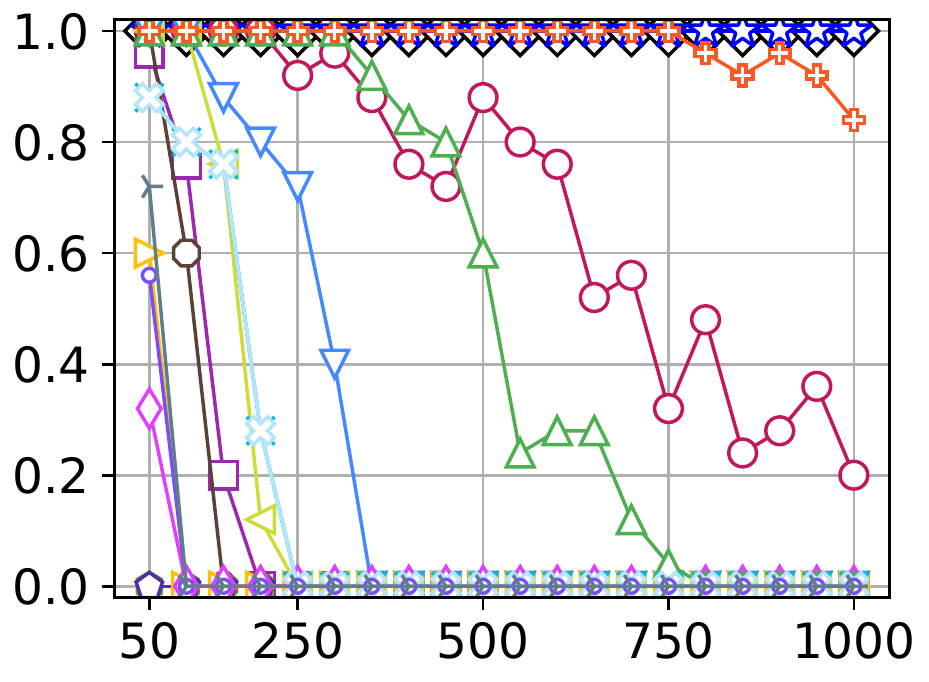}{64x64}{3,646}
      & \entry{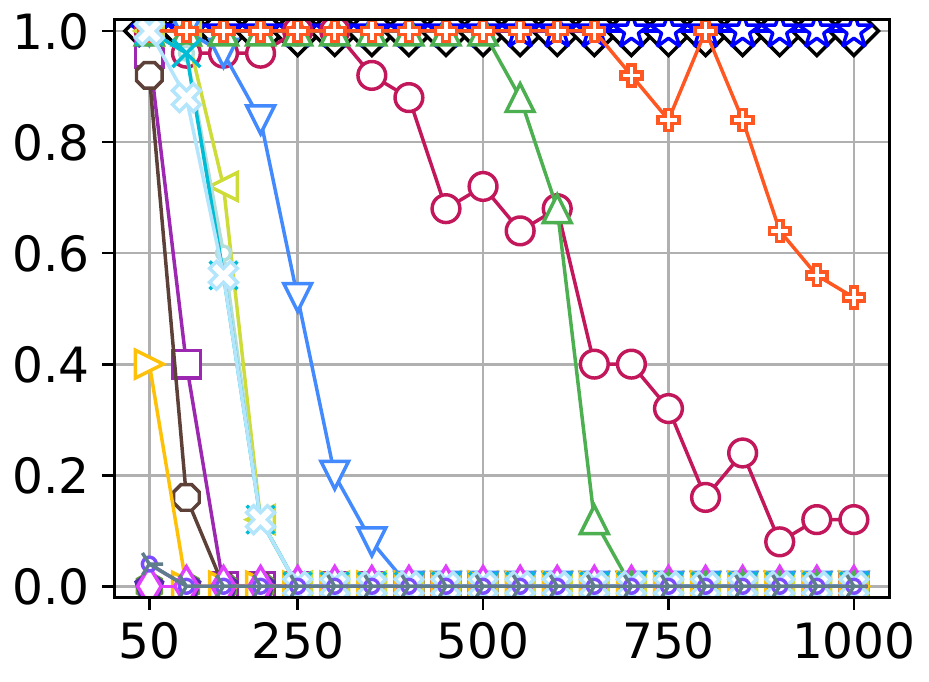}{64x64}{3,232}
      & \entry{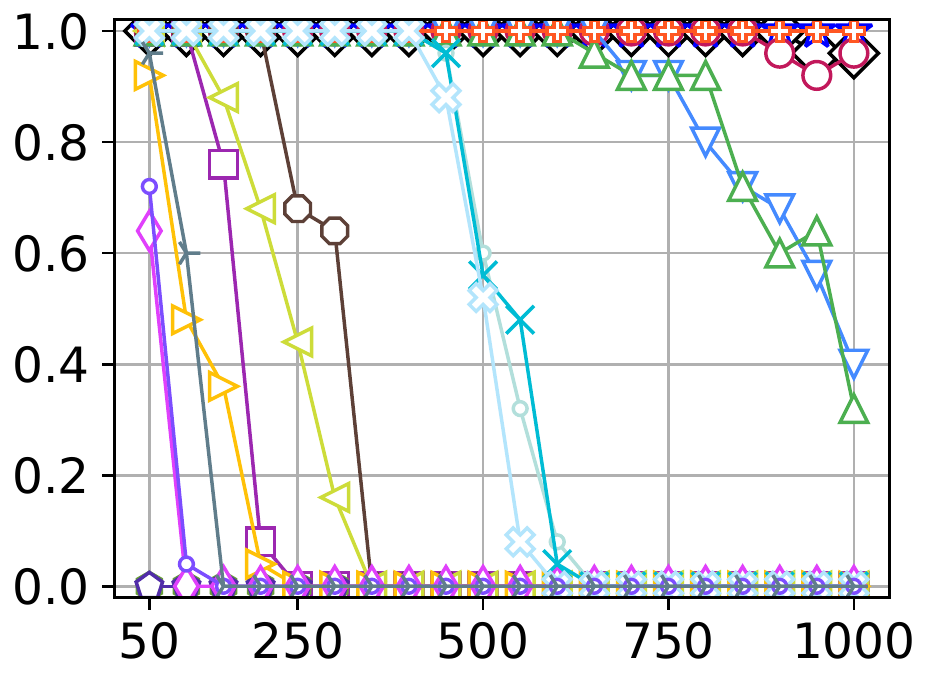}{642x578}{34,020}
      & \entry{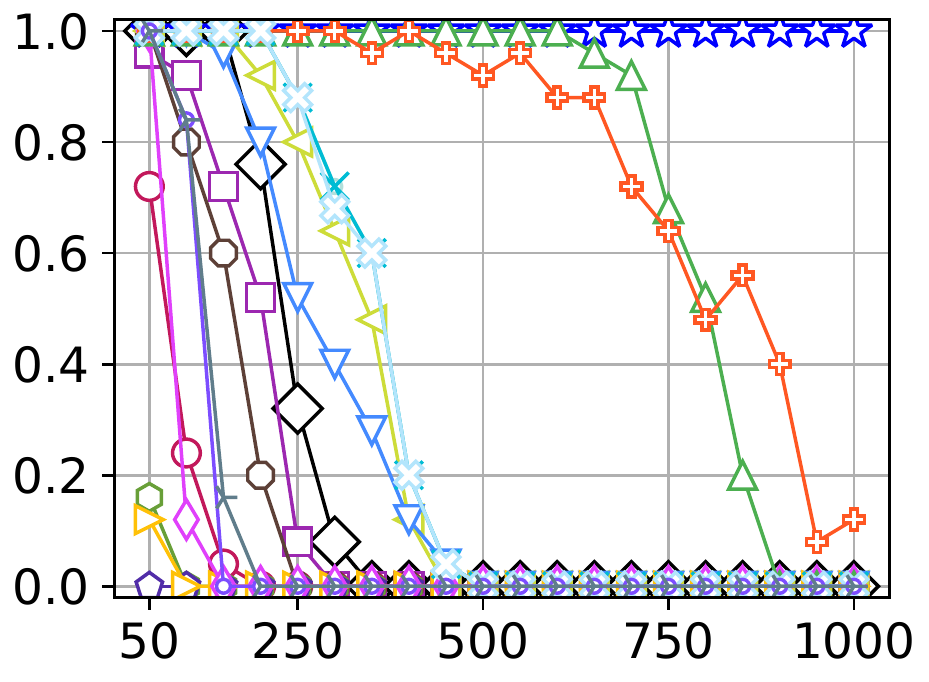}{161x63}{5,699}
      \\
      \multicolumn{7}{c}{\scriptsize agents:~$|A|$}
      \\
      \multicolumn{7}{c}{\includegraphics[width=0.8\paperwidth]{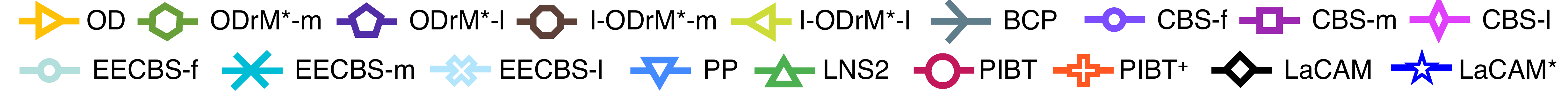}}
      \medskip\\
      \labels
      & \entry{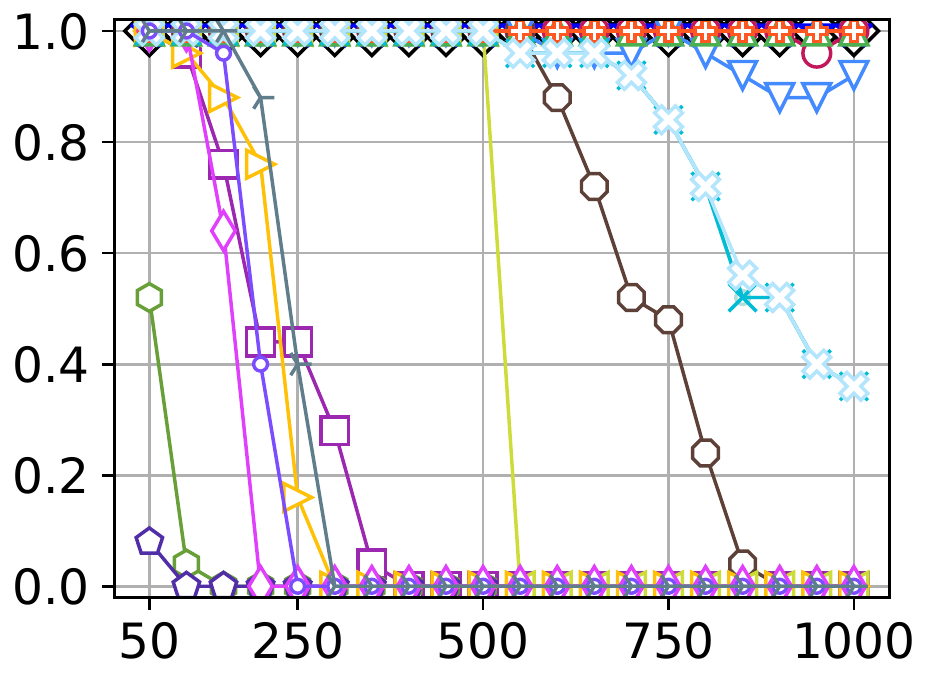}{170x84}{9,776}
      & \entry{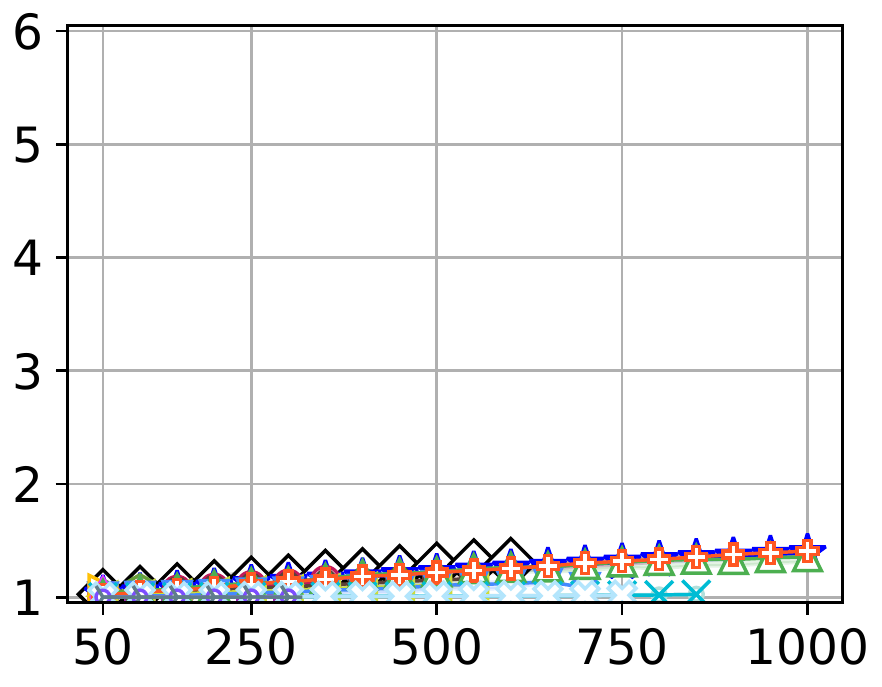}{321x123}{22,599}
      & \entry{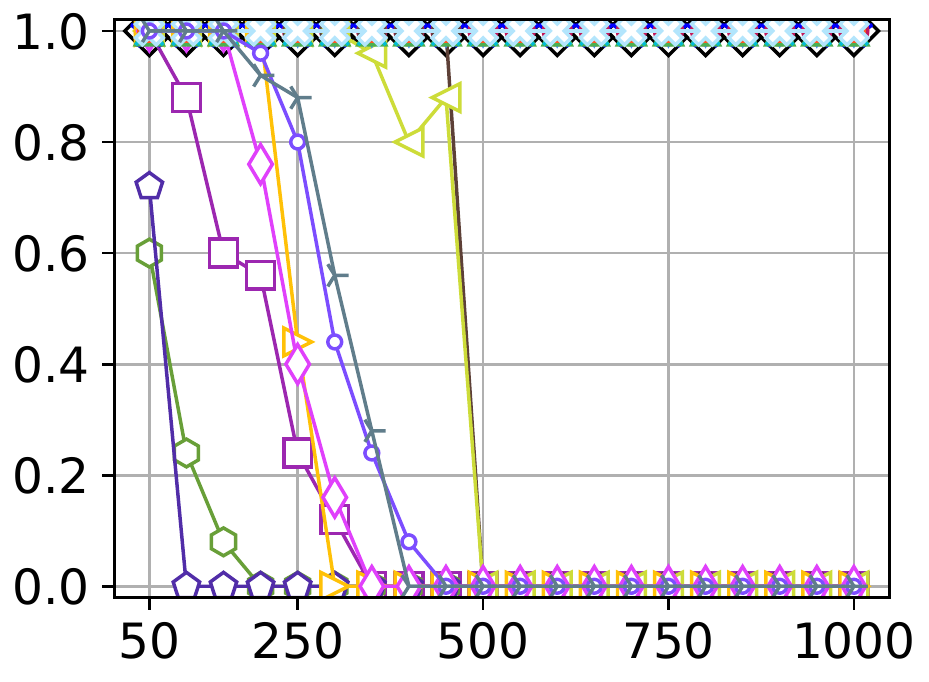}{340x164}{38,756}
      \\
      \multicolumn{7}{c}{\scriptsize agents:~$|A|$}
      \\
      \multicolumn{7}{c}{\includegraphics[width=0.8\paperwidth]{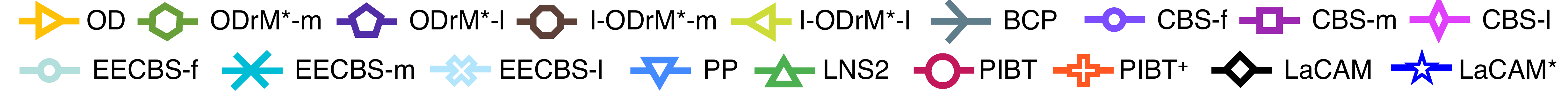}}
    \end{tabular}
    \caption{
    Result of the MAPF benchmark (3/3).
    See also the caption of \cref{fig:mapf-bench-short}.
    }
    \label{fig:mapf-bench-3}
  \end{figure*}
}

\end{document}